\pdfoutput=1
\documentclass[10pt,letterpaper]{article}
\usepackage{iclr2027_conference,times}
\iclrfinalcopy
\usepackage{amsmath,amssymb,amsthm}
\usepackage{booktabs}
\newtheoremstyle{canope}{\topsep}{\topsep}{\normalfont}{}{\bfseries}{.}{.5em}
  {\thmname{#1}\thmnumber{ #2}\thmnote{ \normalfont\itshape #3}}
\theoremstyle{canope}
\newtheorem{cntheorem}{Theorem}
\newtheorem{cnlemma}[cntheorem]{Lemma}
\newtheorem{cnproposition}[cntheorem]{Proposition}
\theoremstyle{canope}
\newtheorem{cnassumption}{Assumption}
\newcommand{\cnE}{\mathbb{E}}
\newcommand{\cnR}{\mathbb{R}}
\newcommand{\cnN}{\mathcal{N}}
\newcommand{\cnGap}{\Delta_{\mathrm{info}}}
\newcommand{\cnCE}{\mathcal{L}_{\mathrm{tok}}}
\newcommand{\cnRank}{r_{\mathrm{eff}}}
\newcommand{\cnKL}{D_{\mathrm{KL}}}
\DeclareMathOperator{\cnCov}{Cov}
\DeclareMathOperator{\cntr}{tr}
\DeclareMathOperator{\cnrank}{rank}
\DeclareMathOperator{\cnvec}{vec}
\DeclareMathOperator{\cnsoftmax}{softmax}
\DeclareMathOperator{\cnsg}{sg}
\DeclareMathOperator{\cnsign}{sign}
\DeclareMathOperator{\cnclip}{clip}
\DeclareMathOperator{\cnTC}{TC}

\usepackage{graphicx}
\usepackage{microtype}
\usepackage{placeins}
\usepackage{caption}
\newsavebox{\semanticcaptionbox}
\newsavebox{\controlscaptionbox}
\newlength{\pairedcaptionheight}
\usepackage{xcolor}
\usepackage{url}
\usepackage{hyperref}
\newcommand{\thead}[1]{{\bfseries\boldmath #1}}
\AddToHook{env/table/begin}{%
  \setlength{\abovecaptionskip}{0pt}%
  \setlength{\belowcaptionskip}{6pt}%
}
\hypersetup{hidelinks,pdfauthor={Kevin Putra Santoso; Naufal Zafran Fadil; Firania Putri Harsanti; Raden Venantius Hari Ginardi; Ganesh Neelakanta Iyer},pdftitle={Isotropic Yet Undecodable: The Sequential Content-Sufficiency Gap in Latent-Predictive Text Representations}}
\title{Isotropic Yet Undecodable: The Sequential Content-Sufficiency Gap in Latent-Predictive Text Representations}
\author{%
\bfseries Kevin Putra Santoso$^{1\ast,2,4}$, Naufal Zafran Fadil$^3$, Firania Putri Harsanti$^4$\\
\bfseries Raden Venantius Hari Ginardi$^2$, Ganesh Neelakanta Iyer$^1$\\[4pt]
\normalfont\small $^1$School of Computing, National University of Singapore\\
\normalfont\small $^2$Department of Information Technology, Institut Teknologi Sepuluh Nopember\\
\normalfont\small $^3$Faculty of Computer Science, Universitas Indonesia\\
\normalfont\small $^4$Avalon AI\\[3pt]
\normalfont\footnotesize\texttt{t0939544@u.nus.edu}, \texttt{naufal.zafran@ui.ac.id}, \texttt{firania@avalon-ai.org}\\
\normalfont\footnotesize\texttt{hari@its.ac.id}, \texttt{gni@nus.edu.sg}
}

\begin{document}
\maketitle
\begingroup
\renewcommand{\thefootnote}{\fnsymbol{footnote}}
\footnotetext[1]{This project was completed by Kevin Putra Santoso during his internship at the School of Computing, National University of Singapore.}
\endgroup
\fancyhead{}
\renewcommand{\headrulewidth}{0pt}
\begin{abstract}
  We study sequential content sufficiency by investigating whether a representation retains the ordered target information available in its input. An information-theoretic decomposition separates input ambiguity, representation loss, and readout mismatch. We construct recoverable views where perfect agreement and joint isotropic Gaussianity coexist with zero target information, and establish limits imposed by deterministic canonical anchors. Token log-loss provides a one-sided information-loss bound; a fixed-penalty ridge analysis shows why rank alone cannot determine prediction risk. These results motivate CANOPE, a nonautoregressive framework with ordered latent canvases, canonical-token supervision, and geometric regularization. On 40,000 validation sequences, latent-agreement (PL0) and token-grounded (PL2) have nearly identical pooled ranks but reach 13.5\% and 98.8\% positional Recall@1, respectively, under strong natural corruption when the correct target length is provided. On 3,930 LJSpeech validation utterances, frozen PL2 with a trained MatchaTTS readout yields 21.54\% word error rate (WER) on corrupted text, versus 99.22\% for frozen PL0, while end-to-end MatchaTTS reaches 10.93\%. These results show that geometric regularity alone does not guarantee recoverable sequential content or effective downstream access in the text settings studied here.
\end{abstract}

% Overview float: full text width; the text of Sections 1--5 is unchanged.
\begin{figure}[t]
  \centering
  \includegraphics[width=\textwidth]{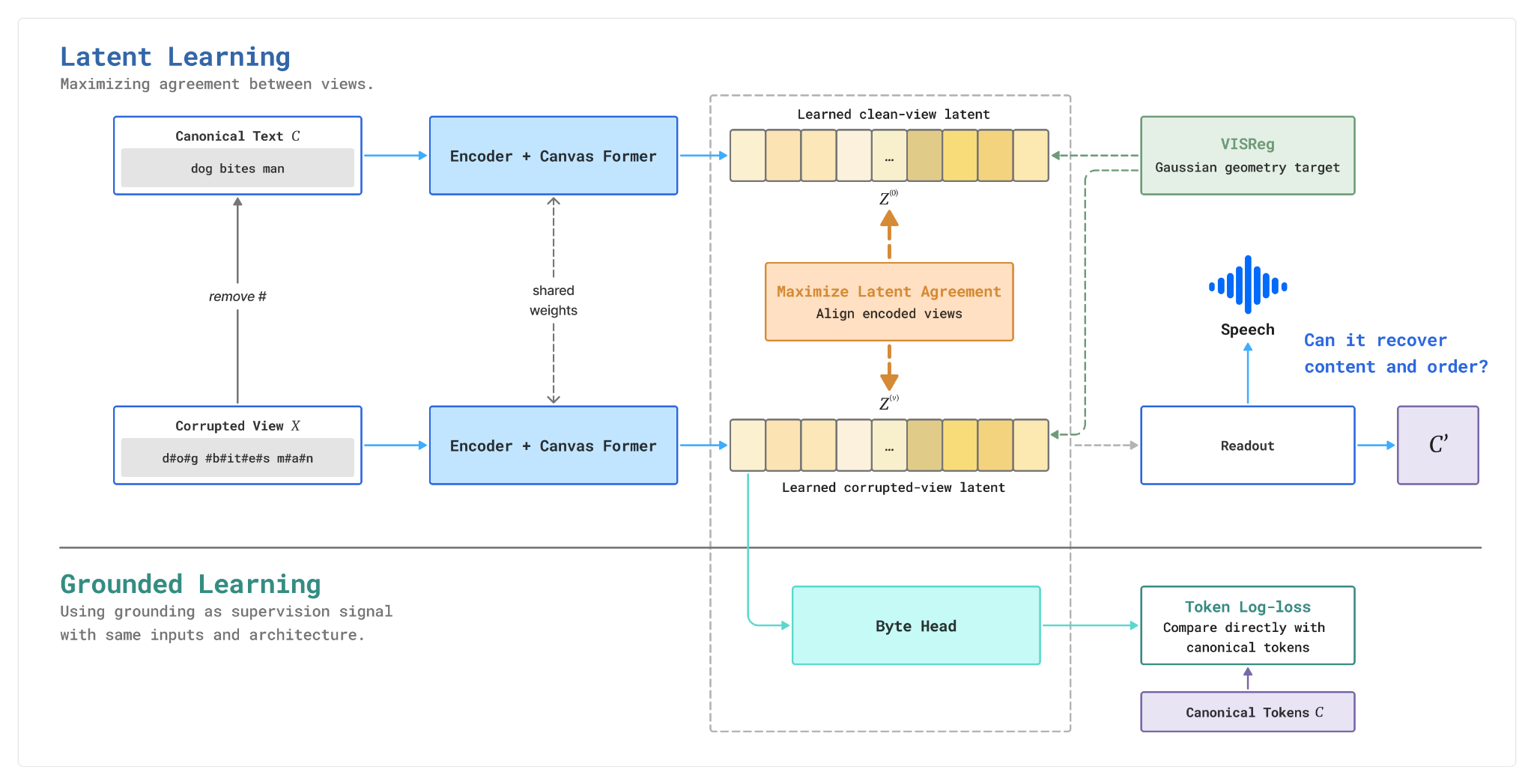}
  \caption{Latent agreement and canonical-token grounding. The recoverable CL channel isolates access to ordered content; downstream readouts test reconstruction and speech. Both objectives use the same architecture.}
  \label{fig:learning-overview}
\end{figure}

\section{Introduction}
The information a text representation must preserve depends on how it will be used. For identity retrieval, a representation may discard substantial surface detail while retaining enough information to distinguish relevant content. Sequential generation, however, requires a downstream readout to recover the target-specific content and order needed to control the output. In text-conditioned speech synthesis, for instance, the generator must say the intended words in the intended order, even when its input contains surface errors. Information treated as nuisance or non-shared across views may therefore still matter downstream~\citep{wang2022rethinking}. Canonicalization makes this requirement testable: surface variations may be discarded, but the distinctions needed to recover a specified clean sequence must remain accessible~\citep{tasawong2023typo,lewis2020bart}.

View-based self-supervised learning encourages agreement between related inputs~\citep{chen2020simclr,gao2021simcse}, while joint-embedding predictive architectures (JEPAs) learn by predicting latent targets~\citep{assran2023ijepa,assran2025vjepa2}. LeJEPA gives this approach a geometric foundation: under its assumptions, isotropic Gaussian representations minimize worst-case downstream risk, motivating Sketched Isotropic Gaussian Regularization (SIGReg)~\citep{balestriero2025lejepa}. VISReg pursues this target through variance and projected-shape constraints~\citep{wu2026visreg}. Learning the encoder also changes which target information its states retain. A favorable marginal distribution does not determine the conditional target distribution, so applying a geometric guarantee to sequential recovery additionally requires preserving the distinctions needed to identify the target.

Our questions are: are view agreement and isotropic Gaussian regularization sufficient to retain the ordered canonical content available in text, and how do grounding and readout affect its recovery? Answering this requires separating input ambiguity, representation loss, and readout mismatch. We first show why agreement and even exact Gaussianity leave target retention undetermined, then establish what canonical anchors and token supervision constrain. A fixed-penalty ridge analysis addresses the complementary case: information is preserved, yet its accessibility changes with geometry and target alignment.

Length controls potential content: for alphabet size $K$ and length $\ell$, sequence entropy is at most $\ell\log K$~\citep{cover1991elements,shannon1948communication}. Controlled-length (CL) experiments supply target lengths and use invertible corruption, so neither input ambiguity nor length error can explain failed recovery. Parametrized-length (PL) experiments then introduce natural corruption and predicted lengths, before trained speech readouts test downstream consequences. Figure~\ref{fig:learning-overview} summarizes this investigation, contrasting latent agreement with canonical-token grounding and illustrating how downstream readouts test whether the learned representations support recovery of the intended content and order. Our contributions are as follows:
\begin{itemize}
  \item[$\bullet$] Agreement and joint isotropic Gaussianity can coexist with zero target information. Nonconstant anchored codes retain information but can merge targets; token log-loss bounds the sufficiency gap, while ridge risk exposes target-dependent access.
  \item[$\bullet$] CANOPE's controlled comparison finds near-perfect recovery with token grounding and poor access with latent agreement, despite recoverable inputs. In PL, nearly equal pooled ranks conceal sharply different recovery.
  \item[$\bullet$] The access contrast persists with trained speech readouts, while external evaluations separate canonical matching from semantic generalization.
\end{itemize}

\section{Related Work}
\label{sec:related-work}

\paragraph{Joint embeddings and geometry.}
SimCLR and VICReg combine cross-view agreement with representation constraints~\citep{chen2020simclr,bardes2022vicreg}. I-JEPA predicts target image regions; data2vec extends contextual latent prediction across modalities~\citep{assran2023ijepa,baevski2022data2vec}. LeJEPA and VISReg target isotropic Gaussian geometry under explicit assumptions~\citep{balestriero2025lejepa,wu2026visreg}.

\paragraph{Task information and readouts.}
InfoMin requires useful views to preserve task information~\citep{tian2020infomin}; augmentation-based identifiability depends on assumptions about content and changing latent factors~\citep{vonkugelgen2021content}. Usable information and the Decodable Information Bottleneck formalize predictive-family constraints~\citep{xu2020usable,dubois2020decodable}. RankMe provides an empirical rank diagnostic~\citep{garrido2023rankme}.

\paragraph{Latent prediction for language.}
LLM-JEPA studies language embedding objectives, and latent prediction has been combined with masked-token reconstruction~\citep{huang2025llmjepa,boukhari2026predict}. Recent work examines conditional concentration, structured downstream geometry, and sequential latent prediction~\citep{dinh2026paradox,alvarez2026beyond,mostafa2026dynamic}. Our focus is ordered canonical content under a controlled, invertible corruption channel, together with fixed-readout constraints and regularization location.

\paragraph{Canonical and semantic representations.}
BART, TSDAE, and ByT5 establish noisy-to-clean supervision and robust representations~\citep{lewis2020bart,wang2021tsdae,xue2022byt5}; typo-robust retrieval studies misspelling invariance~\citep{tasawong2023typo}. Sentence-BERT, SimCSE, and Qwen3 provide embedding references~\citep{reimers2019sbert,gao2021simcse,zhang2025qwen3}; NegMPNet and SemAntoNeg address negation~\citep{anschutz2023negation,vahtola2022semantoneg}. PAWS tests lexical-overlap shortcuts, and ColBERT retains token interactions~\citep{zhang2019paws,khattab2020colbert}.

\section{Sequential Content Sufficiency}
\label{sec:content}
We begin with the information a downstream task needs, before asking what agreement and geometry guarantee about it. Proofs and extensions appear in Appendix~\ref{app:proofs}.

\subsection{What Is Preserved, and What Can a Readout Access?}
\label{sec:decomposition}
Let $C=(C_1,\ldots,C_L)$ be a canonical sequence over an alphabet $\mathcal A$ of size $K$, $X$ its corrupted view, and $Z_M\in\cnR^{M\times d}$ an ordered latent canvas. Side information $S$ is available to encoder and readout. CL and PL oracle controls use $S=L$, $M=L$. For PL predicted-length inference, $M$ is predicted from $X$; take $S$ constant and $Z_M$ as a variable-length object with observable shape. Detaching the length-head input changes training gradients, not inference information.

A sequence readout assigns normalized probabilities $q(C\mid Z,S)$ under the stated routing policy. Conditional entropy $H$ and mutual information $I$ use natural logarithms (nats); $\cnE$ denotes population expectation and $\cnKL$ Kullback--Leibler divergence. The notation $C\perp Z\mid(X,S)$ means conditional independence, $p(c,z\mid x,s)=p(c\mid x,s)p(z\mid x,s)$ almost everywhere, not orthogonality. Appendix~\ref{app:primer} collects the relevant identities.
\\

\begin{cnassumption}[Evaluation information]
  \label{ass:information}
  Regular conditional distributions (conditional probability laws given observed variables) exist, $H(C\mid S)<\infty$, and $C\perp Z\mid(X,S)$. The representation receives no target information beyond the input and stated side information. Readouts access only $(Z,S)$, without target-token prefixes or another target-token side channel. Evaluated expected log-losses are finite.
\end{cnassumption}
A deterministic encoder $F(X,S)$ satisfies this requirement. We define the content-sufficiency gap using conditional mutual information,
\begin{equation}
  \cnGap :=I(C;X\mid Z,S) =H(C\mid Z,S)-H(C\mid X,Z,S) =H(C\mid Z,S)-H(C\mid X,S).
  \label{eq:gap-definition}
\end{equation}
The last equality uses Assumption~\ref{ass:information}~\citep[Chapter~2]{cover1991elements}. Equivalently, $\cnGap=I(C;X\mid S)-I(C;Z\mid S)\geq0$. This existing information quantity measures what processing loses: zero gap means that $(Z,S)$ preserves the input-conditioned target posterior almost surely (Appendix~\ref{app:decomposition}).

\begin{cnlemma}[Log-loss decomposition]
  \label{lem:decomposition}
  Under Assumption~\ref{ass:information}, every normalized readout $q$ with finite expected log-loss satisfies
  \begin{equation}
  \cnE[-\log q(C\mid Z,S)] ={} \underbrace{H(C\mid X,S)}_{\text{input ambiguity}}+\underbrace{\cnGap}_{\substack{\text{representation}\\\text{loss}}}+\underbrace{\cnE_{Z,S}\cnKL\!\left(p(\cdot\mid Z,S)\!\Vert q(\cdot\mid Z,S)\right)}_{\text{readout mismatch}}.
  \label{eq:decomposition}
\end{equation}
  where $p$ is the true conditional target distribution.
\end{cnlemma}
The terms separate unavoidable input uncertainty, representation loss, and readout mismatch. A recoverable channel has $H(C\mid X,S)=0$, hence $\cnGap=H(C\mid Z,S)$. Low population log-loss upper-bounds the gap; high loss need not imply information loss because the readout may be restricted or poorly fitted~\citep{xu2020usable,dubois2020decodable}. Appendix~\ref{app:decomposition} separates family limits from fitting error.
Squared loss similarly separates input Bayes risk, its increase after encoding, and readout error (Appendix~\ref{app:squared-bridge}), connecting target preservation to downstream-risk guarantees.

\paragraph{Length and ordered content.}
For each $\ell$ with positive probability, the entropy chain rule gives
\begin{equation}
  H(C\mid L=\ell) = \sum_{t=1}^{\ell} H(C_t\mid C_{<t},L=\ell) \le \ell\log K.
  \label{eq:length-entropy}
\end{equation}
Here $C_{<t}=(C_1,\ldots,C_{t-1})$. Since $L$ is determined by $C$, $H(C)=H(L)+H(C\mid L)$ when finite. Conditioning on $L$ removes length information, not token dependence. CL fixes each example's target length, not a common corpus length (Appendix~\ref{app:primer}).

\subsection{Why Agreement and Isotropy Do Not Guarantee Content}
\label{sec:agreement}
Our counterexample imposes a stronger condition than the empirical diagnostics. Isotropic covariance means $\cnCov(R)=\sigma^2I$, $\sigma^2>0$; standard Gaussianity additionally requires $R\sim\cnN(0,I)$. We impose $\cnvec Z\sim\cnN(0,I_D)$ on the entire canvas, where $\cnvec$ stacks entries and $D=\ell d$, rather than only pooled vectors or sampled rows.

Perfect agreement means $Z^{(1)}=Z^{(2)}$ almost surely (a.s.), with zero squared disagreement $\mathcal L_{\mathrm{dis}}=\cnE\|Z^{(1)}-Z^{(2)}\|_F^2$. Even this agreement and joint Gaussianity do not force target retention.
\begin{cntheorem}[Invariant Gaussian representations without target information]
  \label{thm:counterexample}
  For every nonconstant $C\in\mathcal A^\ell$, there exist recoverable views and a shared deterministic encoder satisfying
  \begin{equation}
  H(C\mid X^{(v)})=0,\quad Z^{(1)}=Z^{(2)}\ \mathrm{a.s.},\qquad \cnvec Z^{(v)}\sim\cnN(0,I_D),\quad I(C;Z^{(v)})=0.
  \label{eq:counterexample}
\end{equation}
  Thus $\mathcal L_{\mathrm{dis}}=0$, its global minimum.
\end{cntheorem}
\paragraph{Construction and interpretation.}
Let $X^{(v)}=(C,U,N_v)$, where $U\sim\cnN(0,I_D)$ is shared but independent of $C$, and $N_v$ is independent view-specific noise. Output only $U$, reshaped as a canvas. The input determines $C$, but the representation discards it while satisfying agreement and Gaussianity. Two equiprobable word orders are therefore undecodable beyond chance. Appendix~\ref{app:counterexample} proves the claim. Exact continuous Gaussianity is impossible for a deterministic finite-string encoder; Appendix~\ref{app:finite} gives finite-string counterparts with identity covariance and arbitrarily small Wasserstein-2 distance to a Gaussian.

\begin{cnproposition}[Canonical-anchor restriction]
  \label{prop:anchor}
  Let $X^{(0)}=a(C)$ be a deterministic canonical anchor, $C$ finite-valued, and the shared encoder deterministic. Exact agreement with the anchor forces $Z=g(C)$ and $I(C;Z)=H(Z)$. Nonconstant representations retain some information, but can still merge distinct sequences.
\end{cnproposition}
The anchor has no extra shared randomness, so it excludes the zero-information construction unless the representation is constant. Conditionally independent identically distributed views given $C$ yield the same conclusion for finite-second-moment outputs (Appendix~\ref{app:anchor}). Anchoring thus restricts the counterexample; token supervision specifies which distinctions must survive. CL latent arms use a differentiable mean of selected views including the clean view, while PL0 uses a detached clean target (Section~\ref{sec:objectives}).

\section{Readout Constraints and Target-Dependent Geometry}
\label{sec:readout}
We next connect token supervision to retained distinctions, then isolate how geometry affects readout access when information is preserved.

\subsection{What Canonical Token Supervision Constrains}
\label{sec:token}
Fix $L=\ell$, $S=L$ as in CL and the PL oracle control. Token grounding predicts each canonical token from its canvas position. For a fixed head $q_W(\cdot\mid z)=\cnsoftmax(Wz+b)$, $W\in\cnR^{K\times d}$, $b\in\cnR^K$, a common logit shift leaves probabilities unchanged. Let $P_K=I_K-\mathbf1\mathbf1^\top/K$, $A=P_KW$, and $r_W=\cnrank(A)$. The population average token log-loss is
\begin{equation}
  \cnCE=\frac1\ell\cnE\sum_{t=1}^{\ell}-\log q_W(C_t\mid z_t).
  \label{eq:token-loss}
\end{equation}
\begin{cnproposition}[Observable directions and content bound]
  \label{prop:readout}
  For finite logits, fixed $(W,b)$, and any perturbation $\delta\in\cnR^d$,
  \begin{equation}
  q_W(\cdot\mid z+\delta)=q_W(\cdot\mid z)\iff A\delta=0,\qquad r_W\leq\min(d,K-1).
  \label{eq:kernel}
\end{equation}
  Under Assumption~\ref{ass:information}, $0\leq\cnGap\leq H(C\mid Z,S)\leq\ell\cnCE$.
\end{cnproposition}
The fixed head observes logit differences and is blind to the $d-r_W$ dimensional null space $\ker A$. Applying Lemma~\ref{lem:decomposition} to $\prod_tq_W(C_t\mid z_t)$ gives the bound without assuming independent true tokens. Low population loss over every target position certifies a small gap, but does not constrain all latent directions. Appendix~\ref{app:readout} proves the variable-length and curvature extensions; prefix-only supervision cannot certify the suffix.

\paragraph{Local access is stricter than full-canvas access.}
A local head sees only $z_t$, whereas a sequence readout sees all of $Z$. For independent fair bits and $Z=(C_2,C_1)$, each local prediction is at chance, but swapping positions recovers the target. Equation~\eqref{eq:local-readout} decomposes optimal local loss into full-canvas uncertainty, residual token dependence, and information at other positions. A fixed linear head is further restricted; its failure alone cannot prove information loss.

\subsection{Why Higher Effective Rank Can Help or Hurt a Fixed Readout}
\label{sec:ridge}
We rescale representation directions while preserving target information, then fit a population predictor with fixed ridge penalty. This isolates geometry--readout interaction without finite-sample estimation error.

\paragraph{The target and its coefficients.}
Let $G\sim\cnN(0,I_d)$ and $Y=\theta^\top G+\varepsilon$. Fixed $\theta$ selects target-relevant factors; $w$ denotes learned readout coefficients and $\|\cdot\|$ the Euclidean norm.
\begin{cnassumption}[Population ridge model]
  \label{ass:ridge}
  Fix $\theta\ne0$, penalty $\mu>0$, and noise with $\cnE[\varepsilon\mid G]=0$ and $\cnE\varepsilon^2=\sigma^2<\infty$. For a positive-definite covariance $\Sigma$ with total variance $\cntr\Sigma=d$, define
  \begin{equation}
  Z_\Sigma=\Sigma^{1/2}G,\qquad w_\Sigma=\arg\min_w\left\{\cnE(Y-w^\top Z_\Sigma)^2+\mu\|w\|^2\right\}.
  \label{eq:ridge-setup}
\end{equation}
\end{cnassumption}
Invertibility, $G=\Sigma^{-1/2}Z_\Sigma$, preserves target information and unrestricted risk $\sigma^2$. Fixing total variance prevents trivial penalty reduction by uniformly scaling features.

\paragraph{Spectrum, effective rank, and excess risk.}
In a covariance eigenbasis, let $\lambda_i$ be variances and $\theta_i$ target coefficients. With $p_i=\lambda_i/\cntr\Sigma$, define
\begin{equation}
  \cnRank(\Sigma)=\exp\!\left(-\sum_{i=1}^{d}p_i\log p_i\right).
  \label{eq:effective-rank}
\end{equation}
Effective rank approaches 1 for one-direction concentration and equals $d$ for equal eigenvalues. It measures neither target alignment nor Gaussian shape; unlike RankMe~\citep{garrido2023rankme}, it uses covariance eigenvalues. Ridge normal equations give excess risk, excluding the penalty:
\begin{equation}
  R_\mu :=\cnE(Y-w_\Sigma^\top Z_\Sigma)^2-\sigma^2 =\mu^2\sum_{i=1}^{d}\frac{\theta_i^2}{(\lambda_i+\mu)^2}.
  \label{eq:ridge-risk}
\end{equation}
Small-variance target directions require larger coefficients and are penalized more strongly (derivation: Appendix~\ref{app:ridge}).

\paragraph{A path toward isotropy.}
Fix the eigenbasis and split it into blocks of sizes $r$ and $d-r$, with $1\leq r<d$. Set their variances to $a=[d-(d-r)\rho]/r$ and $b=\rho$, where $0<\rho<1$. This preserves total variance $d$. As $\rho$ increases, both variances approach 1 and effective rank increases. Let $\eta=\sum_{i\leq r}\theta_i^2/\|\theta\|^2$ be the target's energy share in the first block. This analytical family is not an assumed trained-model spectrum.

\begin{cntheorem}[Target-alignment threshold]
  \label{thm:threshold}
  Under Assumption~\ref{ass:ridge}, along the fixed-basis family above,
  \begin{equation}
  \cnsign\!\left(\frac{dR_\mu}{d\cnRank}\right)=\cnsign(\eta-\eta_*),\qquad \eta_*:=\frac{r(a+\mu)^3}{r(a+\mu)^3+(d-r)(b+\mu)^3}.
  \label{eq:threshold}
\end{equation}
\end{cntheorem}
Higher effective rank lowers risk when $\eta<\eta_*$ and raises it when $\eta>\eta_*$. Equality gives zero local derivative. The threshold depends on $\rho$, so this is a local comparison, not a universal ordering of representations (proof: Appendix~\ref{app:threshold}).

\paragraph{A two-dimensional example.}
For $d=2$, $r=\mu=1$, $\sigma^2=0$, increasing $\rho$ from 0.2 to 0.8 changes the spectrum from $(1.8,0.2)$ to $(1.2,0.8)$ and effective rank from 1.3841 to 1.9601. Risk rises from 0.1276 to 0.2066 for $\theta=(1,0)^\top$, but falls from 0.6944 to 0.3086 for $\theta=(0,1)^\top$. No information is lost; variance moves away from the first target direction toward the second (Figure~\ref{fig:ridge}).

\paragraph{Scope and implication.}
Penalty retuning is excluded; excess risk vanishes as $\mu\to0$ for fixed positive-definite $\Sigma$. Yet $I_d$ uniquely minimizes worst-target risk at fixed target norm and trace $d$ (Appendix~\ref{app:minimax}). LeJEPA-type guarantees~\citep{balestriero2025lejepa} therefore coexist with target-specific differences: rank changes need not reflect information changes.

\section{CANOPE Framework}
\label{sec:framework}
CANOPE (Canonicalization-Oriented Predictive Embedding) makes these distinctions testable by mapping corrupted text to ordered latent canvases. Its common architecture separates the choice of predictive target, regularization location, and readout.

\subsection{Architecture and Canvas Length}
\label{sec:architecture}
For input $x$ containing $N$ UTF-8 bytes, a Transformer encoder produces $H\in\cnR^{N\times d}$~\citep{vaswani2017attention}. A resampler $\mathcal R_{N\to M}$ forms $M$ query vectors from $H$; sinusoidal positions $P_M$ specify their output order. A noncausal decoder attends to both these queries and the encoder states.
\begin{equation}
  H=E_{\vartheta}(x),\qquad p_t=\cnsoftmax(Wz_t+b),\qquad Z_M=D_{\vartheta}\!\left(\operatorname{LN}(\mathcal R_{N\to M}(H)+P_M),H\right).
  \label{eq:architecture}
\end{equation}
Here $\vartheta$ denotes network parameters, distinct from ridge target $\theta$, and $\operatorname{LN}$ is layer normalization. The head predicts 256 bytes and PAD. All positions are computed in parallel without target-token prefixes~\citep{gu2018nat}. Runs use width 512, four encoder blocks, two decoder blocks, and eight heads (Appendix~\ref{app:architecture}).

Oracle routing sets $M=L$; predicted routing uses $M=\max(1,\lceil\widehat L\rceil)$ from a learned head; input routing sets $M=N$ as a stress test. The resampler Gaussian-averages valid source states, with explicit identity at $M=N$ (Appendix~\ref{app:architecture}). CL removes final decoder normalization; PL retains it, so raw geometry comparisons stay within each series.

\subsection{Objectives and the CL/PL Comparisons}
\label{sec:objectives}
\paragraph{Controlled-Length (CL): Isolate the predictive objective and regularization location.}
CL uses oracle lengths. Grounding (G) averages canonical-byte negative log-probabilities; latent agreement (L) minimizes squared distance to a differentiable mean of designated views. Both add $0.01\mathcal R_{\mathrm{VIS}}$, with no application auxiliary loss (Appendix~\ref{app:mechanism}).

Pooled (P) regularizes mean canvas rows; sampled token (T) regularizes one uniformly sampled valid position per item, shared across views. CL-GP/GT/LP/LT form a two-by-two design with equal vector counts per batch; T constrains a position mixture, not the joint canvas.
The four configurations are summarized in Table~\ref{tab:e-config}.

Our VISReg adaptation penalizes coordinate means and nonunit standard deviations, then matches sorted standardized projections to Gaussian quantiles. Its center, scale, and shape constraints do not certify exact Gaussianity. Equation~\eqref{eq:visreg} specifies normalization, stop-gradients, and 256 projection directions.

\paragraph{Parametrized-Length (PL): Compare parametrized-length recipes.}
The PL setting adds objectives for length, surface invariance, and premise--negation distinctions. Its token-grounded recipe is
\begin{equation}
  \mathcal L_{\mathrm{app}}={} \mathcal L_{\mathrm{CE}}+\eta_{\mathrm e}\mathcal L_{\mathrm{empty}}+\gamma\mathcal L_{\mathrm{len}}+\alpha\mathcal L_{\mathrm{inv}}+\beta\mathcal L_{\mathrm{neg}}+\zeta\mathcal R_{\mathrm{VIS}}.
  \label{eq:application}
\end{equation}
CE supervises canonical bytes and a short PAD band; empty-tail loss discourages confident later predictions; length loss predicts $L$; pooled invariance aligns views; negation loss ranks same-content anchors above negated and unrelated anchors. We use $(\eta_{\mathrm e},\gamma,\alpha,\beta)=(0.1,1,0.5,0.3)$ and pooled VISReg weights $\zeta=0.1/0.5$ for PL2/PL3 (Appendix~\ref{app:objective}).

PL0 uses detached-anchor latent agreement, geometric weight 0.3, and detached token-head inputs: token loss trains the head but not the backbone. PL0--PL2 thus compares complete recipes (Appendix~\ref{app:application}); Table~\ref{tab:run-names} maps all variants to checkpoints.

\subsection{Readouts: Identity Retrieval and Token Reconstruction}
\label{sec:retrieval-readouts}
Retrieval matches a corrupted query to its canonical target in a fixed gallery. We compare either mean-canvas cosines or the average cosine at 16 linearly interpolated relative positions. The latter retains coarse order across different lengths without target-token alignment; it does not certify word-order understanding. Appendix~\ref{app:readout-implementation} gives the exact score.

Native reconstruction uses per-position argmax and stops at PAD. These diagnostics fit no new probes and exclude untrained CL-LP/LT token heads. Section~\ref{sec:tts} separately trains a speech readout. Each task measures access through a specified readout, not $\cnGap$ itself.

\section{Experiments and Results}
\label{sec:results}
We test the contrast in Figure~\ref{fig:learning-overview} on 40,000 validation targets per series, with one training seed and 500 paired bootstrap resamples. Pretraining uses English--Singlish G2P text~\citep{english_singlish_g2p}, without phonemes; Appendix~\ref{app:evaluation} details data and evaluation.

\begin{figure}[!htb]
    \centering
    \includegraphics[width=\linewidth]{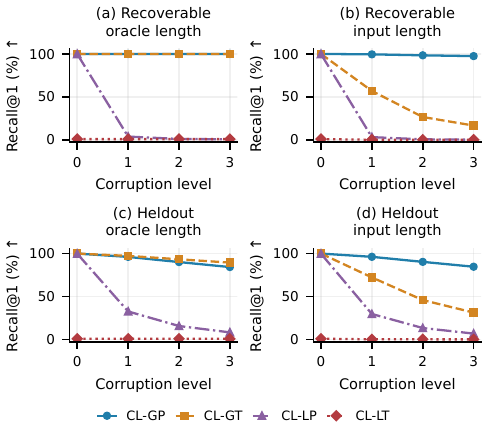}
    \caption{CL relative-position retrieval with oracle or input lengths (40,000 targets; level 0: clean). CL-GT is strongly length-sensitive (Table~\ref{tab:e-main}).}
    \label{fig:cl-length}
\end{figure}

\subsection{Controlled Content Recovery}
\label{sec:cl-results}
Under recoverable severity 3 and oracle lengths, CL-GP/GT reach 99.985\%/99.980\% retrieval and 98.97\%/98.78\% exact reconstruction, while CL-LP/LT reach only 0.52\%/0.75\% retrieval (Table~\ref{tab:e-decode}). Grounding enables recovery at either location.

Input-length routing instead leaves CL-GP/GT at 97.48\%/16.53\% retrieval (Figure~\ref{fig:cl-length}); CL-GP's 91.84\% CER separates identity matching from reconstruction. Centering CL-LT raises clean retrieval from 0.76\% to 100\%, but corrupted retrieval only from 0.75\% to 1.58\%, revealing clean distinctions without restoring cross-view access (Appendix~\ref{app:geometry-results}).

\begin{figure}[!htb]
    \centering
    \includegraphics[width=\linewidth]{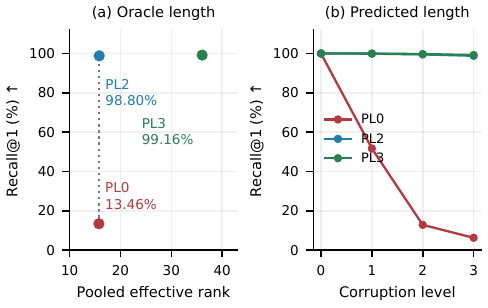}
    \caption{PL identity access despite similar pooled ranks: oracle (a) and predicted (b) lengths, 40,000 targets per query. Checkpoints and reconstruction: Table~\ref{tab:a-full}.}
    \label{fig:pl-evidence}
\end{figure}

\subsection{Natural Corruption}
\label{sec:pl-results}
With natural severity 3 and oracle lengths, PL0/PL2 have nearly equal pooled ranks (15.76/15.82), yet retrieval differs sharply (13.46\%/98.80\%; Figure~\ref{fig:pl-evidence}). Rank misses this contrast.

\begin{table}[!htb]
\centering
\normalsize\setlength{\tabcolsep}{3pt}\renewcommand{\arraystretch}{1.12}
\caption{All evaluated PL setting checkpoints at natural severity 3. Rank and reconstruction use oracle lengths; retrieval uses predicted lengths. R1, CER, and exact match (EM) are percentages.}
\label{tab:a-full}
\begin{tabular*}{\linewidth}{@{\extracolsep{\fill}}lrrrrrr@{}}
\toprule
\thead{Model} & \multicolumn{2}{c}{\thead{Eff. rank}} & \multicolumn{2}{c}{\thead{R1 $\uparrow$}} & \thead{CER $\downarrow$} & \thead{EM $\uparrow$} \\
 & \thead{Pool} & \thead{Token} & \thead{Pool} & \thead{Rel.} & & \\
\midrule
PL0 & 15.76 & 15.76 & 6.15 & 6.31 & 87.50 & 0.00 \\
PL1 & 4.79 & 9.72 & 86.76 & 97.34 & 21.38 & 12.81 \\
PL2 & 15.82 & 33.04 & 94.45 & 98.80 & 23.09 & 12.64 \\
PL3 & 36.06 & 65.11 & 97.57 & 99.19 & 22.51 & 12.53 \\
PL4 & 83.59 & 126.10 & 91.31 & 99.30 & 21.63 & 12.80 \\
PL5$^\dagger$ & 10.19 & 36.46 & 68.68 & 78.41 & 28.35 & 6.68 \\
PL6 & 7.28 & 16.89 & 82.72 & 97.24 & 21.95 & 12.68 \\
PL7 & 6.16 & 19.13 & 74.44 & 96.63 & 22.46 & 12.64 \\
\bottomrule
\end{tabular*}
\par\smallskip\begin{minipage}{\linewidth}\footnotesize PL0: latent agreement; PL1: token baseline; PL2/3: weak/stronger VISReg; PL4/5/6: no invariance/negation/empty loss; PL7: zero second margin. PL4--7 modify the historical unregularized recipe, not PL2. $^\dagger$PL5 has a different training budget and is excluded from controlled component comparisons. Checkpoint IDs: Table~\ref{tab:run-names}.\end{minipage}
\end{table}

Predicted lengths and centering leave this contrast intact. PL2's 23.09\% CER and 12.64\% exact match nevertheless limit what retrieval establishes, while PL7's positional gain shows that readout choice also matters (Table~\ref{tab:a-full}).

\subsection{External Identity and Semantic Transfer}
\label{sec:semantic-results}
LexNorm2015~\citep{baldwin2015lexnorm} tests natural noncanonical text against frozen references (Table~\ref{tab:external}).

\begin{table}[t]
\centering
\normalsize\setlength{\tabcolsep}{3pt}\renewcommand{\arraystretch}{1.12}
\caption{LexNorm identity retrieval (1,941 queries, 1,954 targets). R1: original Recall@1; WS3: strongest whitespace perturbation. CANOPE uses predicted lengths and relative positions; external readouts are in Appendix~\ref{app:evaluation}.}
\label{tab:external}
\begin{tabular*}{\linewidth}{@{\extracolsep{\fill}}lrrr@{}}
\toprule
\thead{Model} & \thead{R1 (\%) $\uparrow$} & \thead{MRR $\uparrow$} & \thead{WS3 (\%) $\uparrow$} \\
\midrule
PL0 & 23.34 & 0.281 & 1.96 \\
PL2 & 88.82 & 0.908 & 83.82 \\
PL3 & 97.37 & 0.979 & 81.20 \\
ByT5 encoder & 87.48 & 0.894 & 22.05 \\
SimCSE & 98.71 & 0.991 & 98.71 \\
NegMPNet & 99.12 & 0.994 & 99.12 \\
Qwen3-0.6B & 99.69 & 0.998 & 99.64 \\
\bottomrule
\end{tabular*}
\end{table}

PL3 leads CANOPE on original queries, but PL2 leads under whitespace perturbation. Relative positions improve both (Figure~\ref{fig:position-access}), yet weak external semantics limit generalization (Table~\ref{tab:semantic}). Strong synthetic-foil results (Table~\ref{tab:controls}) indicate sensitivity to specified changes, rather than general semantic equivalence.

\begin{table}[!htbp]
\centering
\sbox{\semanticcaptionbox}{\begin{minipage}[t]{.49\linewidth}
\caption{External semantic results on complete common cohorts. SAN (SemAntoNeg): three-choice accuracy (\%); PAWS: ROC AUC; STS-B and packaged SICK-R: Spearman correlation. CANOPE uses predicted lengths and relative positions.}
\label{tab:semantic}
\end{minipage}}
\sbox{\controlscaptionbox}{\begin{minipage}[t]{.49\linewidth}
\caption{Synthetic semantic-foil decisions (\%): 512 base groups per type, two directions per group. Neg.: negation; Num.: number; Role: role swap. Chance is 50\%. All encoders have zero failures.}
\label{tab:controls}
\end{minipage}}
\setlength{\pairedcaptionheight}{\dimexpr\ht\semanticcaptionbox+\dp\semanticcaptionbox\relax}
\ifdim\dimexpr\ht\controlscaptionbox+\dp\controlscaptionbox\relax>\pairedcaptionheight
  \setlength{\pairedcaptionheight}{\dimexpr\ht\controlscaptionbox+\dp\controlscaptionbox\relax}
\fi
\begin{minipage}[t]{.49\linewidth}
\vspace{0pt}
\vbox to \pairedcaptionheight{\usebox{\semanticcaptionbox}\vfil}
\small\setlength{\tabcolsep}{1pt}\renewcommand{\arraystretch}{1.12}
\begin{tabular*}{\linewidth}{@{\extracolsep{\fill}}lrrrr@{}}
\toprule
\thead{Model} & \thead{SAN $\uparrow$} & \thead{PAWS $\uparrow$} & \thead{STS-B $\uparrow$} & \thead{SICK-R $\uparrow$} \\
\midrule
PL0 & 26.24 & 0.488 & 0.003 & 0.143 \\
PL2 & 2.28 & 0.585 & 0.154 & 0.333 \\
PL3 & 9.11 & 0.587 & 0.177 & 0.375 \\
ByT5 encoder & 8.12 & 0.584 & 0.248 & 0.283 \\
SimCSE & 63.80 & 0.647 & 0.842 & 0.808 \\
NegMPNet & 69.57 & 0.623 & 0.842 & 0.767 \\
Qwen3-0.6B & 45.11 & 0.684 & 0.846 & 0.812 \\
\bottomrule
\end{tabular*}
\end{minipage}\hfill
\begin{minipage}[t]{.49\linewidth}
\vspace{0pt}
\vbox to \pairedcaptionheight{\usebox{\controlscaptionbox}\vfil}
\small\setlength{\tabcolsep}{1pt}\renewcommand{\arraystretch}{1.12}
\begin{tabular*}{\linewidth}{@{\extracolsep{\fill}}lrrrr@{}}
\toprule
\thead{Model} & \thead{Neg. $\uparrow$} & \thead{Num. $\uparrow$} & \thead{Entity $\uparrow$} & \thead{Role $\uparrow$} \\
\midrule
PL0 & 42.48 & 50.39 & 79.10 & 81.93 \\
PL2 & 100.00 & 90.62 & 100.00 & 100.00 \\
PL3 & 96.19 & 97.75 & 100.00 & 100.00 \\
ByT5 encoder & 85.94 & 53.42 & 67.09 & 57.71 \\
SimCSE & 100.00 & 100.00 & 100.00 & 100.00 \\
NegMPNet & 100.00 & 100.00 & 100.00 & 100.00 \\
Qwen3-0.6B & 98.63 & 97.85 & 100.00 & 57.71 \\
\bottomrule
\end{tabular*}
\end{minipage}
\par\smallskip\begin{minipage}{\linewidth}\footnotesize Cohort sizes: 3,152 groups; 8,000 pairs; 1,379 pairs; and 9,927 pairs, respectively. The SICK-derived split is the packaged mteb/sickr-sts test split, not a claim about the original SICK test-set size.\end{minipage}
\end{table}

\subsection{Speech from Frozen Canvases}
\label{sec:tts}
We train MatchaTTS~\citep{mehta2024matcha} on frozen PL0/PL2 predicted-length canvases, without decoded text, alongside an end-to-end phoneme baseline. All use identical LJSpeech pairs, 70\% corrupted inputs, and pretrained acoustic modules (Appendix~\ref{app:tts}).

\begin{table}[!htb]
\centering\normalsize\setlength{\tabcolsep}{3pt}\renewcommand{\arraystretch}{1.12}
\caption{Downstream speech WER (\% $\downarrow$); 3,930 utterances and 66,489 reference words per cell. $\Delta$: augmented minus clean WER (percentage points), reported descriptively.}
\label{tab:tts-main}
\begin{tabular*}{\linewidth}{@{\extracolsep{\fill}}lrrr@{}}
\toprule
\thead{System} & \thead{Clean $\downarrow$} & \thead{Augmented $\downarrow$} & \thead{$\Delta$} \\
\midrule
MatchaTTS & \textbf{2.21} & \textbf{10.93} & +8.71 \\
Frozen PL2 + Matcha & 13.60 & 21.54 & +7.94 \\
Frozen PL0 + Matcha & 99.12 & 99.22 & +0.09 \\
\bottomrule
\end{tabular*}
\end{table}

Whisper large-v3~\citep{radford2023whisper,openai2023whisperv3} gives corrupted WER of 21.54\% for PL2 versus 99.22\% for PL0, extending the contrast to speech (Table~\ref{tab:tts-main}). The end-to-end baseline remains stronger (10.93\%), despite PL2's smaller degradation.

\section{Discussion}
\label{sec:discussion}
The central issue is what an agreement target requires the encoder to preserve. Both CL objectives see canonical text, but the latent target is learned along with the encoder. If two sequences become indistinguishable, agreement with their encoded anchors need not separate them again. Proposition~\ref{prop:anchor} makes this precise at exact agreement. A nonconstant anchored code retains some information, yet can still merge targets. Canonical-token grounding instead penalizes distinctions lost at the supervised positions. The relevant difference is a fixed content target versus a jointly learned latent target, not cross-entropy versus squared loss itself.

CL connects this distinction to Lemma~\ref{lem:decomposition}. Invertible corruption and oracle lengths make input ambiguity zero, leaving representation loss and readout mismatch as the possible sources of failure. Near-perfect grounded reconstruction and low empirical token loss are consistent with Proposition~\ref{prop:readout}, which bounds the gap through population token loss while leaving latent null-space directions unconstrained. Failed latent retrieval cannot separate the remaining terms, but shows that agreement training left available content difficult to access through these readouts.

The centering control further explains why small disagreement is an inadequate diagnostic. CL-LT has a large common mean and tiny residual variance (Appendix~\ref{app:geometry-results}); small absolute discrepancies can then coexist with unstable content distinctions. Removing the mean exposes clean identity but barely helps corrupted views. This behavior is consistent with the scale mechanism in Eq.~\eqref{eq:agreement-scale}: shrinking residuals reduces raw disagreement without improving their separation relative to corruption. Theorem~\ref{thm:counterexample} establishes a stronger logical limit under ideal geometry, while these trained states expose a practical failure before that ideal is attained.

This also clarifies the relationship to LeJEPA. Its optimality results concern downstream risk under specified learner, target, and distributional assumptions~\citep{balestriero2025lejepa}; a learned encoder must additionally preserve the target-conditioned information to which a risk guarantee applies. Equation~\eqref{eq:squared-bridge} separates the resulting Bayes-risk increase from readout error. Our ridge analysis holds the former at zero by invertibility, and Theorem~\ref{thm:threshold} shows how target alignment still changes fixed-penalty access. Isotropy remains minimax in that model (Appendix~\ref{app:minimax}). Thus favorable worst-target geometry and poor access to a particular target are compatible. PL's matched pooled ranks support the need for target-specific evaluation; they neither certify joint Gaussianity nor instantiate the ridge spectrum path.

CL identifies a controlled access gap, PL extends it to natural corruption, and speech shows its downstream consequence. Retrieval--reconstruction differences and weak semantics delimit what canonical matching establishes. Causal attribution remains strongest in CL: PL compares complete recipes with different auxiliary losses and weights. Bootstrap intervals omit seed variation; speech uses validation data with pretraining exposure, and frontend, freezing, and alignment differences remain. ASR WER measures lexical fidelity, not naturalness. These factors motivate broader readout and pretraining-disjoint evaluations.

\section{Conclusion}
Agreement and geometric regularity alone do not guarantee recoverable sequential content. CANOPE’s controlled experiments reveal recovery failures despite recoverable inputs and known target lengths. Canonical-token grounding supports accurate recovery, with benefits extending to downstream speech. These findings motivate evaluating ordered-content recovery alongside geometric diagnostics for sequential generation.
\label{main-text-end}

\clearpage
\bibliography{references}
\bibliographystyle{iclr2027_conference}
\clearpage
\appendix
\section{Proofs and Scope of the Theoretical Results}
\label{app:proofs}
All information quantities concern the discrete target $C$, including when $Z$ is continuous. The analysis requires no differential-entropy identity. Appendix~\ref{app:decomposition} proves the content-gap decomposition; Appendix~\ref{app:squared-bridge} connects it to squared prediction risk; Appendices~\ref{app:counterexample}--\ref{app:anchor} give the Gaussian construction and its channel restrictions; Appendix~\ref{app:readout} treats token-head observability and local access; and Appendices~\ref{app:ridge}--\ref{app:minimax} derive the ridge comparison. The main text states the definitions and assumptions used to interpret each result.

\subsection{Information Decomposition and Readout Mismatch}
\label{app:decomposition}
\begin{proof}[Proof of Lemma~\ref{lem:decomposition}]
  The conditional Markov property gives $H(C\mid X,Z,S)=H(C\mid X,S)$. Hence
  \begin{equation}
  I(C;X\mid Z,S) =H(C\mid Z,S)-H(C\mid X,S) =I(C;X\mid S)-I(C;Z\mid S).
  \label{eq:app-gap}
\end{equation}
  Finite $H(C\mid S)$ makes these differences well defined. Adding and subtracting $\log p(C\mid Z,S)$ yields
  \begin{align}
    \cnE[-\log q(C\mid Z,S)]
     & =\cnE[-\log p(C\mid Z,S)]+\cnE\log\frac{p(C\mid Z,S)}{q(C\mid Z,S)}\nonumber \\
     & =H(C\mid Z,S)+\cnE_{Z,S}\cnKL(p(\cdot\mid Z,S)\Vert q(\cdot\mid Z,S)).
  \end{align}
  Substituting~\eqref{eq:app-gap} proves~\eqref{eq:decomposition}. A readout assigning zero probability to a positive-probability target event has infinite loss and is excluded by the assumption. Nonnegativity of the three terms gives the stated bounds. Moreover,
  \begin{equation}
    I(C;X\mid Z,S)=\cnE_{X,Z,S}\cnKL(p(\cdot\mid X,S)\Vert p(\cdot\mid Z,S)),
  \end{equation}
  so $\cnGap=0$ is equivalent to preserving the input-conditioned target posterior almost surely.
\end{proof}

For a readout family $\mathcal Q$ containing a finite-loss predictor, define $L(q)=\cnE[-\log q(C\mid Z,S)]$, $R_{\mathcal Q}=\inf_{q\in\mathcal Q}L(q)$, and $\Delta_{\mathcal Q}=\inf_{q\in\mathcal Q}\cnE\cnKL(p\Vert q)$. Every fitted $\widehat q\in\mathcal Q$ satisfies
\begin{equation}
  L(\widehat q)=H(C\mid X,S)+\cnGap+\Delta_{\mathcal Q}+[L(\widehat q)-R_{\mathcal Q}].
\end{equation}
The final terms separate approximation limits from excess population loss of the fitted predictor; neither infimum needs to be attained. An empirical estimate additionally has sampling uncertainty. This distinction motivates readout-family comparisons without interpreting a failed probe as a proof of information loss.

\subsection{From Target Preservation to Squared Prediction Risk}
\label{app:squared-bridge}
The information decomposition has a standard squared-loss counterpart that makes its connection to downstream-risk guarantees explicit. Let $T$ be a square-integrable vector target satisfying $T\perp Z\mid(X,S)$, and define $m_X=\cnE[T\mid X,S]$ and $m_Z=\cnE[T\mid Z,S]$. For any square-integrable readout $g(Z,S)$,
\begin{equation}
  \cnE\|T-g(Z,S)\|^2 ={} \underbrace{\cnE\|T-m_X\|^2}_{\text{input Bayes risk}} +\underbrace{\cnE\|m_X-m_Z\|^2}_{\substack{\text{increase in Bayes risk}\\\text{after encoding}}}+\underbrace{\cnE\|m_Z-g(Z,S)\|^2}_{\substack{\text{readout error above the}\\\text{representation Bayes risk}}}.
  \label{eq:squared-bridge}
\end{equation}
\begin{proof}
  Write $T-g=(T-m_X)+(m_X-m_Z)+(m_Z-g)$. Conditional independence gives $\cnE[T\mid X,Z,S]=m_X$, so $T-m_X$ is orthogonal in $L^2$ to both remaining terms. The tower property gives $\cnE[m_X\mid Z,S]=m_Z$, making the last two terms orthogonal as well. Expanding the squared norm proves the identity.
\end{proof}
For finite-valued $C$, choose $T$ to be its one-hot encoding. Then $m_X$ and $m_Z$ are the input- and representation-conditioned target probability vectors. The middle term vanishes exactly when these posteriors agree almost surely, equivalently $\cnGap=0$ under Assumption~\ref{ass:information}. Its value is not $\cnGap$: it measures squared posterior discrepancy rather than expected KL divergence. For a general vector target, it instead certifies preservation of the conditional mean relevant to squared loss.

This identity separates two ways in which a representation can affect downstream prediction. Geometry can change a learner's error above the representation Bayes risk, while a learned encoder can also change that Bayes risk by discarding target information. The ridge family in Section~\ref{sec:ridge} makes the latter increase zero through invertibility, isolating the former effect. The counterexample below makes target information vanish despite exact agreement and Gaussianity. Thus the two analyses address complementary terms rather than competing claims about the same risk.

\subsection{Invariant Gaussian Counterexample}
\label{app:counterexample}
\begin{proof}[Proof of Theorem~\ref{thm:counterexample}]
  Choose $U\sim\cnN(0,I_D)$ independently of $C$, and view-specific $N_v$ independent of $(C,U)$. Set $X^{(v)}=(C,U,N_v)$ and $F(c,u,n)=\operatorname{reshape}_{\ell\times d}(u)$. The first input coordinate determines $C$, whereas the encoder retains only the shared second coordinate. Therefore $H(C\mid X^{(v)})=0$, $Z^{(1)}=Z^{(2)}$ almost surely, $\cnvec Z^{(v)}\sim\cnN(0,I_D)$, and $I(C;Z^{(v)})=0$. Disagreement is nonnegative and equals zero, proving global optimality for that objective.
\end{proof}

Independence implies $p(C\mid Z)=p(C)$. Thus every readout, including randomized ones with randomness independent of $C$, has sequence error at least $1-\max_c p(c)$; the optimal log-loss is $H(C)$. For two equiprobable orders $(a,b)$ and $(b,a)$, these are $1/2$ and $\log2$. Each row $z_t$ is standard Gaussian, as is $\ell^{-1/2}\sum_tz_t$. Ordinary mean pooling instead has covariance $I_d/\ell$; this scaling is distinct from normalizing vector length.

For a finite-variance vector target $T$ and $m(Z)=\cnE[T\mid Z]$, expansion and conditional centering give
\begin{equation}
  \cnE\|T-h(Z)\|^2=\cnE\|T-m(Z)\|^2+\cnE\|m(Z)-h(Z)\|^2.
\end{equation}
The cross term vanishes because $\cnE[T-m(Z)\mid Z]=0$. If $T$ is the one-hot encoding of $C\perp Z$, the Bayes risk is $1-\sum_cp(c)^2>0$. A perfectly fitted conditional mean therefore need not imply preservation of the original target. Likewise, the Gaussian construction is a population existence result, not a claim that a finite-batch regularizer returns zero or that optimization finds this encoder.

\subsection{Finite-String Construction}
\label{app:finite}
A deterministic encoder of countably many strings has countable image, which a nondegenerate Gaussian assigns probability zero. Exact equality of these laws is impossible. We instead construct, for every $\epsilon>0$, a finite code $V$ with zero target information, exact identity covariance, and $W_2(\mathcal L(V),\cnN(0,I_D))<\epsilon$.

Let $G\sim\cnN(0,I_D)$ independently of $C$. For an integer $m\geq1$, clip coordinates to $[-m,m]$ and round to a grid of spacing $1/m$, symmetrically at ties. The resulting $Q_m(G)$ has finite support, zero mean, and covariance $s_m^2I_D$ with $s_m>0$, because identical odd coordinate maps preserve symmetry and independence and each coordinate remains nonconstant. The $L^2$ triangle inequality gives
\begin{equation}
  \|Q_m(G)-G\|_{L^2}\leq\frac{\sqrt D}{2m}
  +\|\cnclip(G,-m,m)-G\|_{L^2}\longrightarrow0.
\end{equation}
The clipping term converges by dominated convergence using $\cnE\|G\|^2<\infty$. Consequently $s_m\to1$. Set $V_m=Q_m(G)/s_m$. Then $\cnE V_m=0$, $\cnCov(V_m)=I_D$, and
\begin{equation}
  \|V_m-G\|_{L^2}\leq|s_m^{-1}-1|\|Q_m(G)\|_{L^2}+\|Q_m(G)-G\|_{L^2}\longrightarrow0.
\end{equation}
This coupling bounds $W_2(\mathcal L(V_m),\cnN(0,I_D))$ by $\|V_m-G\|_{L^2}$. Choose $m$ large enough for the prescribed $\epsilon$.

Let $J_m$ index the finite support of $V_m$. Encode $(C,J_m,N_v)$ injectively as a finite string, with finite-valued nuisances $N_v$. Such an encoding exists because the tuple space is finite. The input decoder recovers $C$; the representation decoder returns the vector indexed by $J_m$. Sharing $J_m$ across views gives exact agreement, and $J_m\perp C$ gives zero target information. This establishes the finite-string counterpart of Theorem~\ref{thm:counterexample} on a synthetic channel.

Approximation is in Wasserstein-2, not total variation or KL: the discrete law remains singular to the Gaussian. For any unit $u$ and frequency $\omega$, however,
\begin{equation}
  \left|\cnE e^{\mathrm i\omega u^\top V_m}-e^{-\omega^2/2}\right|
  \leq |\omega|\cnE\|V_m-G\|\leq|\omega|\|V_m-G\|_{L^2}.
\end{equation}
The first inequality follows from $|e^{\mathrm i a}-e^{\mathrm i b}|\leq|a-b|$. Squaring and integrating on bounded frequencies shows convergence of the population characteristic-function diagnostic in Appendix~\ref{app:mechanism}. Finite-sample fluctuations remain present even under an exact Gaussian null.

\subsection{Canonical Anchors and Conditionally Independent Views}
\label{app:anchor}
\begin{proof}[Proof of Proposition~\ref{prop:anchor}]
  For $X^{(0)}=a(C)$ and a deterministic shared encoder, exact agreement forces $Z^{(v)}=F(a(C))=:g(C)$. Since $C$ is finite-valued, $H(Z\mid C)=0$ and $I(C;Z)=H(Z)$. Zero mutual information implies zero entropy and an almost-sure constant. Nonconstant $g$ must retain some information, but need not be injective: if $C$ contains two independent bits and $g$ retains only one, one bit remains unrecoverable.
\end{proof}

More generally, let $C=(A,B)$ encode two independent, nonconstant finite factors as an equal-length canonical sequence. An agreeing encoder $Z=h(A)$ with injective $h$ satisfies
\begin{equation}
  I(C;Z)=H(A)>0,\qquad H(C\mid Z)=H(B)>0.
  \label{eq:anchor-partial}
\end{equation}
Every view can contain enough information to recover $(A,B)$, while its encoder retains only $A$. Agreement with a clean anchor then holds exactly because the anchor is encoded in the same way. The example establishes partial retention under anchoring; it imposes no Gaussian marginal. Its relevance to CL is that access to clean text during training does not itself force the latent target to distinguish every canonical sequence.

If views are conditionally independent and identically distributed given $C$, and their encoded outputs have finite second moments, then
\begin{equation}
  \cnE[\|Z^{(1)}-Z^{(2)}\|_F^2\mid C] =2\cnE[\|\cnvec Z\|^2\mid C]-2\|\cnE[\cnvec Z\mid C]\|^2 =2\cntr\cnCov(\cnvec Z\mid C).
\end{equation}
Conditional independence factors the cross moment. Zero disagreement forces the nonnegative conditional variance to vanish, hence $Z=g(C)$ almost surely and the preceding argument applies. Shared nuisance in Theorem~\ref{thm:counterexample} violates this conditional-independence premise.

\subsection{Fixed Token Heads and Local Readouts}
\label{app:readout}
\begin{proof}[Proof of Proposition~\ref{prop:readout}]
  For finite logits $u$, the softmax probabilities satisfy $\log(p_i/p_j)=u_i-u_j$. Two softmax outputs therefore agree exactly when their logits differ by a constant multiple of $\mathbf1$. For logits $W(z+\delta)+b$ and $Wz+b$, this is $P_KW\delta=0$. Since $\cnrank(P_K)=K-1$, $r_W\leq\min(d,K-1)$, and rank--nullity gives $\dim\ker A=d-r_W$.

  The product $q_{\mathrm{prod}}(C\mid Z,S)=\prod_tq_W(C_t\mid z_t)$ is normalized because each factor sums to one. Its population log-loss is $\ell\cnCE$. Lemma~\ref{lem:decomposition} gives
  \begin{equation}
    \ell\cnCE=H(C\mid Z,S)+\cnE\cnKL(p\Vert q_{\mathrm{prod}})
    \geq H(C\mid Z,S)\geq\cnGap\geq0.
  \end{equation}
  This model factorization imposes no independence assumption on the true tokens. For variable $L$ and $S=L$, condition on each length and average, obtaining $H(C\mid Z,L)\leq\cnE\sum_{t=1}^L-\log q_W(C_t\mid z_t)$. Prefix-only supervision bounds prefix uncertainty only.
\end{proof}

The complete probability vector identifies $Az$ through $P_K\log p=Az+P_Kb$; a scalar CE value need not. For one label $c$,
\begin{equation}
  \nabla_z(-\log p_c)=W^\top(p-e_c),\qquad
  \nabla_z^2(-\log p_c)=W^\top(\operatorname{diag}(p)-pp^\top)W.
\end{equation}
For any direction $\delta$, the Hessian quadratic form equals $\operatorname{Var}_{k\sim p}((W\delta)_k)$. Since all probabilities are positive, it vanishes exactly when $W\delta$ is constant across classes, equivalently $A\delta=0$. Thus the Hessian has kernel $\ker A$ and rank $r_W$. These are fixed-head statements; joint training can change the observable subspace.

At fixed length $\ell$, let $Z_{-t}$ denote all canvas rows except $z_t$. Each unrestricted local factor is optimized by its local posterior, so
\begin{align}
  \inf_{\{q_t\}}\cnE[-\log\textstyle\prod_t q_t(C_t\mid z_t,S)]
   & =\sum_tH(C_t\mid z_t,S)\nonumber                                                      \\
   & =\sum_tH(C_t\mid Z,S)+\sum_tI(C_t;Z_{-t}\mid z_t,S)\nonumber                          \\
   & =H(C\mid Z,S)+\cnTC(C\mid Z,S)+\sum_tI(C_t;Z_{-t}\mid z_t,S).\label{eq:local-readout}
\end{align}
Here $\cnTC(C\mid Z,S)=\sum_tH(C_t\mid Z,S)-H(C\mid Z,S)$ is conditional total correlation, the dependence remaining among target tokens after observing the canvas. It is nonnegative because it equals the expected KL divergence between $p(C\mid Z,S)$ and $\prod_tp(C_t\mid Z,S)$. The positional terms measure additional information from other positions beyond $z_t$, including possible synergistic dependencies. With independent fair bits and $Z=(C_2,C_1)$, the full canvas determines $C$, but the local optimum is $2\log2$; a sequence readout can swap the positions and recover $C$ exactly. A shared linear head is more restricted than the unrestricted factors and may incur additional loss.

\subsection{Population Ridge Risk}
\label{app:ridge}
Under Assumption~\ref{ass:ridge}, $\cnE[Z_\Sigma Z_\Sigma^\top]=\Sigma$ and $\cnE[Z_\Sigma Y]=\Sigma^{1/2}\theta$, since $\cnE[G\varepsilon]=\cnE[G\cnE(\varepsilon\mid G)]=0$. Thus
\begin{equation}
  J(w)=\cnE Y^2-2w^\top\Sigma^{1/2}\theta+w^\top(\Sigma+\mu I)w.
\end{equation}
Its positive-definite Hessian $2(\Sigma+\mu I)$ gives the unique solution $w_\Sigma=(\Sigma+\mu I)^{-1}\Sigma^{1/2}\theta$. Functions of $\Sigma$ commute, so
\begin{equation}
  Y-w_\Sigma^\top Z_\Sigma =\varepsilon+\theta^\top[I-\Sigma(\Sigma+\mu I)^{-1}]G =\varepsilon+\mu\theta^\top(\Sigma+\mu I)^{-1}G.
\end{equation}
The noise cross term is zero. Taking the expectation of the squared residual gives $\sigma^2+\mu^2\theta^\top(\Sigma+\mu I)^{-2}\theta$, proving~\eqref{eq:ridge-risk}. Since $G=\Sigma^{-1/2}Z_\Sigma$, all representations yield Bayes risk $\sigma^2$. Excess risk arises from the fixed ridge penalty, not information loss. The comparison fixes $\theta$ in the factor coordinates; fixing coefficients in $Z_\Sigma$ coordinates while changing $\Sigma$ would instead change the target.

\subsection{Effective-Rank Threshold with All Intermediate Steps}
\label{app:threshold}
The basis and the target vector $\theta$ remain fixed while $\rho$ varies. In particular, $q_0=\|\theta\|^2$ and $\eta=\sum_{i\leq r}\theta_i^2/q_0$ are constants in the derivatives below. The two eigenvalues are
\begin{equation}
  a(\rho)=\frac{d-(d-r)\rho}{r},\qquad b(\rho)=\rho,\qquad
  a'=-\frac{d-r}{r},\qquad b'=1.
  \label{eq:app-ab-derivatives}
\end{equation}
There are $r$ copies of $a$ and $d-r$ copies of $b$. Their sum is $d$, so their normalized weights are $a/d$ and $b/d$. Restating the definition in Eq.~\eqref{eq:effective-rank} makes the dependence on $\rho$ explicit,
\begin{equation}
  h(\rho) :=\log\cnRank(\Sigma(\rho)) =-r\frac ad\log\frac ad-(d-r)\frac bd\log\frac bd,\qquad \cnRank(\Sigma(\rho)) =\exp(h(\rho)).
  \label{eq:app-reff-function}
\end{equation}
\paragraph{Constructing a path toward isotropy.}
Fix the eigenbasis; split its directions into blocks of sizes $r$ and $d-r$, $1\leq r<d$. Give the second block variance $b=\rho\in(0,1)$. The trace constraint determines first-block variance $a$:
\begin{equation}
  \underbrace{ra}_{\text{first block}}+\underbrace{(d-r)\rho}_{\text{second block}}=d \Longrightarrow\quad a=\frac{d-(d-r)\rho}{r},\qquad b=\rho.
  \label{eq:two-block-construction}
\end{equation}
The spectrum is $(\underbrace{a,\ldots,a}_{r},\underbrace{b,\ldots,b}_{d-r})$. As $\rho$ increases, $a>1>b>0$ both approach 1; every map remains invertible. This analytical family is not an assumed trained-model spectrum.

Let $q_0=\|\theta\|^2$ and $\eta=\sum_{i\leq r}\theta_i^2/q_0$. Fixed target and basis keep $\eta$ constant, so Eq.~\eqref{eq:ridge-risk} becomes
\begin{equation}
  R_\mu(\rho)=\mu^2q_0\left[\frac{\eta}{(a+\mu)^2}+\frac{1-\eta}{(b+\mu)^2}\right].
  \label{eq:two-block-risk}
\end{equation}
Increasing $\rho$ raises the first contribution and lowers the second. Their balance gives the threshold below.
\begin{proof}[Proof of Theorem~\ref{thm:threshold}]
  \textbf{Step 1. Differentiate the effective-rank function.}
  For any positive differentiable $x(\rho)$, the product and chain rules give
  \begin{equation}
    \frac{d}{d\rho}[x\log x]=x'\log x+x\frac{x'}x=x'(\log x+1).
  \end{equation}
  Apply this identity to $x=a/d$ and $x=b/d$, noting that $d$ is fixed,
  \begin{align}
    h'(\rho)
     & =-\frac{ra'}d\left(\log\frac ad+1\right)
    -\frac{(d-r)b'}d\left(\log\frac bd+1\right)\nonumber \\
     & =\frac{d-r}{d}\left(\log\frac ad+1\right)
    -\frac{d-r}{d}\left(\log\frac bd+1\right)\nonumber   \\
     & =\frac{d-r}{d}\log\frac ab>0.
    \label{eq:app-rank-derivative}
  \end{align}
  The last inequality follows from $d-r>0$ and $a>b>0$ for $0<\rho<1$. Differentiating the exponential in Eq.~\eqref{eq:app-reff-function} then gives
  \begin{equation}
    \frac{d\cnRank}{d\rho}=e^{h(\rho)}h'(\rho)=\cnRank\,h'(\rho)>0.
    \label{eq:app-rank-positive}
  \end{equation}
  Thus increasing $\rho$ strictly increases effective rank on the open interval.

  \textbf{Step 2. Differentiate prediction risk.}
  Grouping Eq.~\eqref{eq:ridge-risk} by the two blocks yields
  \begin{equation}
    R_\mu(\rho)=\mu^2q_0\left[\eta(a+\mu)^{-2}+(1-\eta)(b+\mu)^{-2}\right].
  \end{equation}
  Since $d(x+\mu)^{-2}/d\rho=-2x'(x+\mu)^{-3}$,
  \begin{equation}
  R_\mu'(\rho) =\mu^2q_0\left[-\frac{2\eta a'}{(a+\mu)^3} -\frac{2(1-\eta)b'}{(b+\mu)^3}\right] =2\mu^2q_0\left[\frac{\eta(d-r)}{r(a+\mu)^3} -\frac{1-\eta}{(b+\mu)^3}\right].
  \label{eq:app-risk-derivative}
\end{equation}
  The positive term is the cost of reducing variance in the first block; the negative term is the benefit of increasing variance in the second block.

  \textbf{Step 3. Solve for the sign change.}
  Multiply the bracket by $r(a+\mu)^3(b+\mu)^3>0$. This preserves its sign and gives
  \begin{align}
     & \eta(d-r)(b+\mu)^3-(1-\eta)r(a+\mu)^3\nonumber                      \\
     & \quad=\eta\bigl[(d-r)(b+\mu)^3+r(a+\mu)^3\bigr]-r(a+\mu)^3\nonumber \\
     & \quad=\bigl[(d-r)(b+\mu)^3+r(a+\mu)^3\bigr](\eta-\eta_*),
  \end{align}
  where
  \begin{equation}
    \eta_*=
    \frac{r(a+\mu)^3}{r(a+\mu)^3+(d-r)(b+\mu)^3}.
  \end{equation}
  All prefactors are positive. Therefore $\cnsign(R_\mu')=\cnsign(\eta-\eta_*)$.

  \textbf{Step 4. Change from $\rho$ to effective rank.}
  Because Eq.~\eqref{eq:app-rank-positive} is strictly positive, effective rank is locally invertible as a function of $\rho$, and
  \begin{equation}
    \frac{dR_\mu}{d\cnRank}
    =\frac{dR_\mu/d\rho}{d\cnRank/d\rho}
    =\frac{R_\mu'(\rho)}{\cnRank\,h'(\rho)}.
  \end{equation}
  Its denominator is positive, so the sign is again $\cnsign(\eta-\eta_*)$, proving the theorem.
\end{proof}

The threshold tends to $r/d$ as $\rho\uparrow1$. It can change relative to a fixed $\eta$ along the path, which is why the theorem states a local derivative. At $\rho=1$, $h'=0$ and the ratio above is not asserted. For $\mu\to0$ at a fixed positive-definite covariance, the excess risk tends to zero.

For $d=2$, $r=1$, and $\mu=q_0=1$, moving from $\rho=0.2$ to $0.8$ changes the eigenvalues from $(1.8,0.2)$ to $(1.2,0.8)$ and effective rank from $1.3841$ to $1.9601$. Risk increases from $0.1276$ to $0.2066$ for $Y=G_1$, and decreases from $0.6944$ to $0.3086$ for $Y=G_2$. Figure~\ref{fig:ridge} plots these exact population quantities.

\begin{figure}[!htbp]
  \centering
  \includegraphics[width=\linewidth]{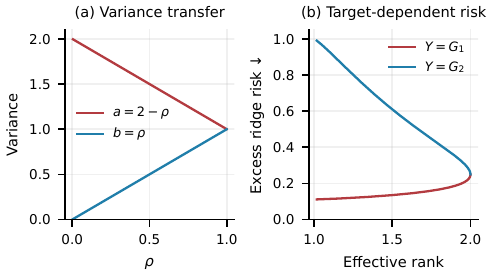}
  \caption{The same increase in effective rank can help or hurt a fixed readout. (a) Increasing $\rho$ transfers variance from the first direction to the second while keeping $a+b=2$. (b) Ridge risk rises when the target uses the first direction and falls when it uses the second. The map remains invertible for every $0<\rho<1$, so both targets remain fully represented. These are exact population curves for $d=2$, $r=1$, $\mu=q_0=1$, and zero observation noise, illustrating Theorem~\ref{thm:threshold}.}
  \label{fig:ridge}
\end{figure}

\subsection{Minimax Compatibility and Scope Relative to LeJEPA}
\label{app:minimax}
For fixed $\Sigma$, $R_\mu=\sum_i b_i\theta_i^2$ with $b_i=\mu^2/(\lambda_i+\mu)^2$. Under $\sum_i\theta_i^2=q_0$, the maximum is $q_0\max_i b_i$, attained by placing all energy in a minimum-eigenvalue direction. This proves the main-text supremum. The trace constraint implies $\lambda_{\min}\leq1$, and the supremum decreases strictly with $\lambda_{\min}$. Equality $\lambda_{\min}=1$ forces every eigenvalue to equal one, hence uniquely $\Sigma=I_d$. This is uniqueness of covariance within the assumed Gaussian family, not uniqueness among arbitrary distributions.

LeJEPA's linear results compare fixed feature spans and energy. Its nonlinear results assume smooth embedding densities and regression functions; the radial-neighbor result additionally uses a mean-zero isotropic target-gradient prior, while the kernel result controls a worst-case integrated-bias bound over a smooth function class~\citep{balestriero2025lejepa}. Atomic laws induced by deterministic string encoders do not automatically meet the density assumptions. Likewise, the fixed ridge target has gradient $\Sigma^{-1/2}\theta$, whose outer product is rank one for $d>1$, rather than a positive isotropic gradient second moment. These examples delimit direct applicability; continuous encoder noise or a different target prior changes the assumptions. Our information decomposition addresses target retention, whereas the ridge analysis preserves information and isolates geometry--readout interaction.

\section{Implementation and Evaluation Details}
\label{app:implementation}

\subsection{Readout Implementation}
\label{app:readout-implementation}
Retrieval matches a corrupted query to its canonical target in a fixed gallery. Pooling compares mean-canvas cosines; relative-position scoring preserves coarse order. For $Z\in\cnR^{M\times d}$, linearly interpolate $\widetilde z_j$ at $j(M-1)/15$ and define
\begin{equation}
  \phi_{16}(Z)=\frac1{4}\operatorname{concat}_{j=0}^{15} \frac{\widetilde z_j}{\max(\|\widetilde z_j\|,10^{-12})},\qquad s_{16}(U,V)=\langle\phi_{16}(U),\phi_{16}(V)\rangle.
  \label{eq:position}
\end{equation}
The factor $1/4=1/\sqrt{16}$ makes the inner product a mean cosine for nonzero rows. This compares lengths without target-token alignment but does not preserve every detail or certify word-order understanding.

Native reconstruction uses per-position argmax and stops at PAD, requiring the whole sequence rather than candidate selection. These diagnostics fit no new probes and exclude untrained CL-LP/LT heads. Section~\ref{sec:tts} separately trains a downstream speech readout. Each task measures specified readout access, not $\cnGap$ itself.

\subsection{Architecture and Length Policies}
\label{app:architecture}
The application architecture uses width 512, eight attention heads, four encoder blocks, two decoder blocks, feed-forward width 2048, and a 2048-position cap. Source masks exclude padding from resampling and attention; output masks exclude batch padding from losses. Valid queries originate from resampled source states, not zero padding. Application states retain final decoder layer normalization; mechanism runs remove it. With zero-based source and query indices, the resampler uses weights $\omega_{tj}\propto\exp(-(j-c_t)^2/0.3)$ normalized over valid source positions, where $c_t=t(N-1)/(M-1)$ for $M>1$ and $c_0=0$ for $M=1$. It returns the input unchanged when $M=N$. Gaussian diagnostics on raw mechanism states therefore do not transfer automatically to application states.

$U\{a,\ldots,b\}$ and $U[a,b]$ below denote discrete and continuous uniform draws. The prediction $\widehat L$ is the learned target length. The length head attention-pools detached encoder states and uses the source log length. Training samples near-oracle $M=L+U\{0,\ldots,8\}$, overestimated $M=\lceil LU[1.15,1.80]\rceil$, or underestimated $M=\lceil LU[0.50,0.95]\rceil$, clipped to the supported range. At progress $0,0.15,0.40,0.65,1$, the corresponding mixture probabilities are $(1,0,0)$, $(1,0,0)$, $(0.5,0.5,0)$, $(0.35,0.45,0.20)$, and $(0.20,0.60,0.20)$, with smoothstep interpolation. Only lengths, not target-token identities, enter decoder routing.

The reported native-head evaluation uses a single canvas, per-position argmax, and stopping at the first PAD; it does not expand and retry. PL setting practical routing uses $M=\max(1,\lceil\widehat L\rceil)$ for both queries and clean candidates. CL setting input routing uses $M=N$, the number of input bytes, as a length-sensitivity control rather than a trained length predictor. Oracle routing uses $M=L$. Capacity violations fail explicitly rather than truncate. The same extracted states feed retrieval and the native head. Oracle token cross-entropy covers every canonical target byte; empirical corpus loss per byte is reported descriptively, not substituted directly into the population sequence-entropy bound.

\subsection{Recoverable Channel and VISReg}
\label{app:mechanism}
The channel escapes literal \texttt{\#} as \texttt{\#\#} and inserts removable \texttt{\#\textasciitilde} markers; held-out patterns use \texttt{\#?} at word boundaries. A left-to-right inverse removes insertion markers and reverses escaping without the corruption seed. Thus case, punctuation, numbers, and canonical bytes remain recoverable. Natural spelling errors form a separate channel because unique recovery can fail. Canonical-content hashes determine splits and group all related views.

The seven configured views comprise two deterministic clean copies, three easy corruptions, and two stronger corruptions. The global-view indices are $\mathcal G=\{0,2\}$, indexed from zero. The available training entry point sets \texttt{global\_views=(0,2)}, and the loss implementation computes \texttt{z[:, global\_views].mean(1)} without detaching that center. This differs from PL0's detached view-0 target. For batch size $B$, view count $V$, target byte $c_{it}$, token probability $p_{ivt}$, state $z_{ivt}$, and target length $L_i$, let $\bar z_{it}=|\mathcal G|^{-1}\sum_{v\in\mathcal G}z_{ivt}$. The implemented predictive terms are
\begin{align}
  \mathcal L_{\mathrm{lat}}    & =\frac1{BV}\sum_{i,v}\frac1{L_id}\sum_{t=1}^{L_i}\|z_{ivt}-\bar z_{it}\|^2, \\
  \mathcal L_{\mathrm{ground}} & =\frac1{BV}\sum_{i,v}\frac1{L_i}\sum_{t=1}^{L_i}-\log p_{ivt}(c_{it}).
\end{align}
The latent center receives gradients. The latent runs have no canonical-token CE contribution to backbone training. The grounded runs use the token head on undetached states. Neither branch includes the application auxiliary losses. Table~\ref{tab:e-config} specifies the training configurations.

Pooled regularization uses $r_{iv}=L_i^{-1}\sum_{t=1}^{L_i}z_{ivt}$. Sampled-token regularization draws $T_i$ uniformly from $\{1,\ldots,L_i\}$ and sets $r_{iv}=z_{ivT_i}$, using the same position across all views of an item. Thus both use $B$ vectors per view. Token sampling targets a sentence-balanced mixture of valid positions, not Gaussianity at every position or of the joint canvas. Position and projection draws use separate seeded generators. VISReg uses 256 unit projections, coordinate standard deviations with divisor $B$, unit center/scale/shape weights, and the detached standard-deviation floor $0.1$ in Eq.~\eqref{eq:visreg}. Finite-batch Gaussian samples need not attain zero penalty.

\paragraph{Geometry diagnostic.}
The evaluation code separately computes a characteristic-function discrepancy. For $n$ vectors $r_i$, $J$ unit directions $u_j$, frequency $\omega$, and imaginary unit $\mathrm i$,
\begin{equation}
  D_{\mathrm{CF}}=\frac1{3J}\sum_{j=1}^J\int_{0}^{3}
  \left|\frac1n\sum_{i=1}^n e^{\mathrm i\omega u_j^\top r_i}-e^{-\omega^2/2}\right|^2\,d\omega.
  \label{eq:cf-diagnostic}
\end{equation}
The integral uses trapezoidal quadrature at 31 equally spaced frequencies; $J=32$ unit directions are seeded for evaluation. This is a diagnostic, not the CL setting training loss or a calibrated $p$-value. It has no sample-size multiplier. Evaluation also records covariance eigenvalues, covariance effective rank, and deviation from identity, with sampling units and finite-sample Gaussian references. This distinction permits assessment of geometry independently of the VISReg training objective.

No new readout is fitted in the reported evaluation. Native heads are assessed only where training supplied a token objective: PL0's detached-input head is included, while the untrained heads of CL-LP/CL-LT are excluded. Geometry uses all 40,000 sentence means and one deterministic valid token per sentence, before post-hoc centering or whitening. The token statistic is a position mixture, not a joint-canvas Gaussianity test. A separate set of 1,024 nonoverlapping validation examples defines the centering vector as the average of clean sentence means; no label fitting or test-dependent centering is used. Finite-sample Gaussian reference files exist in the evaluation output, but no calibrated Gaussian goodness-of-fit conclusion is inferred from the summary statistics reported here.

\subsection{Application Objective and Regularization}
\label{app:objective}
Let $v=0$ denote the clean anchor, $v>0$ its training corruptions (natural corruptions are not certified invertible), and $M_{iv}$ the requested canvas length. The coefficients $\eta_{\mathrm e},\gamma,\alpha,\beta,\zeta$ weight empty-tail, length, invariance, negation, and geometry terms, respectively. The combined application objective is Eq.~\eqref{eq:application} in the main text.

\paragraph{Content, termination, and length.}
Set $s_{iv}=\min(M_{iv},L_i+4)$, $y_{it}=c_{it}$ for $t\leq L_i$ and PAD otherwise, and $\mathcal T=\{(i,v,t):s_{iv}<t\leq M_{iv}\}$. Then
\begin{equation}
  \mathcal L_{\mathrm{CE}}=\frac1{BV}\sum_{i,v}\frac1{s_{iv}}\sum_{t\leq s_{iv}}-\log p_{ivt}(y_{it}),\qquad
  \mathcal L_{\mathrm{empty}}=\frac1{|\mathcal T|}\sum_{(i,v,t)\in\mathcal T}\sum_{k=1}^{257}p_{ivtk}^2.
  \label{eq:ce-empty}
\end{equation}
Here $p_{ivtk}$ is the predicted probability of class $k$ at output position $t$. CE grounds ordered bytes and a short PAD band. Empty-tail loss is zero for $\mathcal T=\varnothing$; otherwise its minimum $1/257$ encourages uniform tail predictions. Underestimated canvases supervise only a prefix and do not yield a full-content guarantee. For source length $N_{iv}$, a head using detached encoder states predicts the log length ratio $\rho_{\mathrm{len},iv}$ with
\begin{equation}
  \mathcal L_{\mathrm{len}}=\tfrac12\operatorname{mean}_{i,v}|\log L_i-\log N_{iv}-\rho_{\mathrm{len},iv}|,
  \qquad \widehat L=N\exp\!\bigl(\cnclip(\rho_{\mathrm{len}},-2.5,2.5)\bigr).
  \label{eq:length}
\end{equation}
$\cnclip(x,a,b)$ clips $x$ to $[a,b]$. This symmetric log-ratio loss trains the length head; routing is detached. PAD-based text generation and direct latent extraction use distinct length policies (Appendix~\ref{app:architecture}).

\paragraph{Surface invariance and polarity.}
Let $r_{iv}$ average the first $\min(L_i,M_{iv})$ states and $q_{iv}=r_{iv}/\max(\|r_{iv}\|,\epsilon)$, where $\epsilon>0$ prevents division by zero. Both sides of pooled invariance receive gradients:
\begin{equation}
  \mathcal L_{\mathrm{inv}}=\frac1{B(V-1)}\sum_{i,v>0}(1-\langle q_{iv},q_{i0}\rangle),\qquad
  \ell_{P\to N}^{(v)}=[s_--s_++m_1]_+ +[s_U-s_-+m_2]_+.
  \label{eq:inv-neg}
\end{equation}
Write $[a]_+=\max(a,0)$ and let $m_1,m_2$ be hinge margins; $\ell_{P\to N}^{(v)}$ is the per-view directional negation loss. For a premise--negation pair $(P,N)$, $s_+=\langle q_{Pv},q_{P0}\rangle$, $s_-=\langle q_{Pv},q_{N0}\rangle$, and $s_U$ averages similarity to detached unrelated anchors. $\mathcal L_{\mathrm{neg}}$ averages both directions over pairs and corrupted views, with $m_1=m_2=0.2$. The hinges favor $s_+>s_->s_U$: negations retain topical relatedness but are not positive views. This ranking bias uses paired data without a semantic teacher; invariance alone does not enforce order.

\paragraph{Geometry and configurations.}
For vectors $r_{iv}\in\cnR^d$, our VISReg adaptation~\citep{wu2026visreg} uses coordinate means $\mu_v$ and standard deviations $\sigma_v$ (divisor $B$), then matches standardized projections to Gaussian quantiles:
\begin{equation}
  \mathcal R_{\mathrm{VIS}}=\operatorname{mean}_{v,k}[\mu_{vk}^2+(\sigma_{vk}-1)^2]
  +\operatorname{mean}_{v,j,i}\left[a_{vj,(i)}-\Phi^{-1}\!\left(\frac{i}{B+1}\right)\right]^2,
  \label{eq:visreg}
\end{equation}
Here $k$ indexes coordinates, $j$ indexes projection directions, and $(i)$ denotes sorted batch order. The standardized state $\widetilde r_{iv}=(r_{iv}-\mu_v)/\max(\cnsg(\sigma_v),0.1)$ uses coordinatewise division and stop-gradient $\cnsg$; $a_{vj,(i)}$ sorts $u_j^\top\widetilde r_{iv}$ for unit $u_j$, and $\Phi$ is the standard Gaussian CDF. Center, scale, and shape have unit internal weights; 256 directions are sampled per call and shared across views. Application runs use pooled content states. CL setting runs use pooled or sampled-token states with outer weight $0.01$ (Appendix~\ref{app:mechanism}). This single definition applies to both series.

We use $(\eta_{\mathrm e},\gamma,\alpha,\beta)=(0.1,1,0.5,0.3)$. The historical unregularized recipe and PL2 set $\zeta=0/0.1$, respectively. The former retains all nongeometric losses. PL0 instead uses $\zeta=0.3$, detached-anchor latent agreement of weight one, and detached inputs to the token head. Its CE/empty terms train the head but not the backbone. Thus PL0--PL2 compares full recipes, whereas the historical unregularized recipe and PL2 differ only in the configured geometric term (Appendix~\ref{app:application}).

\subsection{Application Losses and Gradient Paths}
\label{app:application}
Application runs use the same VISReg definition and internal settings as Eq.~\eqref{eq:visreg}, applied to pooled available-content states. Their outer weights differ from the CL setting: PL0 uses $\zeta=0.3$, PL1 uses $0$, and PL2 uses $0.1$. Mean pooling during training uses the available canonical-content mask; inference does not receive the true query length for this purpose.

For PL0, canvas lengths are shared across views of an item. Let $\mathcal M=\{(i,v,t):v>0,1\leq t\leq M_{iv}\}$. Its additional latent term is
\begin{equation}
  \mathcal L_{\mathrm{PL0,lat}}=\frac1{d|\mathcal M|}
  \sum_{(i,v,t)\in\mathcal M}\|z_{ivt}-\cnsg(z_{i0t})\|^2.
\end{equation}
PL0 uses logits $W\cnsg(z)+b$, so both CE and empty-tail losses update the head only. Invariance, negation, geometry, and latent agreement still affect its backbone. PL1/PL2 use $Wz+b$; CE and empty-tail losses can therefore update their backbones. Length-head inputs are detached in every configuration. Invariance differentiates through the corrupted view and clean anchor, while unrelated negation anchors are detached. PL0 is neither the auxiliary-free mechanism arm nor a faithful reproduction of LeJEPA's complete training recipe.

The mean squared probability sum in~\eqref{eq:ce-empty} satisfies $\sum_{k=1}^{257}p_k^2\geq1/257$ by Cauchy--Schwarz, with equality at uniform probabilities. It suppresses confident predictions after the supervised PAD band and does not itself train PAD. Content accuracy excludes the PAD band; sequence metrics count missing suffixes when canvases are too short. An underestimated-canvas CE loss cannot certify sufficiency of the entire target.

\begin{table}[ht]
\centering\small\setlength{\tabcolsep}{3pt}\renewcommand{\arraystretch}{1.12}
\caption{CL setting design. All four use oracle target lengths and VISReg weight 0.01. Only G sends canonical-token loss gradients into the representation.}
\label{tab:e-config}
\begin{tabular*}{\linewidth}{@{\extracolsep{\fill}}lll@{}}
\toprule
\thead{Run} & \thead{Predictive objective} & \thead{Vectors regularized} \\
\midrule
CL-GP & Canonical-token grounding & Mean of canvas rows \\
CL-GT & Canonical-token grounding & One sampled row per item \\
CL-LP & Cross-view latent agreement & Mean of canvas rows \\
CL-LT & Cross-view latent agreement & One sampled row per item \\
\bottomrule
\end{tabular*}
\end{table}
\subsection{Executed Retrieval and Semantic Protocols}
\label{app:evaluation}
Within each PL/CL suite, candidate identities and paired queries are fixed across all severities. Full-gallery search uses unit-normalized pooled states or Eq.~\eqref{eq:position}; the latter interpolates 16 relative positions and has no absolute-position length-mismatch penalty. The same routing applies to queries and clean candidates. No dynamic time warping, target-token matching, or best permutation enters the score. Centered controls subtract the clean calibration mean before constructing the same vectors. All reported failures remain in the query denominator; a failed query receives zero retrieval credit. No input is silently truncated. All reported PL/CL retrieval conditions have zero failures and full 40,000-candidate coverage.

If $g$ candidates score strictly above a tie of size $e$, the relevant item's expected Recall@$k$ contribution is $\min(\max(k-g,0),e)/e$ and reciprocal rank is $e^{-1}\sum_{r=g+1}^{g+e}1/r$. Exact score ties can yield fractional counts. This is particularly relevant for CL-LT, whose nearly common states produce poor raw cosine discrimination even on clean self-retrieval. Centering is estimated once on the disjoint 1,024-example calibration cohort, rather than fitted to individual evaluation targets. The changed-query LexNorm confidence intervals resample its 1,928 canonical identities, retaining duplicate-associated queries. Intervals for multiple contrasts are descriptive and are not familywise-corrected tests.

\paragraph{Data provenance.}
Training data come from the English--Singlish G2P compilation~\citep{english_singlish_g2p}. Its documented source inventory combines LJSpeech transcriptions (\texttt{LJ}) by Keith Ito and Linda Johnson~\citep{ito2017ljspeech}, LibriSpeech transcriptions (\texttt{LIBR}) by Vassil Panayotov and colleagues~\citep{panayotov2015librispeech}, community-authored English Wikipedia text (\texttt{WIKI})~\citep{wikipedia}, the NUS SMS corpus (\texttt{NUS}) by Tao Chen and Min-Yen Kan~\citep{chen2013nus}, and CANNOT negation data (\texttt{CANNOT}) by Miriam Ansch\"utz and colleagues~\citep{anschutz2023negation}. CANNOT is itself a derived resource produced with a rule-based negation tool. The compilation combines and processes existing resources; it does not confer authorship of their original sentences, annotations, or recordings. We credit the compilation and these source creators separately and claim no new collection of the underlying text.

The corpus documentation describes more than ten million text--phoneme pairs with canonical English text, informal/Singlish variants, and negation pairs. The compilation was assembled by the authors and is available at \url{https://huggingface.co/datasets/avalonai/english-singlish-g2p}. The prepared-data counts below describe the combined corpus used in our experiments. Component sources retain their own attribution and licensing terms.

CANOPE pretraining and the text diagnostics use only graphemic text, encoded as UTF-8 bytes, rather than the corpus's phoneme modality. The text-side fields are \texttt{text} (canonical sentence), \texttt{unnormalized\_text} (informal/Singlish variant), and \texttt{negation}; a literal \texttt{-} marks an unavailable variant or negation. Although the corpus also provides eSpeak NG transcriptions in \texttt{phoneme} and \texttt{phoneme\_negation}, neither field supplies model inputs, targets, or auxiliary supervision. The legacy loader nevertheless requires a nonempty phoneme field when filtering rows; reproduction must retain this field to preserve the selected cohort. These experiments concern text representations and canonical-byte recovery, not G2P prediction. The separate downstream MatchaTTS baseline uses its own phonemizer, not the compilation's stored phoneme annotations (Appendix~\ref{app:tts}). The corpus-level size describes the source collection; the filtered PL and prepared CL counts below describe the data used in the reported protocols.

The PL loader reproduces the original filtered corpus and seeded validation split: 9,595,548 filtered rows, 8,635,993 training rows, and 959,555 validation rows. The CL prepared corpus has 8,656,553 training, 470,530 validation, and 468,465 test rows, split by the hash of exact canonical text; duplicates remain in one partition. Each suite samples 40,000 unique targets with selection seed 9028 (evaluation seed 9027 plus one), and an additional 1,024 calibration targets with zero overlap. The CL source manifest reports an inverse-channel audit over 4,096 accepted examples and all strengths of both marker channels; the evaluation channel also checks recovery when generating its views. These are validation-cohort analyses, and the PL evaluation does not rerun a complete train/validation content-overlap audit. No numerical result from a separate test partition is substituted for these cohorts.

\paragraph{External extraction.}
Published checkpoints are \path{princeton-nlp/sup-simcse-bert-base-uncased} (supervised pooler output), \path{tum-nlp/NegMPNet} (published sentence-transformers modules), \path{Qwen/Qwen3-Embedding-0.6B} (last non-padding token), and \path{google/byt5-small} (masked mean of frozen encoder states). All vectors are unit-normalized. The ByT5 result uses neither generation nor task-specific fine-tuning. The configured input cap is 512 tokenizer tokens, further limited by the model interface; NegMPNet's recorded cap is 75. Capacity checks precede encoding, and the recorded external comparisons have full valid coverage. CANOPE uses width-512 canvases and a 2,048-byte/position cap; Eq.~\eqref{eq:position} produces an 8,192-coordinate vector, so parameter count alone would not establish retrieval-storage efficiency.

Qwen's canonical retrieval queries use the instruction: ``Retrieve the canonical text with the same meaning as the query, preserving negation, entities, numbers, and word order.'' Candidates and symmetric semantic pairs are encoded without that query instruction. For every model, LexNorm targets are deduplicated annotated canonical strings and are held fixed when query whitespace changes. Added whitespace is a controlled surface perturbation and may disappear under a subword tokenizer; robustness differences therefore concern the complete representation pipeline. The primary changed-input subset avoids credit for already canonical queries, while the complete original cohort remains in the accompanying results.

\paragraph{Semantic benchmarks and synthetic controls.}
SemAntoNeg supplies 3,152 three-choice groups (9,456 scored pairs), using the annotated equivalent candidate as the positive. PAWS uses the 8,000-example \texttt{labeled\_final/test} split and reports threshold-free ROC AUC. STS-B uses the 1,379-pair \path{sentence-transformers/stsb} test split, and the SICK-derived result uses the 9,927-pair \path{mteb/sickr-sts} split named \texttt{test}. We label the latter by its packaged source rather than imply the original SICK test cardinality. Both report Spearman correlation. The finite-pair intersection contains every example for these runs, so common-cohort metrics retain the full recorded coverage.

Synthetic controls use 512 base groups per type (negation, number, entity, and role swap), each in both directions, giving 1,024 decisions per type. A noisy query chooses between its same-content anchor and a specified semantic foil; ties receive half credit. Whitespace noise is applied identically to all methods. These template controls do not establish open-domain semantic equivalence or isolate word order from every other cue. We do not report unexecuted state-shuffle, random-encoder, or newly fitted readout controls as empirical findings.

\begin{table}[!htbp]
\centering\small
\caption{Run names and their source identities. Short IDs in the text refer to this convention; raw result files retain the original checkpoint keys.}
\label{tab:run-names}
\begin{tabular*}{\linewidth}{@{\extracolsep{\fill}}lll@{}}
\toprule
\thead{Run} & \thead{Source ID} & \thead{Defining setting} \\
\midrule
PL0/JEPA-based & \texttt{a0\_jepa\_s42} & Detached-anchor latent agreement \\
PL1/Baseline & \texttt{a10\_dec2\_s42} & Token supervision, $\zeta=0$ \\
PL2/VISReg-Weak & \texttt{a2\_visreg\_s42} & Pooled VISReg, $\zeta=0.1$ \\
PL3/VISReg-Stronger & \texttt{a3\_visreg\_kuat\_s42} & Pooled VISReg, $\zeta=0.5$ \\
PL4/No-Invariance & \texttt{a4\_no\_inv\_s42} & $\alpha=0$ \\
PL5/No-Negation & \texttt{a5\_no\_neg\_s42} & $\beta=0$ \\
PL6/No-Empty & \texttt{a6\_no\_empty\_s42} & $\eta_{\mathrm e}=0$ \\
PL7/Zero-Second-Margin & \texttt{a9\_single\_hinge\_s42} & $m_2=0$ \\
\bottomrule
\end{tabular*}
\end{table}
The historical keys \path{a1_ce_s42}, \path{a7_no_soft_s42}, and \path{a8_no_uniform_s42} are distinct from the revised PL1 and PL7.

\paragraph{Checkpoint and compute scope.}
All training runs used a combination of NVIDIA A100 GPUs with 80\,GB memory and NVIDIA H100 GPUs with 96\,GB memory. These are the training devices; the forward-only evaluation infrastructure is reported separately below.

The recorded PL0/PL2/PL3 external checkpoints share width 512, eight heads, four encoder layers, two decoder layers, and step 1,199,440. Their optimizer settings are AdamW, learning rate $5\times10^{-5}$, weight decay 0.01, batch size 36, and five configured epochs. Table~\ref{tab:run-names} defines the revised run names and their source IDs. PL4, PL5, PL6, and PL7 modify the historical unregularized recipe, not PL2. PL1/Baseline is the evaluated checkpoint with source key \texttt{a10\_dec2\_s42}, with two decoder layers, batch size 36, and eight unrelated negatives. The historical unregularized checkpoint was not evaluated; it is distinct from the renamed PL1. In PL7, $m_2=0$ removes the second positive margin, not necessarily its hinge term. PL5 is retained only in the historical measurement record. Its evaluated checkpoint has a different recorded training budget, so it is excluded from controlled component comparisons.

Evaluation uses PyTorch 2.11.0+cu128 and NumPy 2.2.6. PL and the external benchmark ran on an A100 40GB; CL ran on an H100 NVL reporting 93.1 GiB. No hardware-speed comparison is made. PL/CL use forward-only extraction with batch size 4; external encoders use batch size 16; retrieval query chunks contain 64 vectors. No backbone or readout parameters are updated. Numeric summaries, paired intervals, figure-generation code, and the compiled manuscript accompany the LaTeX source.

\FloatBarrier
\section{Additional Empirical Results}
\begin{table}[t]
\centering
\small\setlength{\tabcolsep}{3pt}\renewcommand{\arraystretch}{1.12}
\caption{PL setting at natural severity 3. Rank and CER use oracle routing; Recall@1 (R1) uses relative positions. Percentages except rank.}
\label{tab:a-main}
\begin{tabular*}{\linewidth}{@{\extracolsep{\fill}}lrrrrr@{}}
\toprule
\thead{Model} & \thead{Pooled rank} & \thead{Token rank} & \thead{Oracle R1 $\uparrow$} & \thead{Predicted R1 $\uparrow$} & \thead{CER $\downarrow$} \\
\midrule
PL0/JEPA-based & 15.76 & 15.76 & 13.46 & 6.31 & 87.50 \\
PL2/VISReg-Weak & 15.82 & 33.04 & 98.80 & 98.80 & 23.09 \\
PL3/VISReg-Stronger & 36.06 & 65.11 & 99.16 & 99.19 & 22.51 \\
\bottomrule
\end{tabular*}
\end{table}

\begin{table}[t]
\centering
\small\setlength{\tabcolsep}{3pt}\renewcommand{\arraystretch}{1.12}
\caption{CL setting relative-position Recall@1 (\%) at severity 3. Input routing uses source length, not a learned length predictor.}
\label{tab:e-main}
\begin{tabular*}{\linewidth}{@{\extracolsep{\fill}}lrrrr@{}}
\toprule
\thead{Model} & \thead{Recov. oracle $\uparrow$} & \thead{Recov. input $\uparrow$} & \thead{Heldout oracle $\uparrow$} & \thead{Heldout input $\uparrow$} \\
\midrule
CL-GP & 99.985 & 97.475 & 84.278 & 84.665 \\
CL-GT & 99.980 & 16.532 & 89.472 & 31.220 \\
CL-LP & 0.520 & 0.090 & 8.183 & 6.759 \\
CL-LT & 0.747 & 0.002 & 0.715 & 0.014 \\
\bottomrule
\end{tabular*}
\end{table}

\label{app:results}
Tables~\ref{tab:a-full}, \ref{tab:external}, \ref{tab:semantic}, and~\ref{tab:controls} in the main paper report all PL checkpoints and external comparisons. Tables~\ref{tab:a-main}, \ref{tab:e-main}, and~\ref{tab:e-decode} below add oracle/predicted PL comparisons and CL retrieval and reconstruction details. The accompanying machine-readable summaries retain all measured severities and readouts.

\begin{table}[!htbp]
\centering
\small\setlength{\tabcolsep}{3pt}\renewcommand{\arraystretch}{1.12}
\caption{Native-head reconstruction at E severity 3. CER and exact match (EM) are percentages; CE is corpus nats per target byte under oracle routing. CL-LP/CL-LT heads were not trained and are excluded.}
\label{tab:e-decode}
\begin{tabular*}{\linewidth}{@{\extracolsep{\fill}}llrrrrr@{}}
\toprule
\thead{Model} & \thead{Channel} & \thead{Or. CER $\downarrow$} & \thead{Or. EM $\uparrow$} & \thead{Or. CE $\downarrow$} & \thead{In. CER $\downarrow$} & \thead{In. EM $\uparrow$} \\
\midrule
CL-GP & recoverable & 0.13 & 98.97 & 0.0094 & 91.84 & 0.09 \\
CL-GP & heldout & 18.13 & 6.86 & 16.3316 & 17.62 & 6.83 \\
CL-GT & recoverable & 0.15 & 98.78 & 0.0109 & 91.70 & 0.09 \\
CL-GT & heldout & 19.08 & 6.86 & 12.8150 & 17.06 & 6.83 \\
\bottomrule
\end{tabular*}
\end{table}

\subsection{Semantic Generalization and Synthetic Controls}
\label{app:semantic-results}
Canonical identity retrieval and semantic equivalence exhibit markedly different behavior. Although PL3 achieves 100\% accuracy on synthetic role-swap controls, its SemAntoNeg accuracy is only 9.11\%. Thus, successful discrimination of controlled content changes does not establish general semantic equivalence recognition. The evidence supports CANOPE's utility for canonical identity matching while revealing limitations in how the evaluated similarity functions expose semantic relationships. Tables~\ref{tab:semantic} and~\ref{tab:controls} in the main paper retain the complete reported semantic comparisons, including weak performance. Their protocols appear in Appendix~\ref{app:evaluation}.

\FloatBarrier
\subsection{Raw Geometry and Its Interpretation}
\label{app:geometry-results}
\begin{table}[!htbp]
\centering
\small\setlength{\tabcolsep}{3pt}\renewcommand{\arraystretch}{1.12}
\caption{Raw CL geometry at recoverable severity 3, oracle routing. $\bar v$ is mean coordinate variance, $D_\Sigma=\|\widehat\Sigma-I\|_F/\sqrt d$, and $D_{\rm CF}$ is Eq.~\eqref{eq:cf-diagnostic}. Arrows indicate closer agreement with the zero-mean, identity-covariance Gaussian reference, not better content access; $\bar v$ has target 1. Token statistics use one valid token per sentence.}
\label{tab:e-geometry}
\begin{tabular*}{\linewidth}{@{\extracolsep{\fill}}llrrrrr@{}}
\toprule
\thead{Model} & \thead{States} & \thead{Eff. rank} & \thead{$\|\mu\|$ $\downarrow$} & \thead{$\bar v$} & \thead{$D_\Sigma$ $\downarrow$} & \thead{$D_{\rm CF}$ $\downarrow$} \\
\midrule
CL-GP & Pool & 55.08 & 6.94 & 0.73763 & 2.456 & 0.03175 \\
CL-GP & Token & 41.59 & 6.94 & 26.37623 & 109.941 & 0.15207 \\
CL-GT & Pool & 114.14 & 1.04 & 0.02654 & 0.977 & 0.40322 \\
CL-GT & Token & 194.47 & 1.05 & 0.82981 & 1.375 & 0.00280 \\
CL-LP & Pool & 5.36 & 32.53 & 0.04498 & 1.103 & 0.69940 \\
CL-LP & Token & 15.57 & 32.53 & 0.06938 & 1.101 & 0.65024 \\
CL-LT & Pool & 3.68 & 543.35 & 0.00038 & 1.000 & 1.18631 \\
CL-LT & Token & 28.59 & 543.35 & 0.00266 & 0.997 & 1.17967 \\
\bottomrule
\end{tabular*}
\end{table}

Table~\ref{tab:e-geometry} separates pooled and sampled-token diagnostics. Isotropic covariance requires $\Sigma=\sigma^2I$; the standard Gaussian target additionally requires zero mean, unit variance, and Gaussian shape. Low effective rank indicates unequal eigenvalues. A covariance discrepancy near one can also result from near-zero variance. These distinct diagnostics complement the content tests.

CL-GT has substantially lower token characteristic-function discrepancy than CL-GP (0.00280 versus 0.15207), but CL-GT's token covariance deviation from identity remains 1.375 and its effective rank is 194.47. Improved projected shape diagnostics therefore do not justify replacing measured geometry with an assumption of exact isotropy. CL-LT has low variance, low effective rank, and a dominant common mean; its near-zero raw agreement MSE must be interpreted with those scales. Mean-centering improves clean identity access without demonstrating preservation of all ordered bytes. The theoretical Gaussian construction and these learned-state diagnostics answer different questions.

\begin{figure}[!htbp]
\centering
\includegraphics[width=\linewidth]{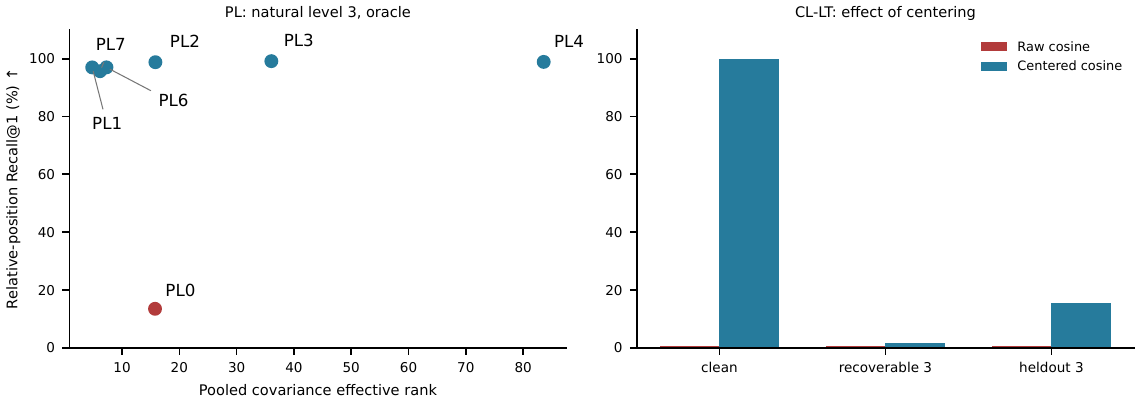}
\caption{Pooled effective rank and content access. The left panel uses PL natural severity 3 with oracle lengths for the principal configurations. The right panel shows that mean subtraction restores CL-LT clean identity access much more than cross-view retrieval. Both panels summarize full-cohort measurements.}
\label{fig:geometry-access}
\end{figure}

\FloatBarrier
\paragraph{Geometry and mean subtraction.}
CL-LT has pooled mean norm 543.35, coordinate variance $3.79\times10^{-4}$, rank 3.68, and raw cross-view MSE $6.67\times10^{-7}$: agreement occurs at a highly concentrated scale. Centering restores clean retrieval from 0.76\% to 100\%, but severity-3 retrieval rises only from 0.75\% to 1.58\% (Figure~\ref{fig:centering}). Clean access does not imply corruption robustness. The first three PCA components retain 9.90\%/11.01\% of CL-GP/GT clean variance versus 88.44\% for CL-LT. Thus similar clean/corrupted clouds can coexist with concentrated geometry and poor access.

\paragraph{Why small disagreement need not mean robust separation.}
A scale illustration clarifies these diagnostics without assuming a training trajectory. Let $Z^{(v)}=m+\epsilon[s(C)+e_v(C)]$, where $m$ is a shared vector, $s(C)$ carries content, $e_v(C)$ is a view-dependent perturbation, and $\epsilon>0$. For finite second moments,
\begin{equation}
\cnE\|Z^{(1)}-Z^{(2)}\|^2
=\epsilon^2\cnE\|e_1(C)-e_2(C)\|^2.
\label{eq:agreement-scale}
\end{equation}
As $\epsilon$ decreases, absolute disagreement vanishes, but both inter-target separation and view perturbations shrink at the same rate. Their relative scale need not improve. If $m\ne0$ is fixed, uncentered cosine similarities also approach one as $\epsilon\to0$ for finite residuals. Removing the common mean can expose clean distinctions while leaving corruption-induced mismatches. This explains how the reported diagnostics can coexist; it does not identify the learned residuals or prove that the optimizer follows this path. In particular, VISReg's variance term penalizes vanishing variance, so the illustration concerns the agreement term, not a global optimum of the combined objective.

\subsection{Paired 3D Views and Distribution Geometry}
\label{app:visual-evidence}
The visual diagnostics use a fixed 4,096-item subset, distinct from the 40,000-item full-gallery summaries. PCA is fit to clean raw states separately for each model and representation type. The same basis projects its clean and corrupted views. Distribution panels display 1,200 points; trajectory panels display 12 matched identities. Figure~\ref{fig:a-identities} uses the PL natural channel; the pooled distribution grid (Figure~\ref{fig:e-distributions}) and Figure~\ref{fig:e-identities} use CL recoverable corruption, all with oracle lengths. Panel sizes are identical within each grid. Bases and units differ across models; common axes and limits apply within each three-panel distribution row.

The pooled distribution grid shows all four CL objectives as averaged text vectors. The matched views below add identity correspondence. Appendix~\ref{app:complete-visuals} extends both views to sampled tokens, PL natural corruption, and the CL heldout channel, and indexes the complete visual atlas.

\clearpage
\begin{figure}[!ht]
\centering
\includegraphics[width=.88\linewidth]{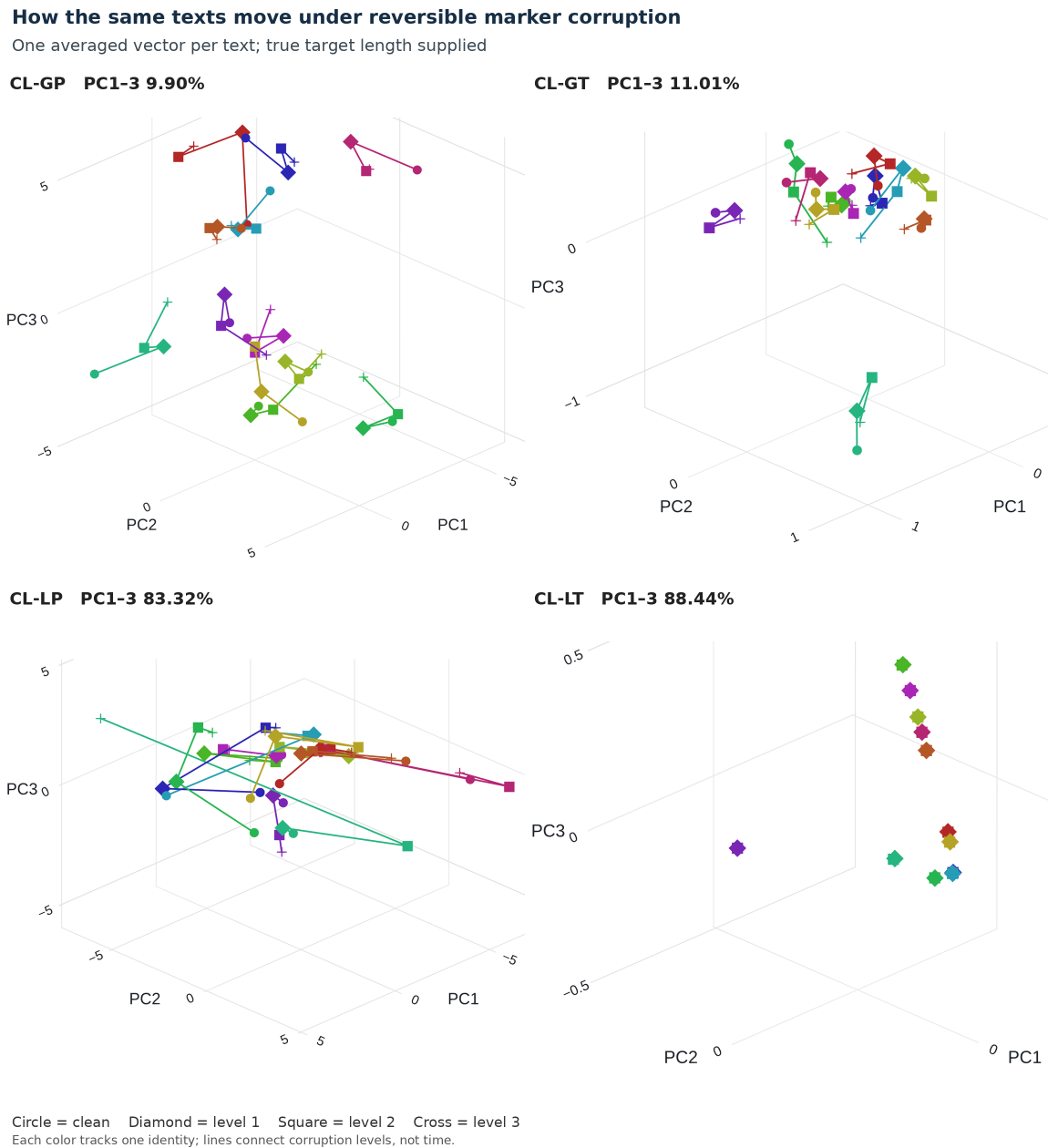}
\caption{Matched CL identities in a $2\times2$ grid. Each panel uses pooled states, oracle lengths, and its own clean PCA basis. Colors identify 12 matched texts; markers distinguish clean and recoverable severities 1--3. Lines connect views rather than time steps. CL-LP exhibits visible displacement for several identities, while some concentrated representations show short trajectories despite poor full-gallery retrieval. The figure illustrates selected within-model behavior, not an estimate of retrieval accuracy.}
\label{fig:e-identities}
\end{figure}

\paragraph{Pairwise margins and Gaussian projections.}
The supplied cosine histograms compare the correct identity with a random wrong identity. Their ECDFs use the margin $s(q,c^+)-s(q,c^-_{\mathrm{random}})$, whereas Recall@1 requires the correct candidate to outrank the full gallery. A positive random-foil margin can therefore coexist with poor Recall@1. The supplied Q--Q plots compare four seeded projections with normal quantiles; their standardized rows remove projection mean and scale. Close agreement on these directions supports the projected shape diagnostic but does not test every direction or the joint canvas. The token mixture and the 3D PCA projections also concern different sampling units and subspaces.

\clearpage
\subsection{Corruption Sensitivity and Positional Access}
\begin{figure}[!htbp]
\centering
\includegraphics[width=\linewidth]{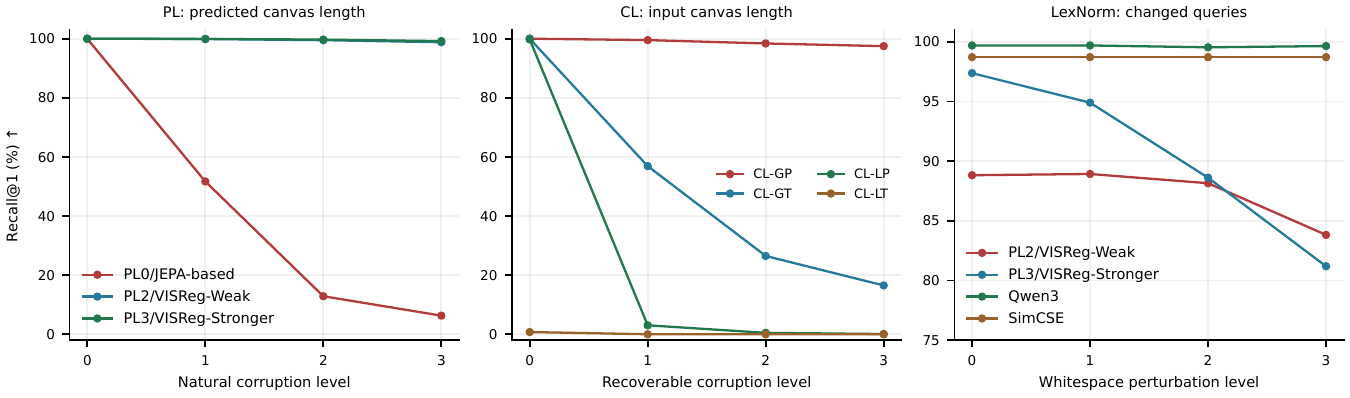}
\caption{Relative-position retrieval across corruption strengths. The panels use PL predicted lengths, CL input lengths, and LexNorm changed-input queries with extra whitespace. Level 0 denotes clean PL/CL inputs and the original noisy LexNorm query. Channels and cohorts differ, so severity scales are not directly comparable.}
\label{fig:robustness}
\end{figure}
These within-checkpoint comparisons isolate the choice of similarity readout. The benefit of relative positions is largest under the strongest whitespace perturbation. It supports the utility of canvas structure for canonical identity access without a claim that the 16-position score recovers every byte or all semantic relations.

\clearpage
\subsection{Supporting Readout Controls}

The main-text comparisons retain the original full-cohort centering and positional-access plots below.
\begin{figure}[!htbp]
\centering
\includegraphics[width=.62\linewidth]{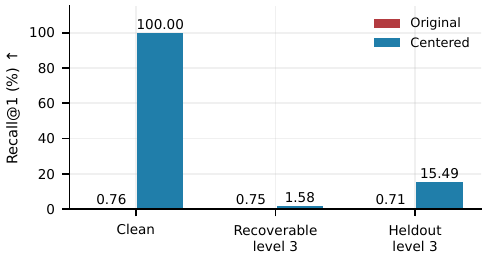}
\caption{Centering restores clean access, not robustness. CL-LT, oracle lengths, relative-position scoring, 40,000 targets. Blue subtracts the clean mean of 1,024 disjoint calibration examples; red uses raw states. Labels give Recall@1 (\%).}
\label{fig:centering}
\end{figure}

\begin{figure}[!htbp]
\centering
\includegraphics[width=.62\linewidth]{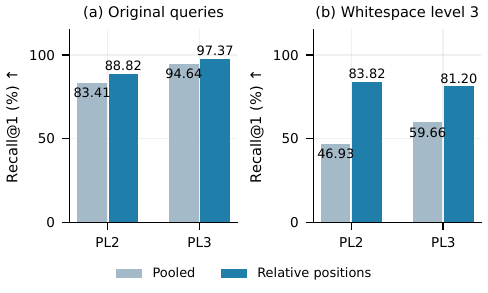}
\caption{Position improves access within a fixed checkpoint and predicted length. Gray pools states; blue compares 16 relative positions. All 1,941 changed-input queries search the same 1,954 targets. At maximum whitespace perturbation, PL2/PL3 gain 36.89/21.54 points.}
\label{fig:position-access}
\end{figure}

\clearpage
\section{Downstream Text-to-Speech: Protocol and Audit}
\label{app:tts}

\subsection{Data, Training, and the Frozen Interface}
The downstream experiment uses LJSpeech 1.1~\citep{ito2017ljspeech}, with normalized transcriptions paired to the original single-speaker recordings. The preparation code shuffles the 13,100 utterance IDs with seed 42 and makes a 70/30 split (9,170 training and 3,930 validation utterances). Each ID has a stored clean transcription and a fixed augmented version; both target the same recording. The conservative augmenter attempts up to eight medium-strength rule-based corruptions and accepts a nonempty changed string only when one minus the character-level \texttt{SequenceMatcher} ratio is at most 0.10. It otherwise retains the clean text. This bound is a string-similarity heuristic, not a certified information-preserving channel or a Levenshtein-error bound.

The training loader chooses the stored augmented branch with probability 0.7 using a hash of training seed, epoch, and utterance ID; it also fixes the per-epoch example order across systems. Validation always selects the stored augmented branch, while the reported downstream evaluation synthesizes both branches. No target audio is corrupted. We trained all three systems for 1,000 epochs using the default submission-script configuration: Adam with learning rate $10^{-4}$, zero weight decay, batch size 36, training seed 42, mixed bfloat16 precision, and gradient clipping at 5.0. The launcher selects one A100 80\,GB for the baseline and one H100 96\,GB per frozen system. The accompanying source and Slurm configuration specify these training settings.

All three systems initialize from the official LJSpeech MatchaTTS checkpoint~\citep{mehta2024matcha}. The native baseline fine-tunes its text encoder and acoustic model. Its frontend applies \texttt{english\_cleaners2} (including lowercasing and eSpeak phonemization) and intersperses blank tokens. PL0 and PL2 instead receive UTF-8 bytes without inserted blanks. The integration restores their canonical-canvas backbones from the latent-agreement and token-grounded checkpoints, respectively, disables all backbone gradients, and keeps them in evaluation mode. Canvas length is the ceiling of the model's predicted canonical length, bounded below by one; over-capacity inputs or predictions raise an error rather than being silently truncated. Neither system uses oracle target lengths, a canonical-token prefix, nor the CANOPE token head at inference.

For both frozen systems, a trainable LayerNorm and linear map project each width-512 canvas row to 192 channels. The integration reuses MatchaTTS's pretrained acoustic-mean projection, duration predictor, and flow-matching decoder. The adapter and acoustic modules train while CANOPE stays frozen; the duration predictor receives detached adapter features, as implemented in the interface. Thus this is latent conditioning of a trained speech generator, not text correction followed by an unchanged TTS model. The common downstream architecture makes PL0/PL2 an informative accessibility comparison, but their canvas lengths, state distributions, and learned alignments can differ.

Synthesis uses ten flow steps, temperature 1, length scale 1, and synthesis seed 9027, followed by a shared pretrained HiFi-GAN vocoder~\citep{kong2020hifigan} at 22,050\,Hz. Its random seed is derived from the utterance ID and input view; systems share that seed for a given ID/view, but clean and augmented views use different derived seeds. Consequently, the clean-to-augmented contrast includes the realized synthesis variation. We evaluate the generated speech with Whisper large-v3. The archived CSVs retain both raw and normalized references and hypotheses, but not the ASR command-line flags. The accompanying evaluator implements English text normalization and unprompted transcription (temperature zero, no previous-text conditioning); the exact ASR invocation is not independently documented. The audit below uses the archived normalized strings directly.

\subsection{Recomputed Word Error Rates and Paired Uncertainty}
\label{app:tts-wer}
Let $e_{imv}=S_{imv}+D_{imv}+I_{imv}$ denote the word edit count for utterance $i$, model $m$, and view $v$, and let $n_i$ be its normalized clean-reference word count. We report
\begin{equation}
\operatorname{WER}_{mv}=100\frac{\sum_i e_{imv}}{\sum_i n_i},\qquad
\Delta_m=\operatorname{WER}_{m,\mathrm{aug}}-\operatorname{WER}_{m,\mathrm{clean}},\qquad
D=\Delta_{\mathrm{PL2}}-\Delta_{\mathrm{base}}.
\label{eq:tts-wer}
\end{equation}
These are corpus ratios, not means of utterance WERs. All six cells contain the same 3,930 unique IDs, one recorded synthesis seed, and 66,489 reference words. We verified unique model/view/ID keys, identical references across the six cells, reference-token counts, every stored utterance ratio, and each total edit count against a fresh word-level Levenshtein calculation on the archived normalized strings. The recomputed corpus totals match both archived summary CSVs. There are 3,923 unique normalized reference strings; the sampling unit remains the audio utterance ID. The CSVs contain scores for every expected cell, but they do not carry synthesis/ASR failure flags for the full cohort. Complete score coverage alone therefore does not establish that every synthesis or ASR call succeeded; empty hypotheses remain in the metric.

\begin{table}[!htbp]
\centering\small\setlength{\tabcolsep}{3pt}\renewcommand{\arraystretch}{1.12}
\caption{Audited downstream error counts and utterance-paired percentile intervals. S/D/I denote substitutions, deletions, and insertions. The word denominator is 66,489 in every row.}
\label{tab:tts-full}
\begin{tabular*}{\linewidth}{@{\extracolsep{\fill}}llrrrrr@{}}
\toprule
\thead{System} & \thead{Input} & \thead{S $\downarrow$} & \thead{D $\downarrow$} & \thead{I $\downarrow$} & \thead{WER $\downarrow$ [95\% interval]} & \thead{Zero-WER clips $\uparrow$} \\
\midrule
MatchaTTS & Clean & 868 & 259 & 345 & 2.21 [2.04, 2.40] & 3,083 \\
MatchaTTS & Augmented & 4,910 & 701 & 1,654 & 10.93 [10.47, 11.40] & 1,647 \\
Frozen PL2 & Clean & 6,657 & 1,398 & 988 & 13.60 [13.21, 14.01] & 816 \\
Frozen PL2 & Augmented & 10,880 & 2,085 & 1,359 & 21.54 [21.01, 22.09] & 458 \\
Frozen PL0 & Clean & 39,958 & 24,689 & 1,259 & 99.12 [98.33, 100.10] & 0 \\
Frozen PL0 & Augmented & 39,465 & 25,231 & 1,272 & 99.22 [98.45, 100.15] & 0 \\
\bottomrule
\end{tabular*}
\end{table}

For uncertainty, we draw 10,000 bootstrap samples of the 3,930 utterance IDs with replacement using NumPy seed 9027, retaining all three systems and both views for each selected ID. We recompute the numerator and word denominator on every draw and take the 2.5th and 97.5th percentiles. This preserves pairing and estimates sampling variation conditional on the fitted checkpoints and realized synthesis/ASR outputs. It does not measure variation across training seeds, other ASR engines, or independent domains. WER can exceed 100\% because insertions are included.

The corruption increments are 8.71 points [8.26, 9.18] for the baseline, 7.94 [7.43, 8.44] for PL2, and 0.09 [$-1.03$, 1.21] for PL0. The paired contrast $D=-0.77$ points has interval [$-1.34$, $-0.22$]. PL2 nevertheless has 10.62 points more augmented WER than the baseline [10.07, 11.18]. Its advantage over PL0 is 77.67 points [76.72, 78.72]. These contrasts describe the evaluated runs; they are not multiple-comparison-adjusted hypothesis tests.

PL0's clean errors are dominated by 39,958 substitutions and 24,689 deletions, with no zero-WER utterances. Its near-zero corruption increment accompanies failure on the clean view and therefore does not demonstrate useful invariance. The baseline has 3,083 zero-WER clean clips and 1,647 augmented clips; PL2 has 816 and 458. PL0 has 44 empty normalized hypotheses on clean audio and 43 on augmented audio, versus none for the other systems. This pattern is consistent with poor linguistic accessibility under the trained interface, but it does not identify the population information gap or distinguish representation loss from optimization, alignment, length, and readout mismatch.

\subsection{Scope of the Downstream Evidence}
The native baseline is an end-to-end adaptation reference with a pretrained phoneme encoder; the CANOPE systems replace that encoder with frozen byte-based states and a newly trained adapter. A casing change may disappear before entering the baseline encoder yet remain visible to CANOPE. The comparison therefore does not isolate the causal effect of freezing or of the CANOPE objective. Likewise, the small reduction in PL2's corruption increment is conditional on its higher clean error and this particular perturbation mixture; it is not a general robustness guarantee.

The 3,930 items are held out from the reported acoustic fine-tuning split, not demonstrated to be unseen during pretraining. The shared MatchaTTS initializer was trained on LJSpeech, and LJSpeech transcriptions are among CANOPE's pretraining sources. A pretraining-to-evaluation overlap audit and original-recording Whisper WER are unavailable. We therefore avoid interpreting these numbers as out-of-domain transfer, a ground-truth intelligibility ceiling, or a clean test-set estimate. Whisper WER~\citep{radford2023whisper,openai2023whisperv3} measures ASR-mediated lexical fidelity and can include recognizer errors; no human listening score or perceptual naturalness claim is made.

\subsection{Reproducibility and Waveform Displays}

The accompanying \path{analyze_tts.py} recreates the audit, tables, bootstrap intervals, and spectrogram panels from the archived CSVs and samples. Source names \texttt{baseline}, \texttt{canope}, and \texttt{jepa} map to MatchaTTS, frozen CANOPE-PL2, and frozen CANOPE-PL0, respectively; the historical \texttt{jepa} name is not a claim that PL0 reproduces an unmodified JEPA method. The package retains the original source mapping and a machine-readable audit of counts, audio hashes, and display settings.

The final ten sample pages show every utterance in the archived sample subset, retaining its recorded order. We do not select a further subset by WER. The original selection procedure is undocumented, so these examples are illustrative rather than a random sample of the evaluation cohort. Each page includes all six waveforms as mel-spectrograms, the exact clean and augmented inputs, and the normalized Whisper hypotheses joined by utterance ID, model, and view. The plots are recomputed from the generated WAV files, not from saved decoder mel targets. All 60 WAV hashes, durations, sampling rates, and mel-frame counts match their sidecar metadata. No original reference recordings were supplied in the sample archive, so no ground-truth-audio panel is shown.

We use a 1,024-sample periodic Hann window, 256-sample hop, 384-sample reflection padding at both ends, and 80 Slaney area-normalized mel filters over 0--8\,kHz, matching the supplied Matcha acoustic transform. Magnitudes are displayed as $20\log_{10}(M/M_{\max})$, clipped to [$-80$, 0]\,dB, where $M_{\max}$ is shared across all 60 clips. Time axes are shared within each six-panel example without time warping; gray beyond a waveform's end denotes absence of samples. These plots expose timing and spectral differences but do not independently establish which words were spoken or how natural the audio sounds.

\clearpage
\section{Probability and Information Identities Used in This Paper}
\label{app:primer}
This appendix explains the identities needed for the content-gap and readout arguments. It also separates population quantities from the finite-cohort diagnostics. The information identities follow standard definitions and chain rules~\citep[Chapter~2]{cover1991elements}; the applications below use this paper's variables and examples.

\subsection{Random Variables, Expectations, and Conditional Distributions}
Uppercase letters denote random variables, while lowercase letters denote realizations. For example, $C$ is a random target and $c$ is one particular string. An expectation averages a function under the joint law of its arguments. For discrete $C$ and possibly continuous $Z$,
\begin{equation}
\cnE[f(C,Z,S)]=\cnE_{Z,S}\left[\sum_c p(c\mid Z,S)f(c,Z,S)\right].
\end{equation}
Thus $\cnE$ without a subscript is shorthand when the variables are clear. In the KL term of Eq.~\eqref{eq:decomposition}, the inner divergence already sums over $c$ at fixed $(z,s)$. The outer $\cnE_{Z,S}$ averages that divergence across contexts. Omitting this outer average would leave a context-dependent quantity rather than population log-loss.

The notation $C\perp Z\mid(X,S)$ means that after $(X,S)$ is known, $Z$ gives no additional information about $C$. It permits $C$ and $Z$ to be strongly dependent before conditioning. A deterministic $Z=F(X,S)$ satisfies it because $Z$ is then already fixed by $(X,S)$. Independent encoder randomness also satisfies it if it introduces no target-dependent side channel.

\subsection{Entropy, Log-Loss, and the Information Gap}
For a discrete target, conditional entropy is
\begin{equation}
H(C\mid Z,S)=\cnE_{Z,S}\left[-\sum_c p(c\mid Z,S)\log p(c\mid Z,S)\right].
\end{equation}
For two discrete laws, $D_{\mathrm{KL}}(p\|q)=\sum_c p(c)\log[p(c)/q(c)]$, with value $+\infty$ if $q$ assigns zero probability to a positive-probability event under $p$. It measures distributional mismatch and is zero exactly when the laws agree. Conditional entropy is uncertainty under the true conditional law. A readout $q$ instead incurs log-loss $\cnE[-\log q(C\mid Z,S)]$. Subtracting the true entropy leaves an expected KL divergence, which is nonnegative. Hence an inaccurate readout can have high loss even when the target is fully determined by $Z$.

The standard identity
\begin{equation}
I(C;X\mid Z,S)=H(C\mid Z,S)-H(C\mid X,Z,S)
\end{equation}
measures how much $X$ reduces the remaining uncertainty after $(Z,S)$ is given. Under the conditional-independence assumption, its final entropy becomes $H(C\mid X,S)$, yielding Eq.~\eqref{eq:gap-definition}. This identifies a processing loss relative to the input, not necessarily the entire uncertainty about $C$.

For a concrete example, let $C=(B_1,B_2)$ contain independent fair bits, $X=C$, and $Z=B_1$. In nats, $H(C)=2\log2$, $H(C\mid X)=0$, and $H(C\mid Z)=\log2$. The encoder loses one bit, so $\cnGap=\log2$. If instead $Z=(B_2,B_1)$, the full representation loses no information. A position-local head still fails, because it sees the wrong independent bit at each position. This is the difference between information loss and local readout mismatch.

If the input itself contains only $X=B_1$ and the encoder preserves $Z=X$, then $H(C\mid X)=\log2$ but $\cnGap=0$. The irreducible uncertainty arose before the encoder. Recoverable corruption removes this confound by ensuring $H(C\mid X,L)=0$.

\subsection{Why Length Is a Separate Information Variable}
At a given length $\ell$, the chain rule expands the entropy of the sequence into $\ell$ conditional token entropies. Each term is at most $\log K$, which proves Eq.~\eqref{eq:length-entropy}. Equality requires the uniform distribution over all $K^\ell$ strings, not ordinary linguistic text. Repetition can add almost no new uncertainty, whereas independent symbols add a constant amount per position.

For an infinite stationary finite-alphabet process with entropy rate $h$, the block entropy satisfies
\begin{equation}
h=\lim_{n\to\infty}\frac{H(C_1,\ldots,C_n)}n,
\qquad H(C_1,\ldots,C_n)=nh+o(n).
\end{equation}
This is an asymptotic statement about a source process~\citep{shannon1948communication}. It does not imply that arbitrary variable-length sentences, or their meanings, contain exactly $h$ units of information per token. Conditioning on the event $L=\ell$ can change the token law.

Because sequence length is a function of the sequence, $H(L\mid C)=0$. The joint-entropy chain rule therefore yields $H(C)=H(L)+H(C\mid L)$ when finite. CL supplies $L$ and studies the remaining ordered content. PL predicted inference must infer a length from $X$. Its output shape is itself part of the representation available to the readout. Stopping gradients through the length-head input does not supply the unknown length or change this inference-time information structure.

\subsection{Covariance, Eigenvalues, and the Quantity Called Rank}
For a random vector $R$ with finite second moments,
\begin{equation}
\Sigma=\cnCov(R)=\cnE[(R-\cnE R)(R-\cnE R)^\top].
\end{equation}
An eigenvector $u_i$ with unit norm identifies a direction whose projected variance is $u_i^\top\Sigma u_i=\lambda_i$. Equal eigenvalues mean equal variance in all directions. The eigenbasis describes directions, while the spectrum lists their variances.

Algebraic rank counts nonzero eigenvalues. Covariance effective rank instead applies entropy to their relative weights $p_i=\lambda_i/\sum_j\lambda_j$ and exponentiates the result. It is sensitive to the unevenness of the spectrum even if every eigenvalue remains positive. For example, both $\operatorname{diag}(1.8,0.2)$ and $\operatorname{diag}(1.2,0.8)$ have algebraic rank two, but their effective ranks are 1.3841 and 1.9601. This is the function differentiated in Appendix~\ref{app:threshold}.

Neither covariance isotropy nor effective rank determines the full distribution. An isotropic covariance can occur in a non-Gaussian law, and a standard Gaussian sample has unequal empirical eigenvalues at finite sample size. Our Gaussian reference uses matched sample count and dimension to contextualize this sampling effect.

\subsection{Population Guarantees and Measured Diagnostics}
Theorems use population expectations. An empirical token-loss average is an estimate, and a per-byte corpus average is not the expected summed sequence loss used in the information bound. Paired bootstrap intervals resample canonical identities while keeping each comparison paired; they describe evaluation-cohort uncertainty conditional on the checkpoint.

Recall@1 asks whether the true target outranks every candidate in the gallery. A positive cosine margin against one random wrong target is a weaker event. Similarly, a 3D principal-component plot shows only three directions. Short trajectories indicate stability in that displayed subspace; interpreting robustness also requires the retained variance, correspondence structure, and full-gallery retrieval curves.

\clearpage
\section{Complete Distribution and Paired-View Comparisons}
\label{app:complete-visuals}
The distribution and paired-view grids in this section complete the coverage across all seven models with supplied feature diagnostics. PL includes natural corruption; CL includes both recoverable and heldout corruption. Every combination is shown for pooled and sampled-token states. The primary pooled CL distribution is in Figure~\ref{fig:e-distributions}, and the primary pooled paired views are in Figures~\ref{fig:a-identities} and~\ref{fig:e-identities}. All feature-level plots supply the true target length (oracle length). Pooled means one vector obtained by averaging the valid positions of a text; token means one sampled valid position per text. Recoverable denotes reversible inserted markers; heldout denotes the alternative marker pattern defined in Section~\ref{sec:results}. The full-cohort tables and curves separately report PL predicted-length and CL input-length performance.

\begin{figure}[!htbp]
\centering
\includegraphics[width=\linewidth]{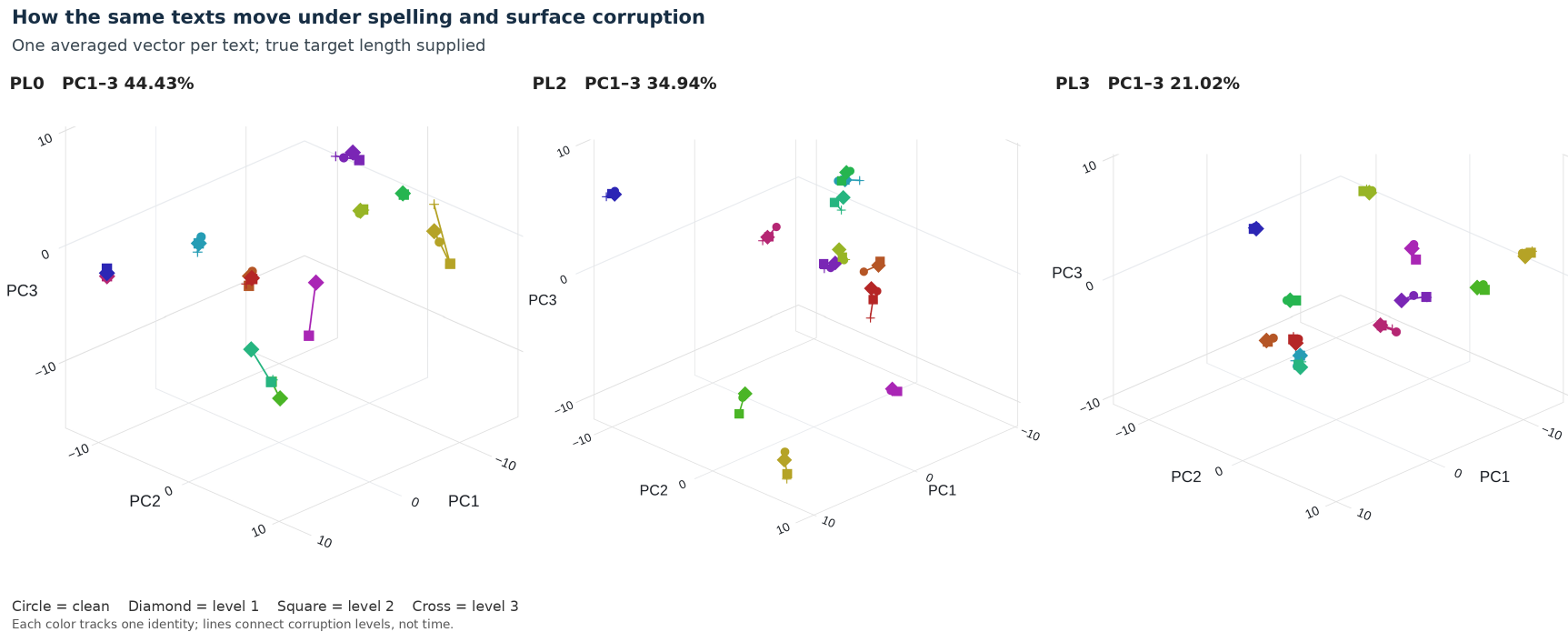}
\caption{Following the same identities through natural corruption. Each panel shows 12 texts, one color per text. Follow a line from the clean circle through the diamond, square, and cross at corruption levels 1, 2, and 3. The displayed states are averaged text vectors with true target lengths supplied. Axes are fitted on clean states separately for each model. Some PL0 views stay close even though full-gallery retrieval is poor in Figure~\ref{fig:pl-evidence}; visual proximity must therefore be read together with identity access. Percentages state how much clean variance the three displayed axes retain.}
\label{fig:a-identities}
\end{figure}

\begin{figure}[t]
\centering
\includegraphics[width=.90\linewidth]{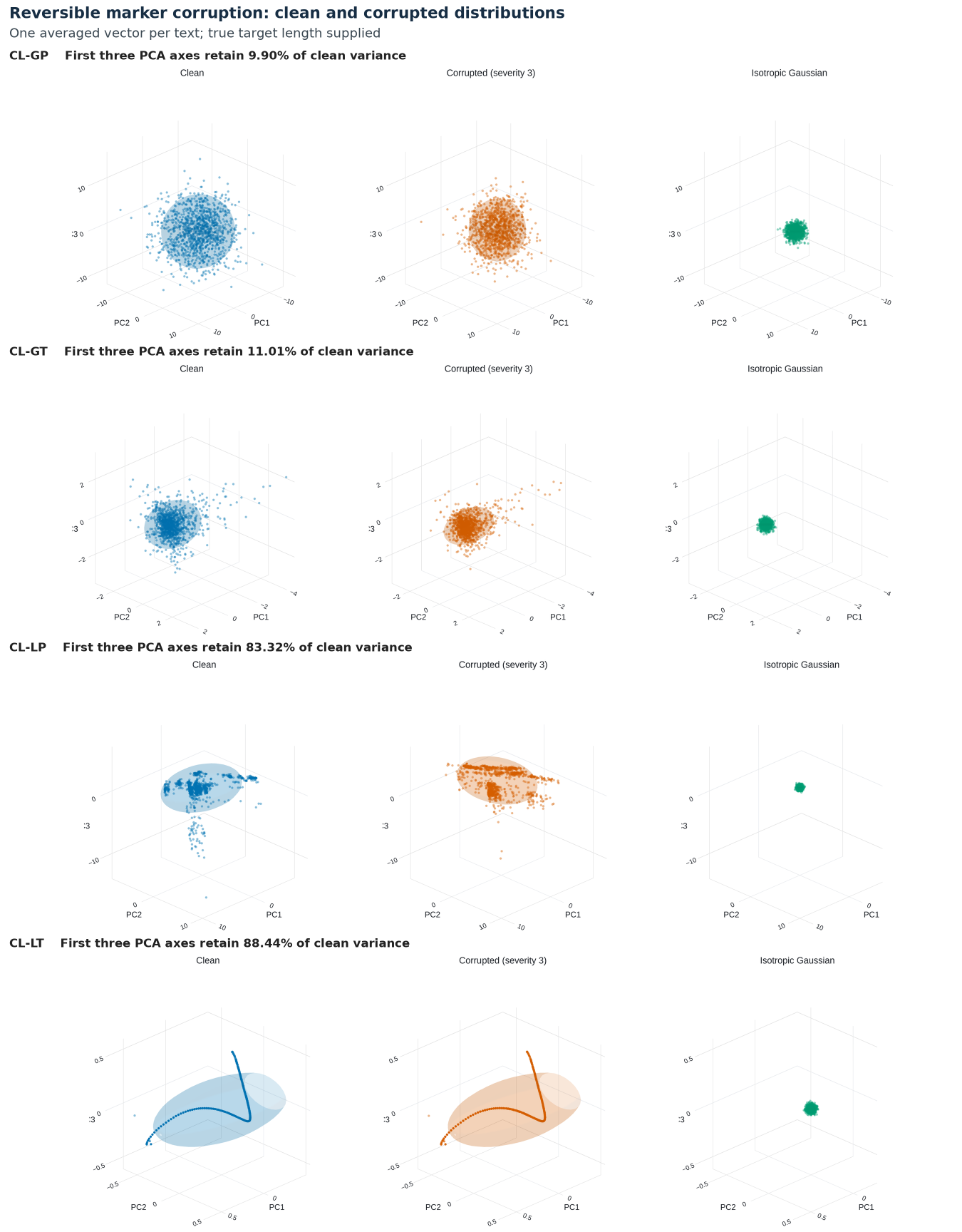}
\caption{What the geometry diagnostic sees in the four CL configurations. Rows identify the models; columns show clean text, reversible corruption at level 3, and an isotropic Gaussian reference. Each point is one averaged text vector; true target lengths are supplied. The same clean PCA axes and scale apply across a row. CL-GP/GT spread most variance beyond the three displayed axes, whereas CL-LP/LT concentrate it in this subspace. Compare these clouds with retrieval in Figure~\ref{fig:cl-length}, rather than reading a compact cloud as evidence of content preservation. Axes are fitted to 4,096 clean examples; each cloud displays 1,200 points. Surfaces span two standard deviations, not confidence regions.}
\label{fig:e-distributions}
\end{figure}
\clearpage
\paragraph{How to read the grids.}
A distribution row compares clean states, severity-3 states, and the supplied standard isotropic Gaussian reference. The clean PCA basis, coordinate units, and limits are fixed within the row. Each paired panel instead follows 12 individual canonical identities through clean and severity levels 1--3. A color identifies a target; a circle marks clean, a diamond level 1, a square level 2, and a cross level 3. The connecting lines show a change across corruption levels, not temporal dynamics. Inspect whether identities remain separated as well as how far their views move. Full-gallery retrieval measures the corresponding identity access beyond these selected examples and three coordinates.

The first three principal components retain a model-dependent fraction of variance. For example, the clean pooled CL-LP and CL-LT projections retain 83.32\% and 88.44\%, while the corresponding token projections retain 59.56\% and 37.50\%. The grids therefore expose different subspaces and scales, not a shared coordinate system for measuring cross-model distances. Table~\ref{tab:visual-subset} gives the exact subset metrics needed to interpret them.
\begin{table}[!htbp]
\centering\small\setlength{\tabcolsep}{3pt}\renewcommand{\arraystretch}{1.12}
\caption{Clean 4,096-item visualization subset, oracle length. The dimensions $d_{90}$ and $d_{95}$ are the smallest numbers of principal components explaining 90\% and 95\% of variance. These subset ranks need not equal the full-cohort ranks.}
\label{tab:visual-subset}
\begin{tabular*}{\linewidth}{@{\extracolsep{\fill}}llrrrr@{}}
\toprule
\thead{Model} & \thead{Representation} & \thead{Effective rank} & \thead{$d_{90}$} & \thead{$d_{95}$} & \thead{PC1--3 (\%)} \\
\midrule
PL0 & pooled & 15.62 & 15 & 18 & 44.43 \\
PL0 & token & 15.63 & 15 & 18 & 44.43 \\
PL2 & pooled & 15.54 & 12 & 14 & 34.94 \\
PL2 & token & 34.22 & 31 & 48 & 25.55 \\
PL3 & pooled & 35.49 & 29 & 37 & 21.02 \\
PL3 & token & 65.41 & 58 & 85 & 15.44 \\
CL-GP & pooled & 54.09 & 41 & 52 & 9.90 \\
CL-GP & token & 42.53 & 34 & 46 & 17.53 \\
CL-GT & pooled & 108.23 & 97 & 131 & 11.01 \\
CL-GT & token & 164.74 & 135 & 177 & 5.46 \\
CL-LP & pooled & 4.65 & 5 & 6 & 83.32 \\
CL-LP & token & 14.25 & 20 & 35 & 59.56 \\
CL-LT & pooled & 3.59 & 4 & 6 & 88.44 \\
CL-LT & token & 28.28 & 32 & 44 & 37.50 \\
\bottomrule
\end{tabular*}
\end{table}

\begin{table}[!htbp]
\centering\small\setlength{\tabcolsep}{3pt}\renewcommand{\arraystretch}{1.12}
\caption{Coverage of the complete accompanying visual atlas. Each registered source PNG appears on its own landscape page. The panel index supplies page numbers, original paths, notes, and SHA-256 hashes.}
\label{tab:atlas-coverage}
\begin{tabular*}{\linewidth}{@{\extracolsep{\fill}}lrrr@{}}
\toprule
\thead{Diagnostic} & \thead{PL} & \thead{CL} & \thead{Total} \\
\midrule
Full-cohort summary charts & 5 & 10 & 15 \\
Full-cohort covariance spectra & 2 & 3 & 5 \\
Clean-subset spectra & 6 & 8 & 14 \\
3D distributions & 6 & 16 & 22 \\
Paired 3D trajectories & 6 & 16 & 22 \\
Projected Gaussian Q--Q plots & 6 & 16 & 22 \\
Positive/random-foil cosines & 3 & 8 & 11 \\
Random-foil margin ECDFs & 3 & 8 & 11 \\
\midrule
Total & 37 & 85 & 122 \\
\bottomrule
\end{tabular*}
\end{table}

\paragraph{Data and reproducibility.}
The source package includes all 122 registered PNGs, their render and selection manifests, the clean-subset JSON metrics, the full-cohort CSV summaries, and the paired uncertainty summaries. The CSV hashes match the source identities in both supplied geometry manifests. Display names use PL/CL; original checkpoint keys remain in machine-readable files for traceability. The archives supply rendered feature plots and feature hashes, but no latent arrays or per-point PCA coordinates. The included scripts regenerate compositions and CSV-based figures; recomputing feature projections requires the original arrays. The atlas includes all supplied spectral, Gaussian-projection, pairwise-score, and centering panels listed in Table~\ref{tab:atlas-coverage}.
\clearpage
\begin{figure}[!ht]
\centering
\includegraphics[width=.94\linewidth,height=.82\textheight,keepaspectratio]{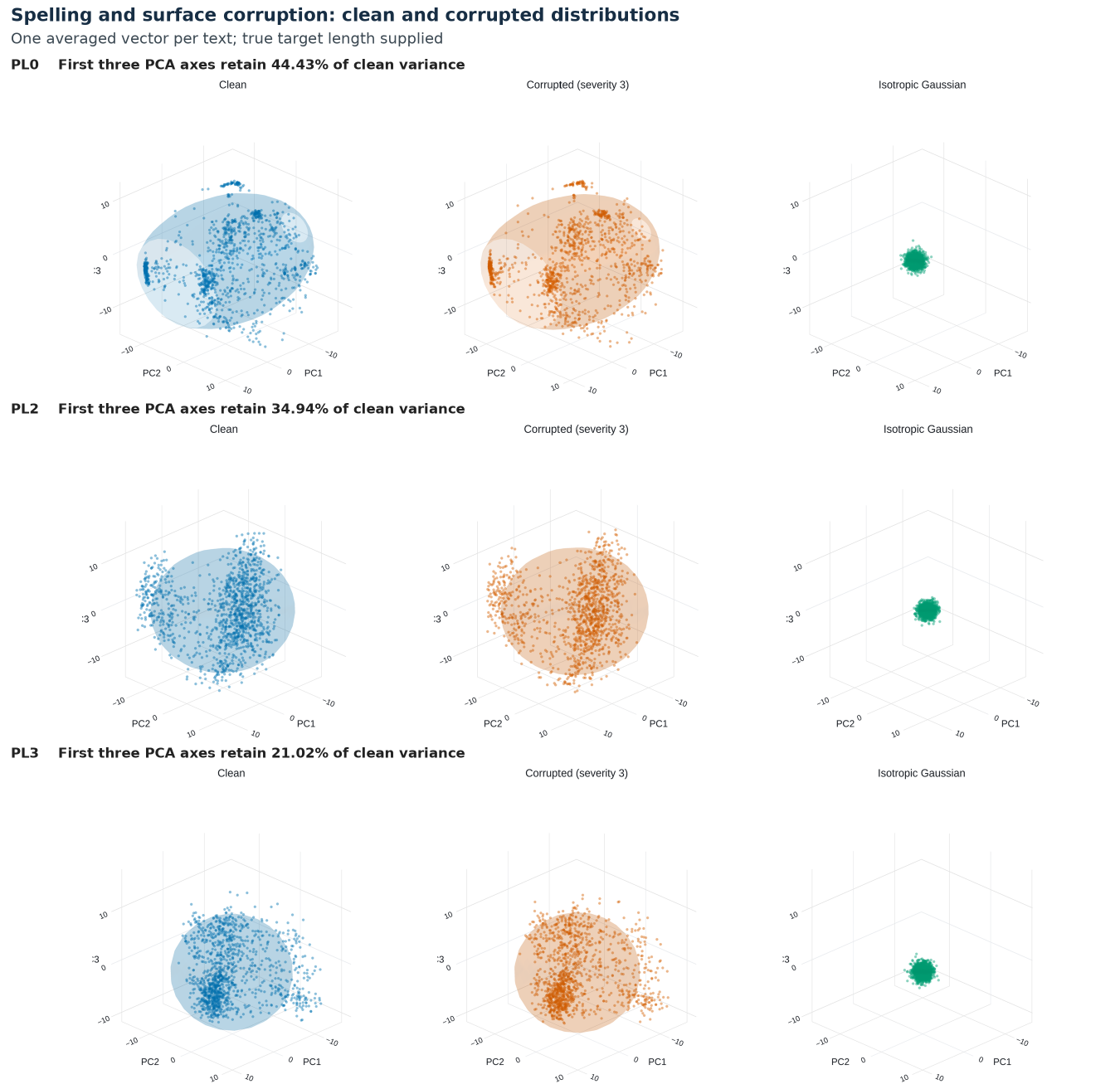}
\caption{PL natural distributions of pooled states with oracle length. Each row shows clean, severity 3, and the isotropic Gaussian reference in the same clean PCA coordinates. All supplied models are included. Clouds contain 1,200 points; surfaces span two standard deviations and are not confidence regions.}
\label{fig:PL-natural-pooled-distribution}
\end{figure}
\clearpage
\begin{figure}[!ht]
\centering
\includegraphics[width=.94\linewidth,height=.82\textheight,keepaspectratio]{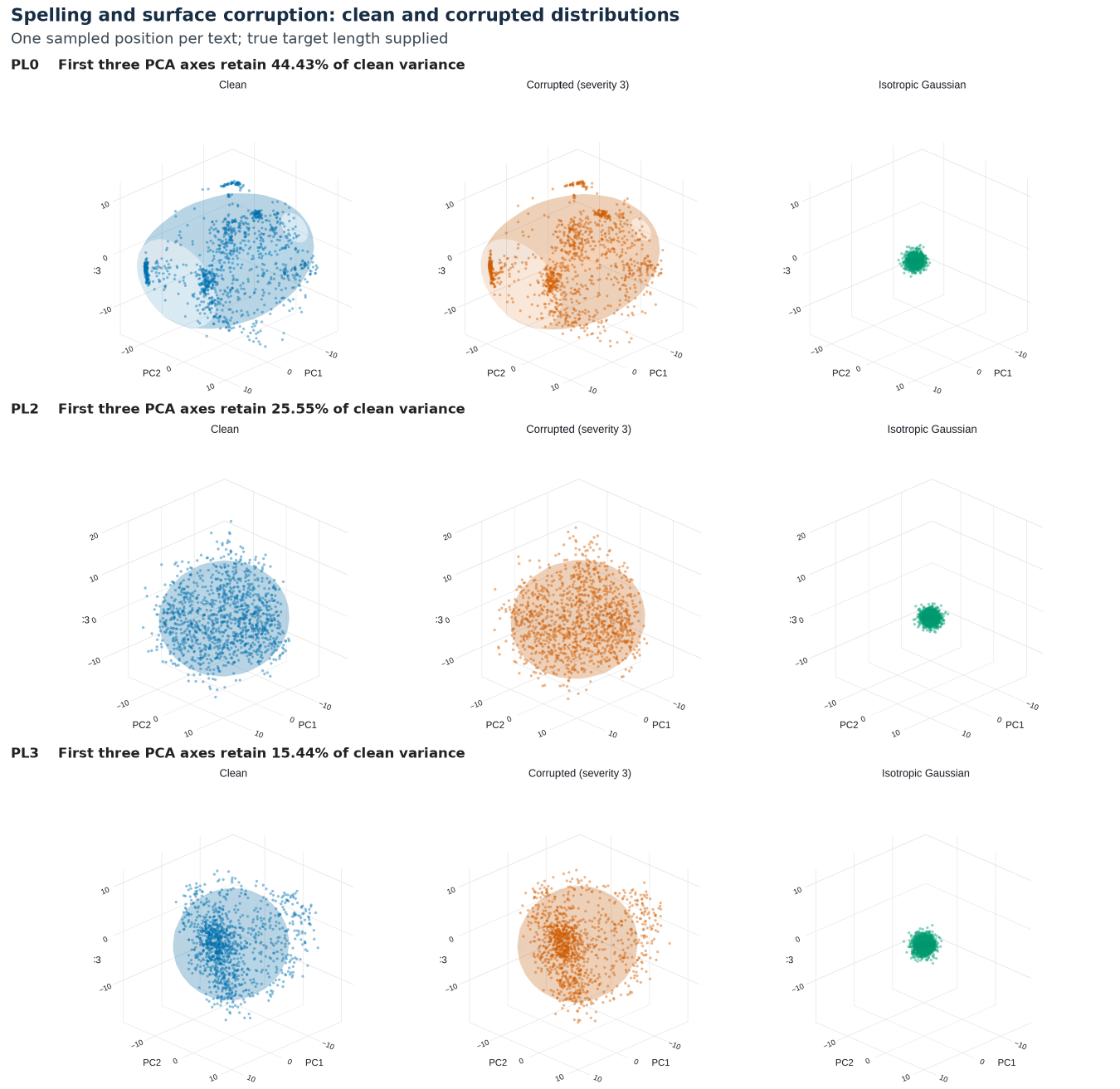}
\caption{PL natural distributions of token states with oracle length. Each row shows clean, severity 3, and the isotropic Gaussian reference in the same clean PCA coordinates. All supplied models are included. Clouds contain 1,200 points; surfaces span two standard deviations and are not confidence regions.}
\label{fig:PL-natural-token-distribution}
\end{figure}
\clearpage
\begin{figure}[!ht]
\centering
\includegraphics[width=\linewidth,height=.80\textheight,keepaspectratio]{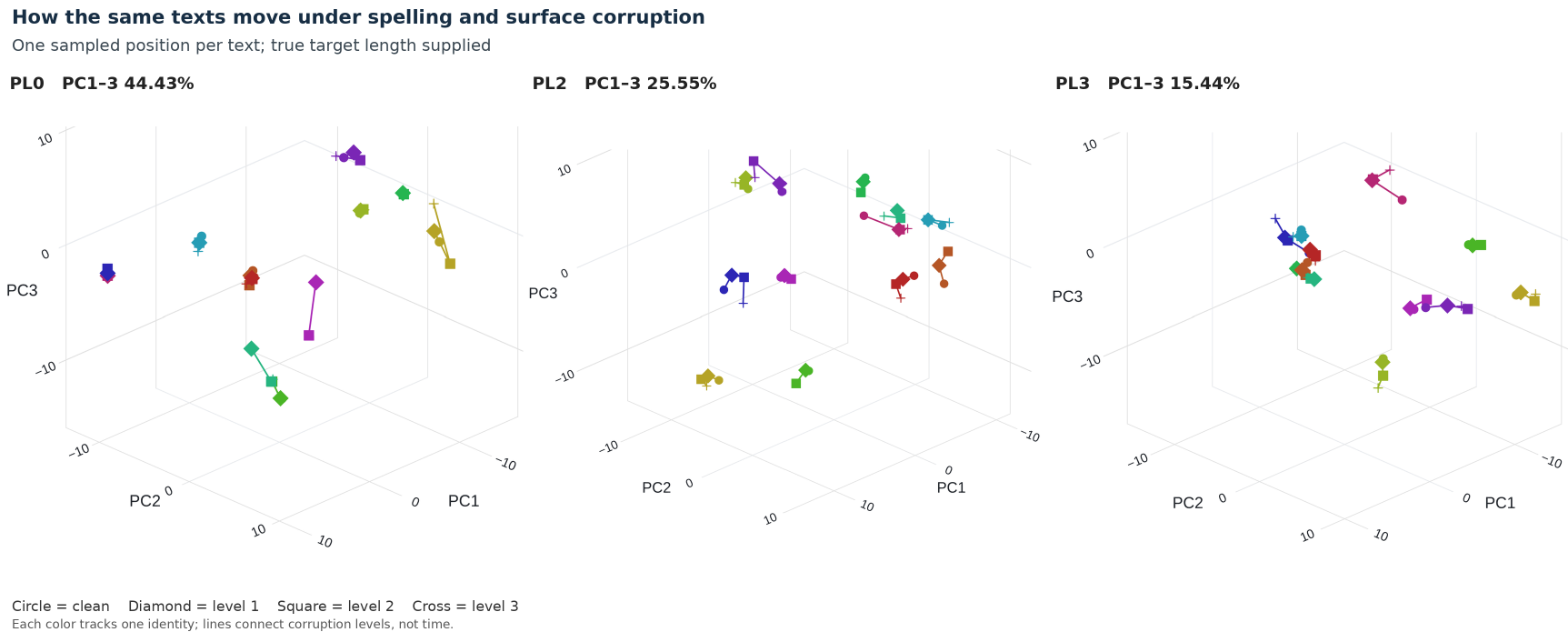}
\caption{Paired PL natural trajectories for token states with oracle length. Each panel follows 12 canonical identities through clean and severity levels 1--3 using one clean PCA basis per model. Colors identify targets and markers identify corruption levels. Differences in axes and captured variance require within-model interpretation.}
\label{fig:PL-natural-token-paired}
\end{figure}
\clearpage
\begin{figure}[!ht]
\centering
\includegraphics[width=\linewidth]{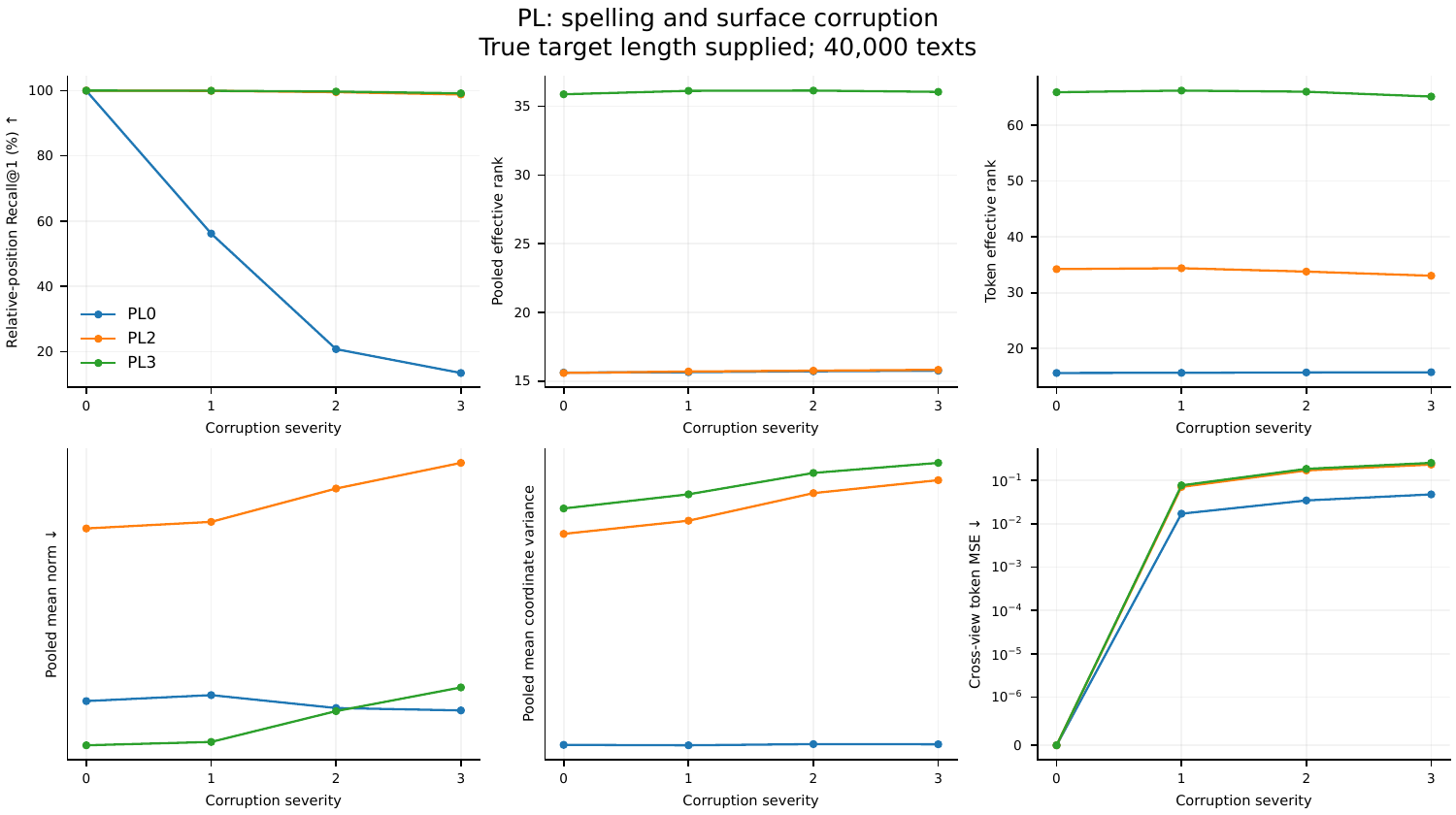}
\caption{Full-cohort PL natural diagnostics across severity, using oracle length. Retrieval, pooled and token effective rank, mean magnitude, variance scale, and cross-view error are shown separately. The bottom row uses a symmetric logarithmic scale with a linear region near zero. These curves use the original 40,000-item numeric summaries, independently of the smaller visualization subset.}
\label{fig:PL-natural-dashboard}
\end{figure}
\clearpage
\begin{figure}[!ht]
\centering
\includegraphics[width=.94\linewidth,height=.82\textheight,keepaspectratio]{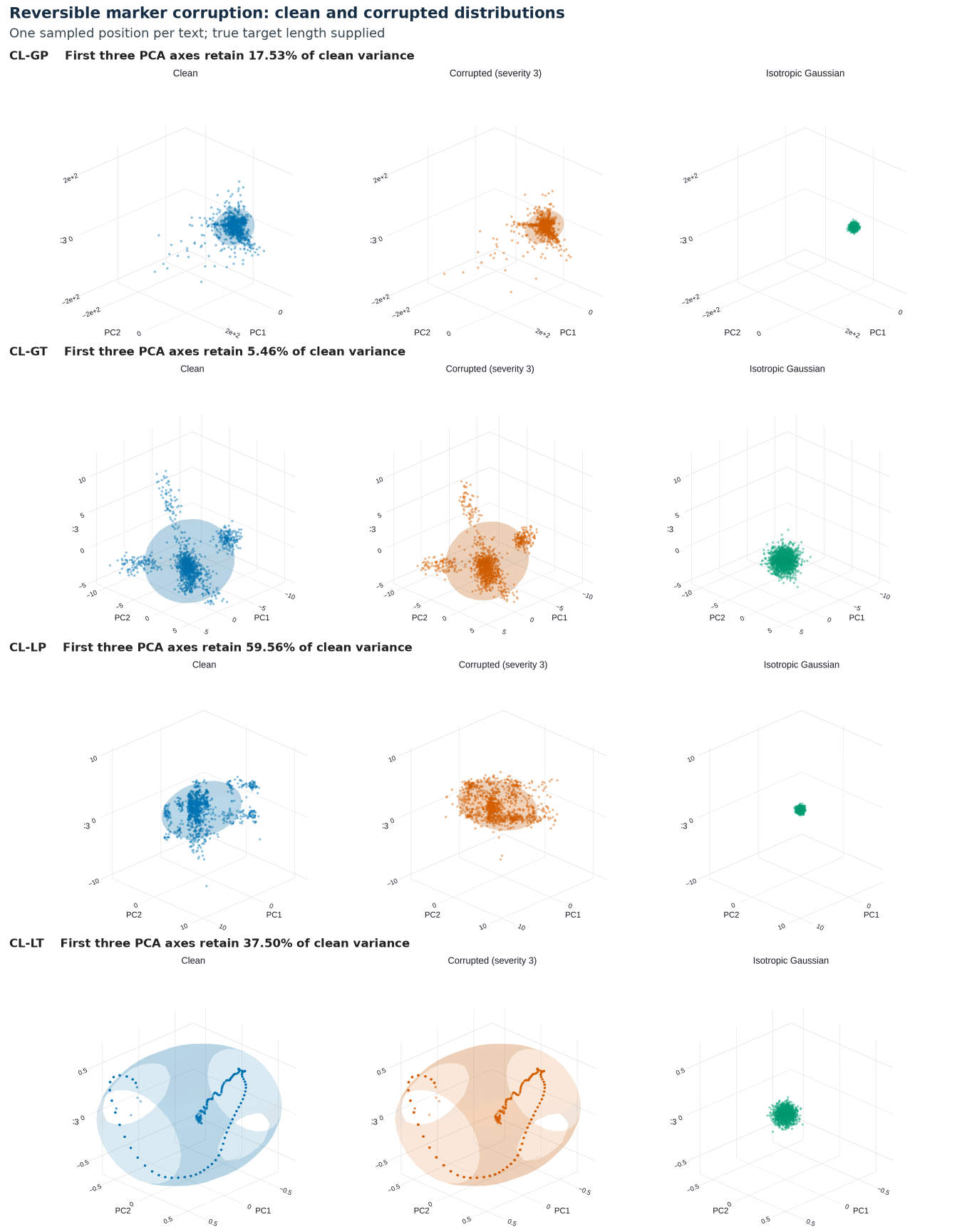}
\caption{CL recoverable distributions of token states with oracle length. Each row shows clean, severity 3, and the isotropic Gaussian reference in the same clean PCA coordinates. All supplied models are included. Clouds contain 1,200 points; surfaces span two standard deviations and are not confidence regions.}
\label{fig:CL-recoverable-token-distribution}
\end{figure}
\clearpage
\begin{figure}[!ht]
\centering
\includegraphics[width=\linewidth,height=.80\textheight,keepaspectratio]{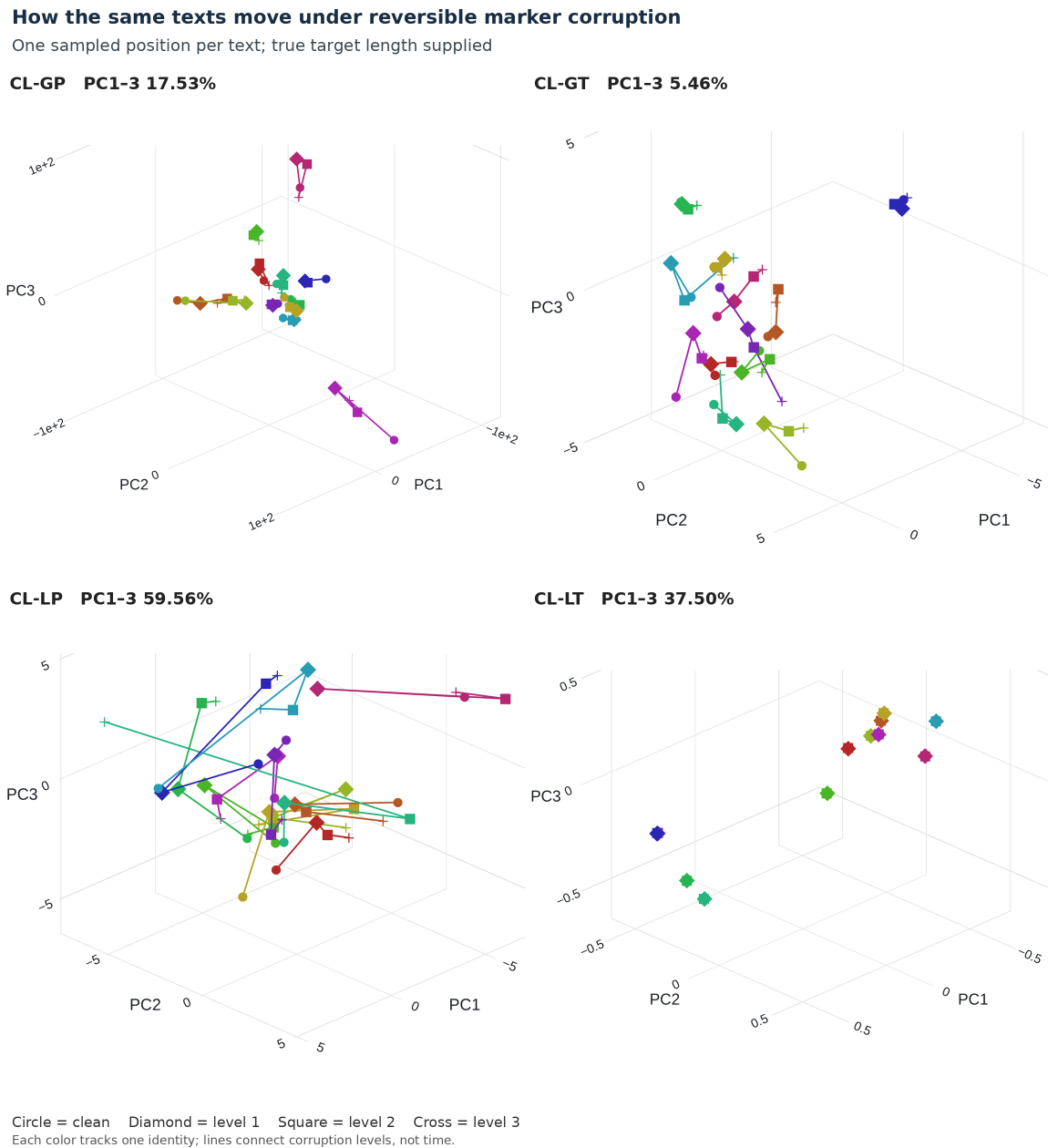}
\caption{Paired CL recoverable trajectories for token states with oracle length. Each panel follows 12 canonical identities through clean and severity levels 1--3 using one clean PCA basis per model. Colors identify targets and markers identify corruption levels. Differences in axes and captured variance require within-model interpretation.}
\label{fig:CL-recoverable-token-paired}
\end{figure}
\clearpage
\begin{figure}[!ht]
\centering
\includegraphics[width=\linewidth]{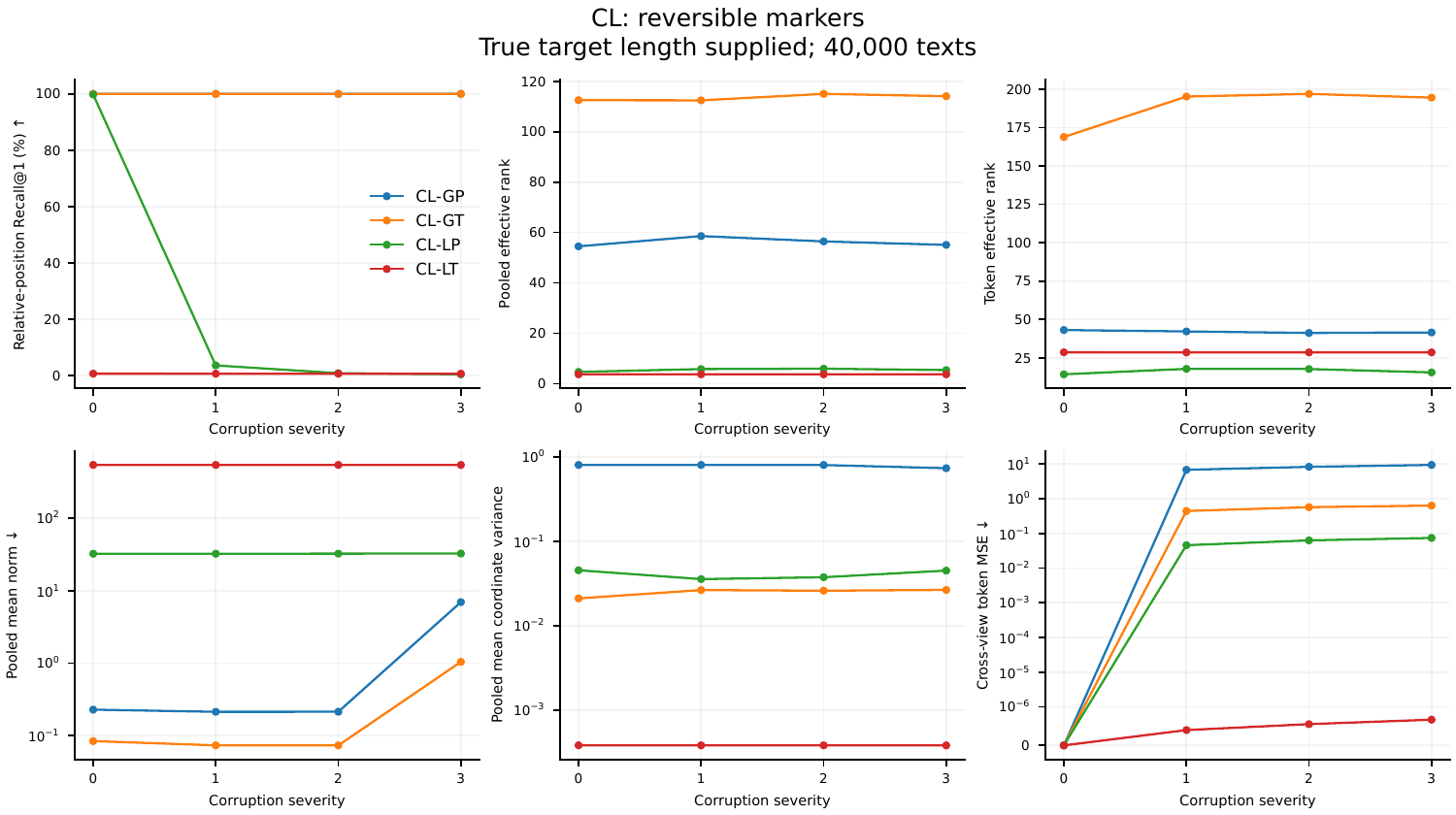}
\caption{Full-cohort CL recoverable diagnostics across severity, using oracle length. Retrieval, pooled and token effective rank, mean magnitude, variance scale, and cross-view error are shown separately. The bottom row uses a symmetric logarithmic scale with a linear region near zero. These curves use the original 40,000-item numeric summaries, independently of the smaller visualization subset.}
\label{fig:CL-recoverable-dashboard}
\end{figure}
\clearpage
\begin{figure}[!ht]
\centering
\includegraphics[width=.94\linewidth,height=.82\textheight,keepaspectratio]{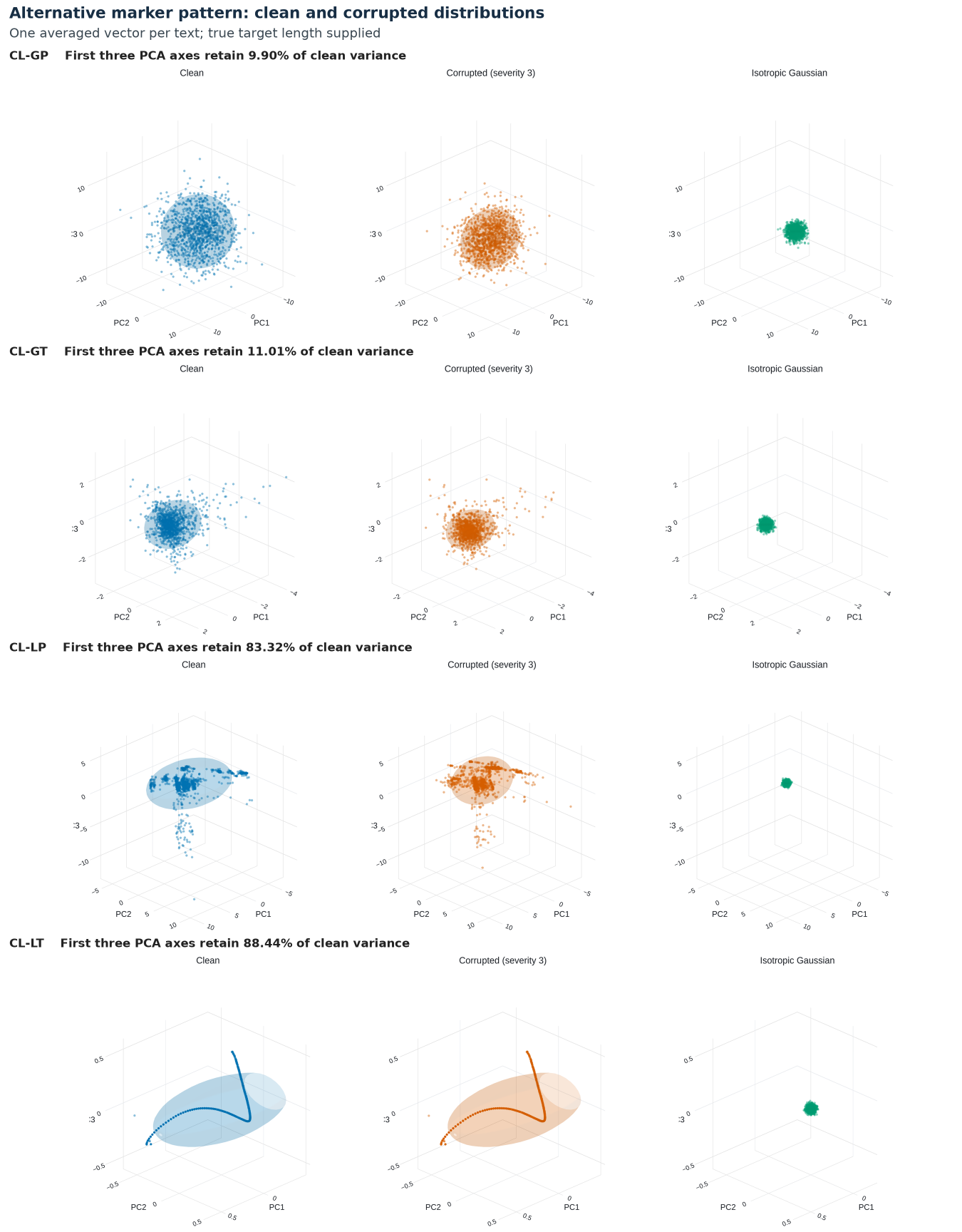}
\caption{CL heldout distributions of pooled states with oracle length. Each row shows clean, severity 3, and the isotropic Gaussian reference in the same clean PCA coordinates. All supplied models are included. Clouds contain 1,200 points; surfaces span two standard deviations and are not confidence regions.}
\label{fig:CL-heldout-pooled-distribution}
\end{figure}
\clearpage
\begin{figure}[!ht]
\centering
\includegraphics[width=\linewidth,height=.80\textheight,keepaspectratio]{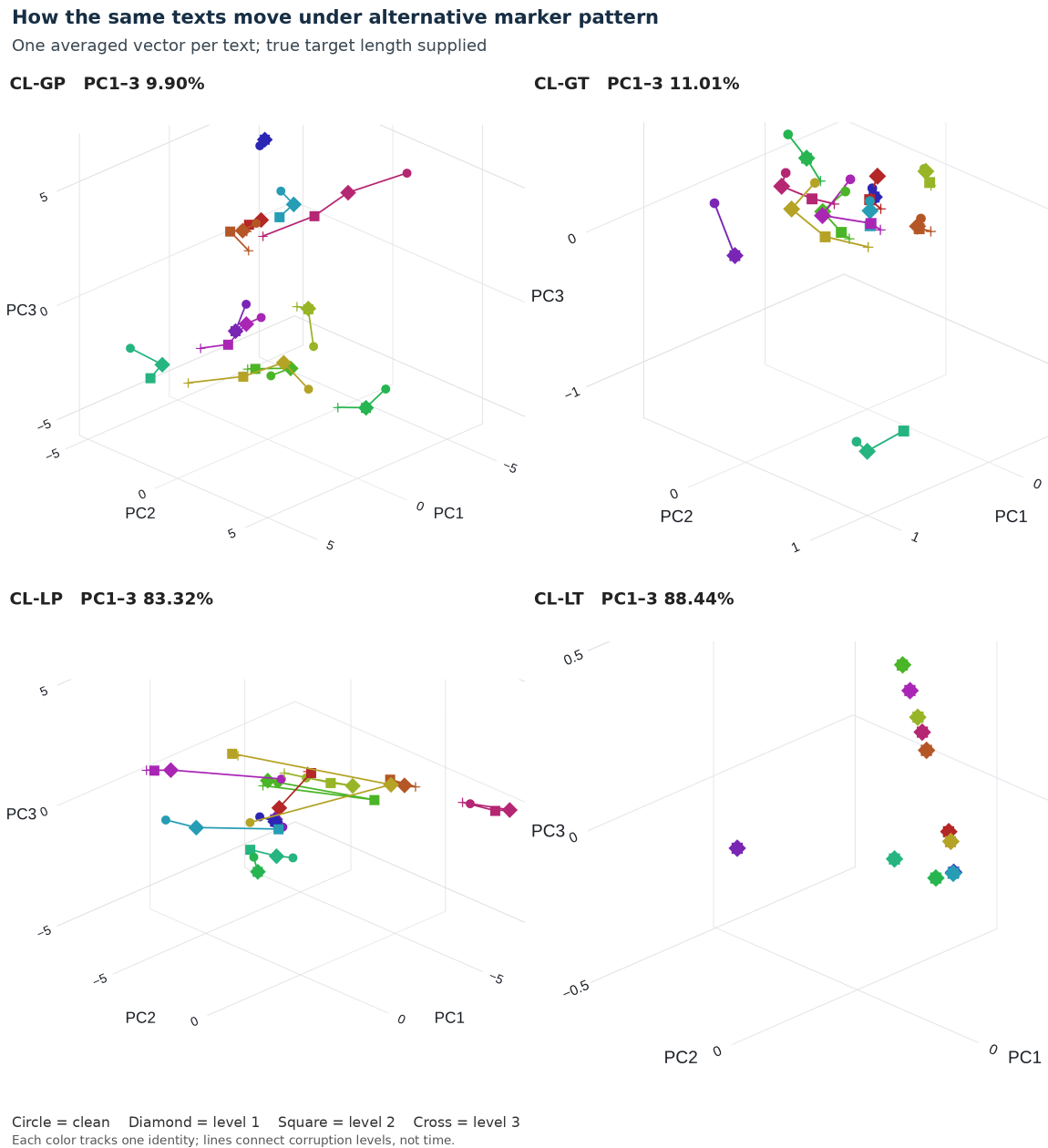}
\caption{Paired CL heldout trajectories for pooled states with oracle length. Each panel follows 12 canonical identities through clean and severity levels 1--3 using one clean PCA basis per model. Colors identify targets and markers identify corruption levels. Differences in axes and captured variance require within-model interpretation.}
\label{fig:CL-heldout-pooled-paired}
\end{figure}
\clearpage
\begin{figure}[!ht]
\centering
\includegraphics[width=.94\linewidth,height=.82\textheight,keepaspectratio]{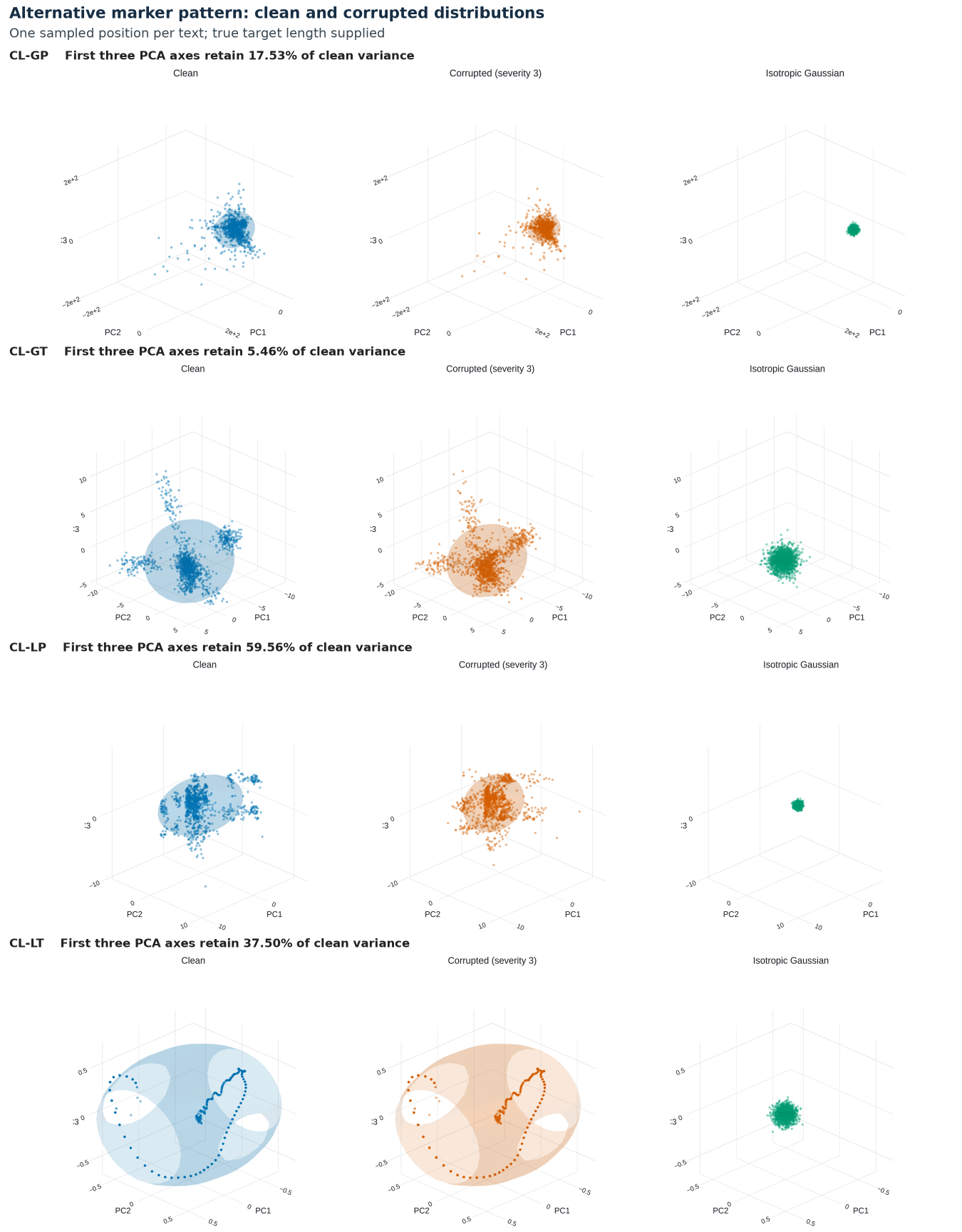}
\caption{CL heldout distributions of token states with oracle length. Each row shows clean, severity 3, and the isotropic Gaussian reference in the same clean PCA coordinates. All supplied models are included. Clouds contain 1,200 points; surfaces span two standard deviations and are not confidence regions.}
\label{fig:CL-heldout-token-distribution}
\end{figure}
\clearpage
\begin{figure}[!ht]
\centering
\includegraphics[width=\linewidth,height=.80\textheight,keepaspectratio]{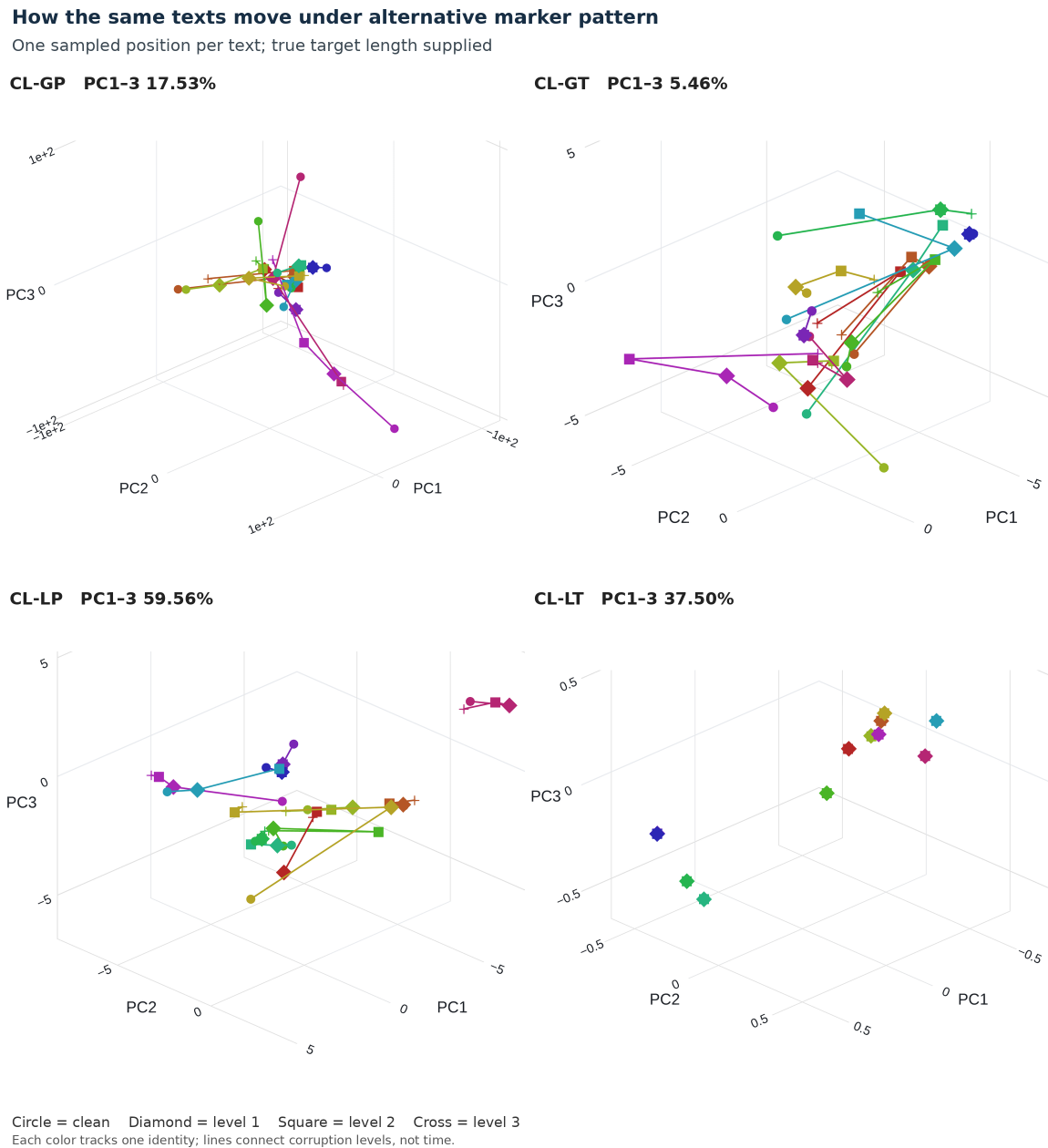}
\caption{Paired CL heldout trajectories for token states with oracle length. Each panel follows 12 canonical identities through clean and severity levels 1--3 using one clean PCA basis per model. Colors identify targets and markers identify corruption levels. Differences in axes and captured variance require within-model interpretation.}
\label{fig:CL-heldout-token-paired}
\end{figure}
\clearpage
\begin{figure}[!ht]
\centering
\includegraphics[width=\linewidth]{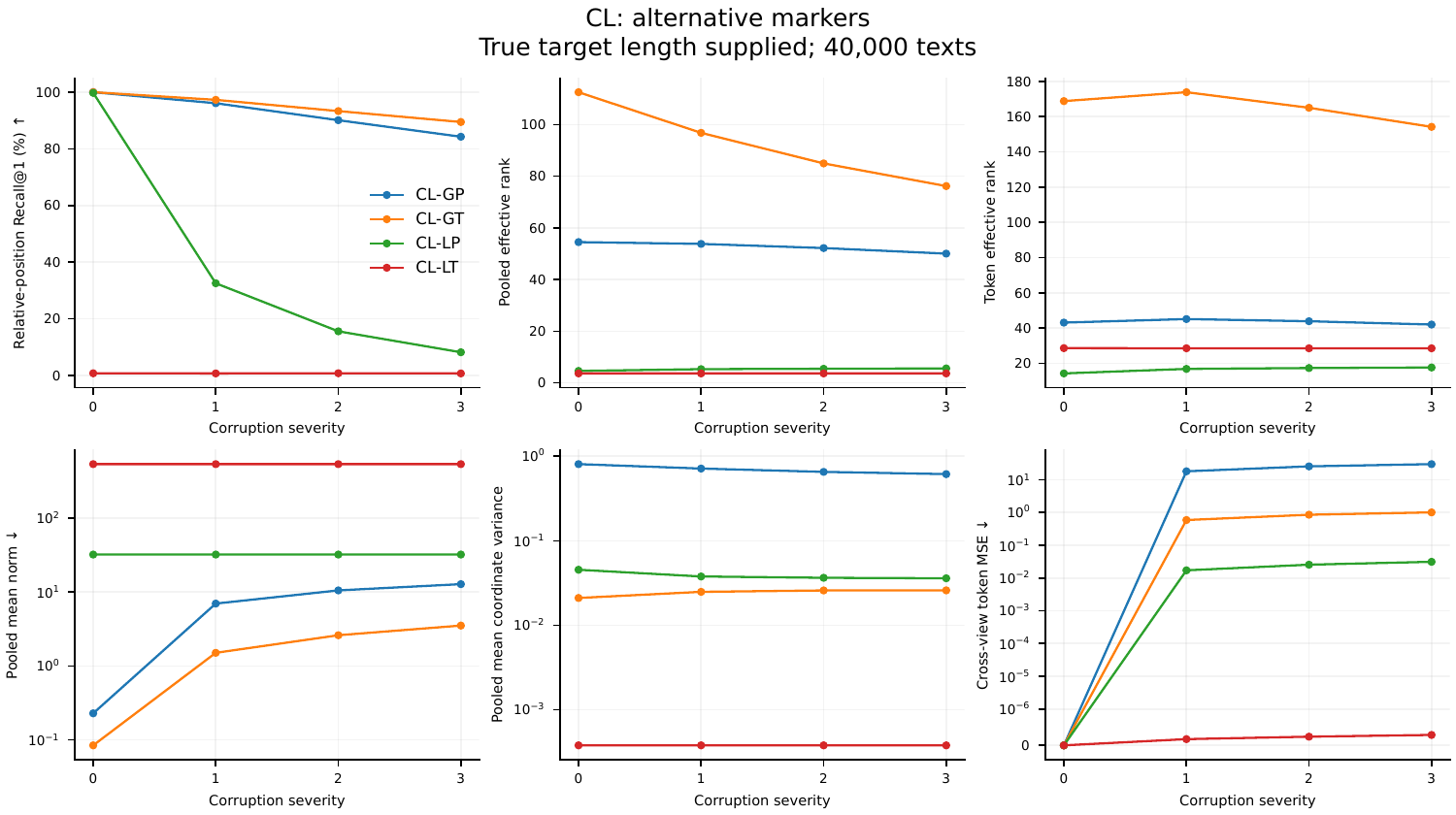}
\caption{Full-cohort CL heldout diagnostics across severity, using oracle length. Retrieval, pooled and token effective rank, mean magnitude, variance scale, and cross-view error are shown separately. The bottom row uses a symmetric logarithmic scale with a linear region near zero. These curves use the original 40,000-item numeric summaries, independently of the smaller visualization subset.}
\label{fig:CL-heldout-dashboard}
\end{figure}

\clearpage
\section{Downstream Speech Samples}
\label{app:tts-samples}
\subsection*{Sample 1: LJ033-0072}
\noindent\textbf{Reference/clean input:} I then stepped off of it and the officer picked it up in the middle and it bent so.\par
\noindent\textbf{Augmented input:} I theN stepped off of it and the officer picked IT up in the middle and it bent so.\par
\begin{figure}[!htbp]
\centering
\includegraphics[width=\linewidth]{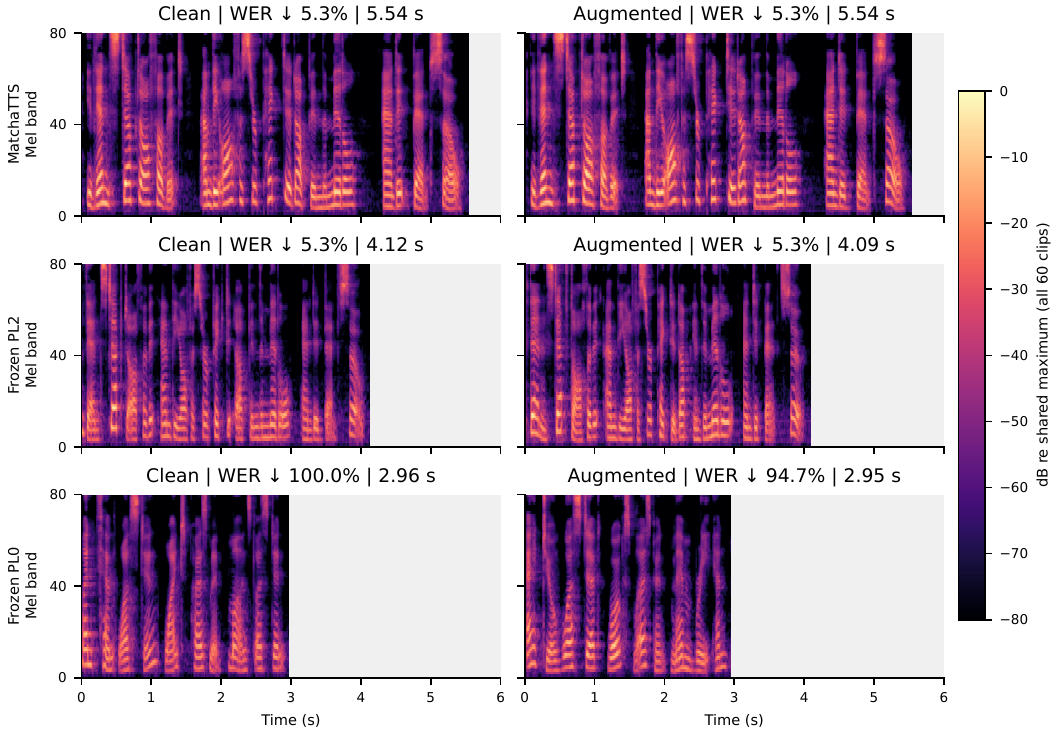}
\caption{LJ033-0072: all three systems under both input views. WER uses the clean reference; gray after a panel ends indicates no generated audio. Axes are shared within this example; the color reference is shared across all ten examples.}
\label{fig:tts-sample-01}
\end{figure}
\noindent\textbf{Whisper hypotheses (normalized, as scored).}
\par\smallskip
\begingroup\small
\noindent\begin{tabular}{@{}p{0.13\linewidth}p{0.405\linewidth}p{0.405\linewidth}@{}}
\toprule
\thead{System} & \thead{Clean audio} & \thead{Augmented audio} \\
\midrule
MatchaTTS & he then stepped off of it and the officer picked it up in the middle and it bent so & they then stepped off of it and the officer picked it up in the middle and it bent so \\
Frozen PL2 & then stepped off of it and the officer picked it up in the middle and it bent so & then stepped off of it and the officer picked it up in the middle and it bent so \\
Frozen PL0 & when kent was still white as mid months was still & at your worst at works placement memory and \\
\bottomrule
\end{tabular}\par\endgroup
\clearpage
\subsection*{Sample 2: LJ006-0145}
\noindent\textbf{Reference/clean input:} He could indulge in snuff if a snuff-taker,\par
\noindent\textbf{Augmented input:} He cou,d indulge ib snuff if a snuff-taker,\par
\begin{figure}[!htbp]
\centering
\includegraphics[width=\linewidth]{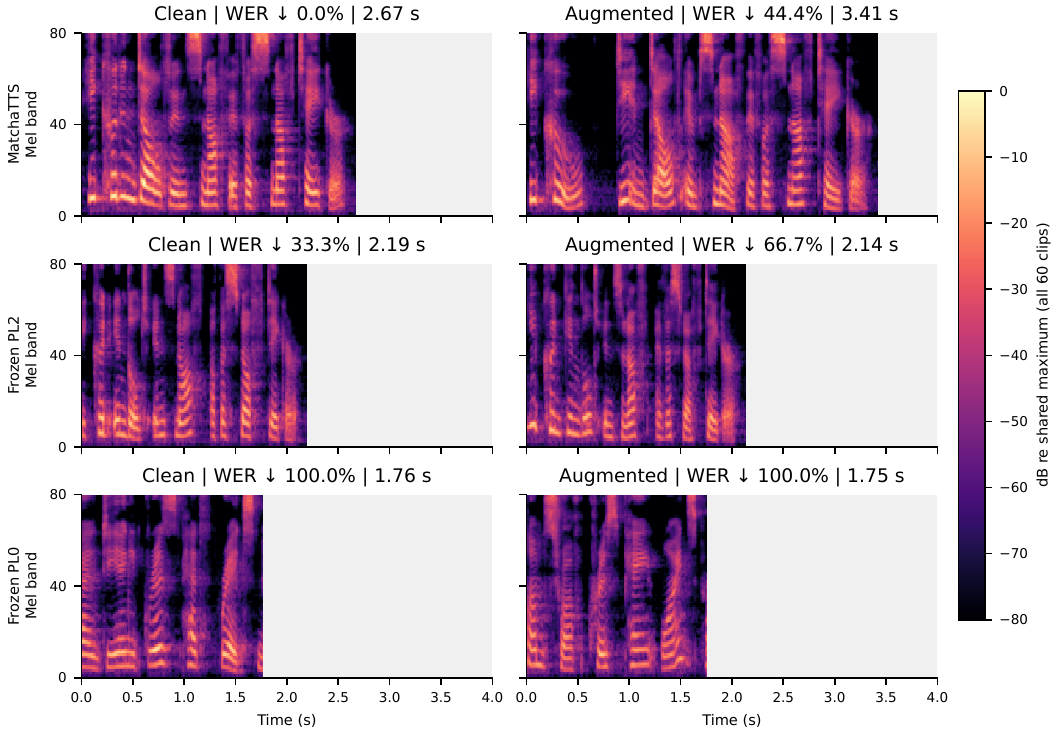}
\caption{LJ006-0145: all three systems under both input views. WER uses the clean reference; gray after a panel ends indicates no generated audio. Axes are shared within this example; the color reference is shared across all ten examples.}
\label{fig:tts-sample-02}
\end{figure}
\noindent\textbf{Whisper hypotheses (normalized, as scored).}
\par\smallskip
\begingroup\small
\noindent\begin{tabular}{@{}p{0.13\linewidth}p{0.405\linewidth}p{0.405\linewidth}@{}}
\toprule
\thead{System} & \thead{Clean audio} & \thead{Augmented audio} \\
\midrule
MatchaTTS & he could indulge in snuff if a snuff taker & he could be an adult and snuff if a snuff taker \\
Frozen PL2 & it could indulge in snuff of a snuff digger & you can kindle it in snuff at the snuff taker \\
Frozen PL0 & \emph{Empty hypothesis} & plums truffle crisp paint wine spray \\
\bottomrule
\end{tabular}\par\endgroup
\clearpage
\subsection*{Sample 3: LJ002-0239}
\noindent\textbf{Reference/clean input:} also for prisoners committed by the Admiralty Court.\par
\noindent\textbf{Augmented input:} also for prisoners committed by the ADmIralTY Court.\par
\begin{figure}[!htbp]
\centering
\includegraphics[width=\linewidth]{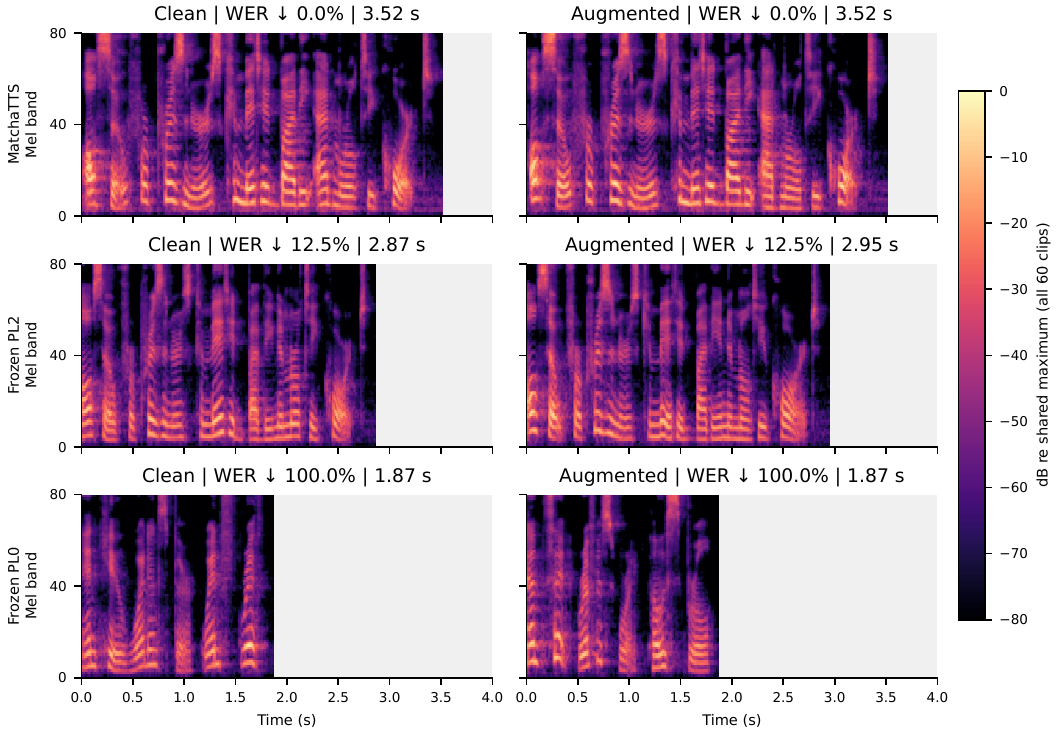}
\caption{LJ002-0239: all three systems under both input views. WER uses the clean reference; gray after a panel ends indicates no generated audio. Axes are shared within this example; the color reference is shared across all ten examples.}
\label{fig:tts-sample-03}
\end{figure}
\noindent\textbf{Whisper hypotheses (normalized, as scored).}
\par\smallskip
\begingroup\small
\noindent\begin{tabular}{@{}p{0.13\linewidth}p{0.405\linewidth}p{0.405\linewidth}@{}}
\toprule
\thead{System} & \thead{Clean audio} & \thead{Augmented audio} \\
\midrule
MatchaTTS & also for prisoners committed by the admiralty court & also for prisoners committed by the admiralty court \\
Frozen PL2 & also for prisoners committed by the mimralty court & also for prisoners committed by the indemnity court \\
Frozen PL0 & and he would not spare winfrey & that is it riffus roy post sprawl \\
\bottomrule
\end{tabular}\par\endgroup
\clearpage
\subsection*{Sample 4: LJ038-0296}
\noindent\textbf{Reference/clean input:} Marina Oswald also testified that her husband had used a bus to return home.\par
\noindent\textbf{Augmented input:} Marina Oswald also tstfd that her hsb had used a bus to return home.\par
\begin{figure}[!htbp]
\centering
\includegraphics[width=\linewidth]{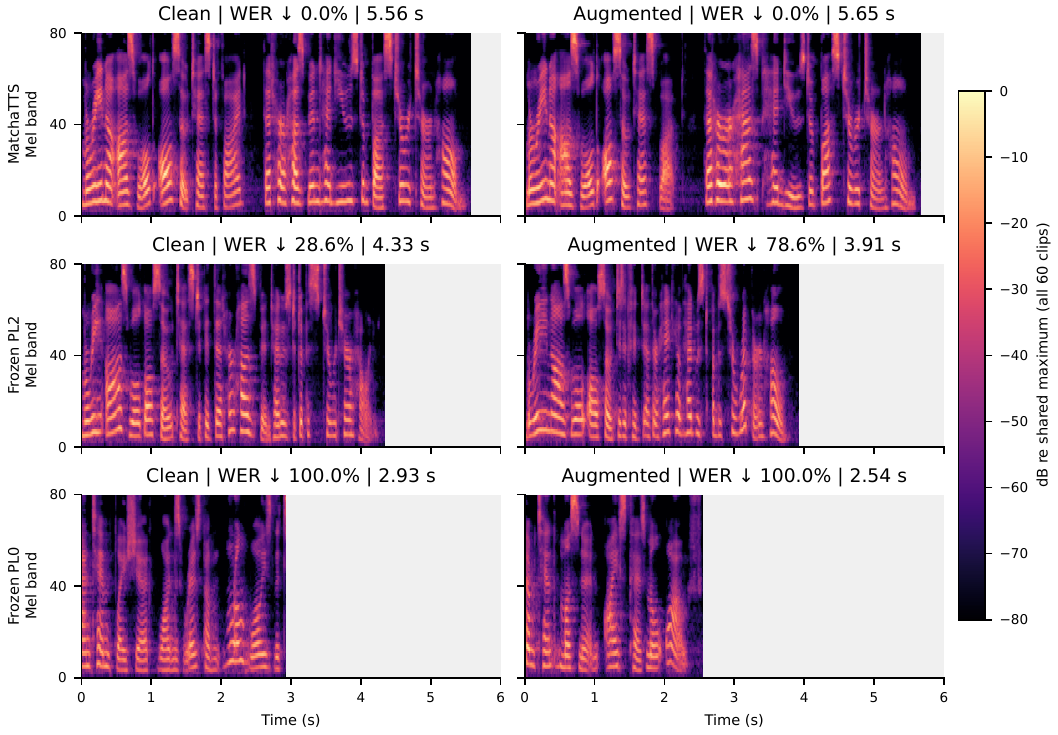}
\caption{LJ038-0296: all three systems under both input views. WER uses the clean reference; gray after a panel ends indicates no generated audio. Axes are shared within this example; the color reference is shared across all ten examples.}
\label{fig:tts-sample-04}
\end{figure}
\noindent\textbf{Whisper hypotheses (normalized, as scored).}
\par\smallskip
\begingroup\small
\noindent\begin{tabular}{@{}p{0.13\linewidth}p{0.405\linewidth}p{0.405\linewidth}@{}}
\toprule
\thead{System} & \thead{Clean audio} & \thead{Augmented audio} \\
\midrule
MatchaTTS & marina oswald also testified that her husband had used a bus to return home & marina oswald also testified that her husband had used a bus to return home \\
Frozen PL2 & marie oswald also testified that her husband had his deputist to return home & marine oswald also disappeared at the head of mister rippon is home \\
Frozen PL0 & and template shared one is rest of the world & the 2nd tent must know an ice cold mill of \\
\bottomrule
\end{tabular}\par\endgroup
\clearpage
\subsection*{Sample 5: LJ014-0142}
\noindent\textbf{Reference/clean input:} was strewn in front of the dock, and sprinkled it towards the bench with a contemptuous gesture.\par
\noindent\textbf{Augmented input:} was strewn in front of the dock, and sprinkled it tOWards the bench with a contemptuous gesture.\par
\begin{figure}[!htbp]
\centering
\includegraphics[width=\linewidth]{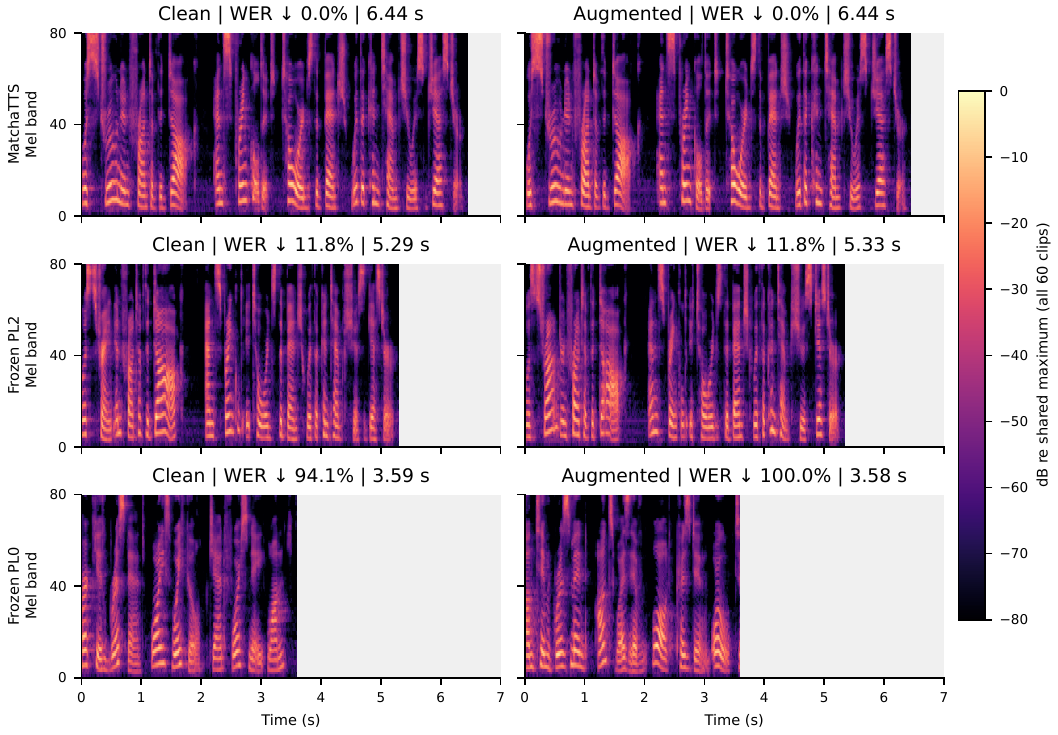}
\caption{LJ014-0142: all three systems under both input views. WER uses the clean reference; gray after a panel ends indicates no generated audio. Axes are shared within this example; the color reference is shared across all ten examples.}
\label{fig:tts-sample-05}
\end{figure}
\noindent\textbf{Whisper hypotheses (normalized, as scored).}
\par\smallskip
\begingroup\small
\noindent\begin{tabular}{@{}p{0.13\linewidth}p{0.405\linewidth}p{0.405\linewidth}@{}}
\toprule
\thead{System} & \thead{Clean audio} & \thead{Augmented audio} \\
\midrule
MatchaTTS & was strewn in front of the dock and sprinkled it towards the bench with a contemptuous gesture & was strewn in front of the dock and sprinkled it towards the bench with a contemptuous gesture \\
Frozen PL2 & was strained in front of the dock and sprinkled it towards the bench with a contemptuous gister & was surrounded in front of the dock and sprinkled it towards the bench with a contemptuous dister \\
Frozen PL0 & perkins brest and west waco jan 14 1981 & tim grismond antwesif walt christian orlton \\
\bottomrule
\end{tabular}\par\endgroup
\clearpage
\subsection*{Sample 6: LJ012-0270}
\noindent\textbf{Reference/clean input:} She had realized all her effects, and brought them with her furniture to Greenacre's lodgings. The two when married were to emigrate to Hudson's Bay.\par
\noindent\textbf{Augmented input:} She h ad realized all her effects, andbrought them with Herfurniture to Greenacres lodgings.The two when married were to emigrate to HudsonsBay.\par
\begin{figure}[!htbp]
\centering
\includegraphics[width=\linewidth]{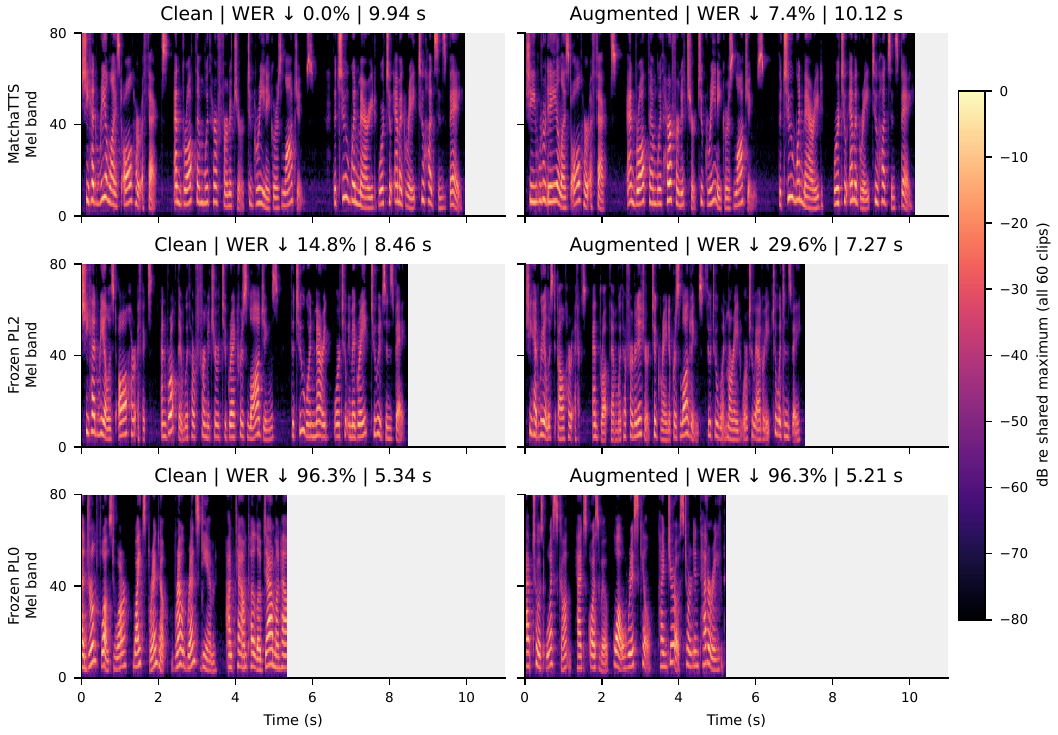}
\caption{LJ012-0270: all three systems under both input views. WER uses the clean reference; gray after a panel ends indicates no generated audio. Axes are shared within this example; the color reference is shared across all ten examples.}
\label{fig:tts-sample-06}
\end{figure}
\noindent\textbf{Whisper hypotheses (normalized, as scored).}
\par\smallskip
\begingroup\small
\noindent\begin{tabular}{@{}p{0.13\linewidth}p{0.405\linewidth}p{0.405\linewidth}@{}}
\toprule
\thead{System} & \thead{Clean audio} & \thead{Augmented audio} \\
\midrule
MatchaTTS & she had realized all her effects and brought them with her furniture to greenacre is lodgings the 2 when married were to emigrate to hudson is bay & she radialized all her effects and brought them with her furniture to greenacre is lodgings the 2 when married were to emigrate to hudson is bay \\
Frozen PL2 & she had realized all her ethics and brought them with her furniture to grecker is lodgings the 2 when married or to immigrate to hudson is bay & she had realized all her effects and brought them with her furniture to gray nipper is lodgings through to when moraine wore to agarty to witton is bay \\
Frozen PL0 & drum flows to our white sprint of brown rents dm on town pillow down on how & apto is wescon heads oswald while rhys kill hanjuro is turn in petare \\
\bottomrule
\end{tabular}\par\endgroup
\clearpage
\subsection*{Sample 7: LJ011-0213}
\noindent\textbf{Reference/clean input:} married the said Ellen Turner.\par
\noindent\textbf{Augmented input:} married the said eLlen Turner.\par
\begin{figure}[!htbp]
\centering
\includegraphics[width=\linewidth]{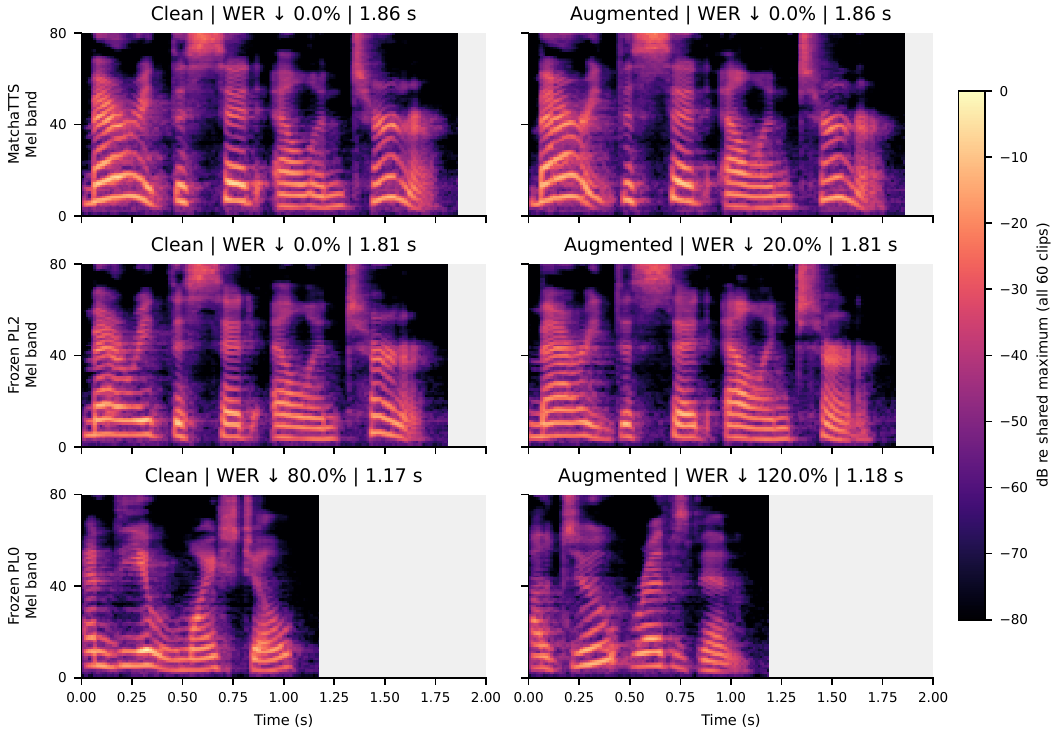}
\caption{LJ011-0213: all three systems under both input views. WER uses the clean reference; gray after a panel ends indicates no generated audio. Axes are shared within this example; the color reference is shared across all ten examples.}
\label{fig:tts-sample-07}
\end{figure}
\noindent\textbf{Whisper hypotheses (normalized, as scored).}
\par\smallskip
\begingroup\small
\noindent\begin{tabular}{@{}p{0.13\linewidth}p{0.405\linewidth}p{0.405\linewidth}@{}}
\toprule
\thead{System} & \thead{Clean audio} & \thead{Augmented audio} \\
\midrule
MatchaTTS & married the said ellen turner & married the said ellen turner \\
Frozen PL2 & married the said ellen turner & married the said alan turner \\
Frozen PL0 & and the moisture & i will do reverend is bim \\
\bottomrule
\end{tabular}\par\endgroup
\clearpage
\subsection*{Sample 8: LJ007-0211}
\noindent\textbf{Reference/clean input:} In our late visits, they say, "we have seen manifest indications of a retrograde movement in this respect,\par
\noindent\textbf{Augmented input:} In ourlate visits, they say, "we haveSEEN manifest indications of a retrograde movementin this rspct,\par
\begin{figure}[!htbp]
\centering
\includegraphics[width=\linewidth]{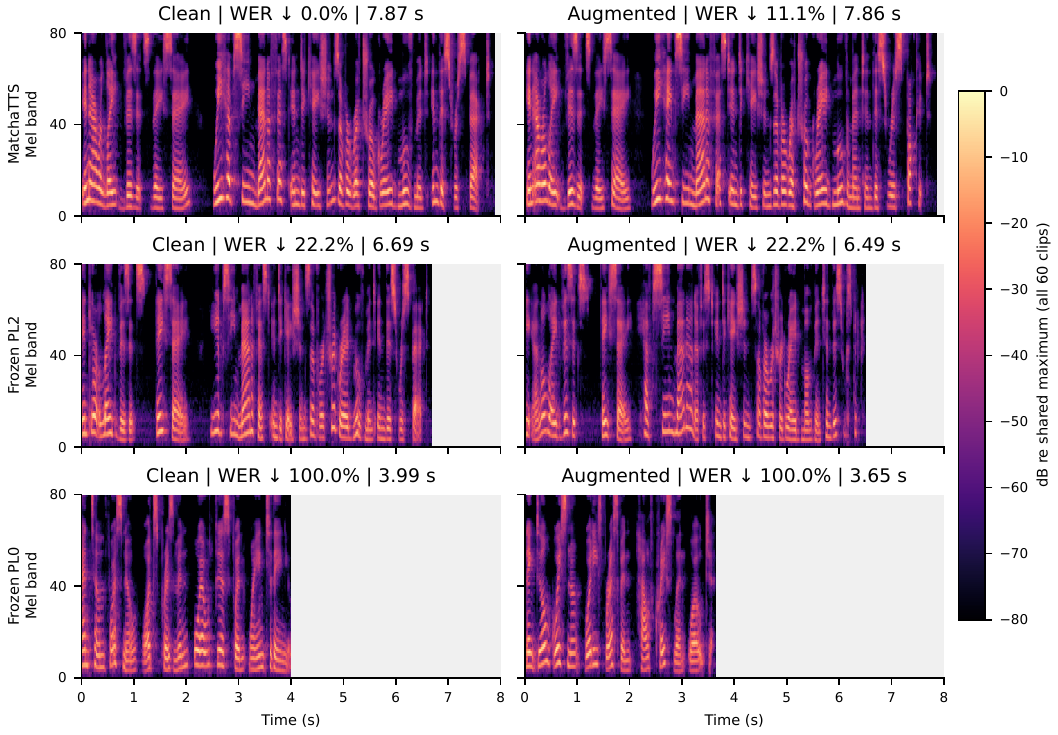}
\caption{LJ007-0211: all three systems under both input views. WER uses the clean reference; gray after a panel ends indicates no generated audio. Axes are shared within this example; the color reference is shared across all ten examples.}
\label{fig:tts-sample-08}
\end{figure}
\noindent\textbf{Whisper hypotheses (normalized, as scored).}
\par\smallskip
\begingroup\small
\noindent\begin{tabular}{@{}p{0.13\linewidth}p{0.405\linewidth}p{0.405\linewidth}@{}}
\toprule
\thead{System} & \thead{Clean audio} & \thead{Augmented audio} \\
\midrule
MatchaTTS & in our late visits they say we have seen manifest indications of a retrograde movement in this respect & in our late visits they say we have seen manifest indications of a retrograde movement in this wrist pocket \\
Frozen PL2 & in your late visits they say you have seen manifest indications of retrigarid movement in this respect & the annalite visits they say have seen manifest indications of a retrograde movement in this respect \\
Frozen PL0 & anthem for snow what is prison well this point went to the new & thank you \\
\bottomrule
\end{tabular}\par\endgroup
\clearpage
\subsection*{Sample 9: LJ038-0221}
\noindent\textbf{Reference/clean input:} and the second month's rent was paid on either April two or April three.\par
\noindent\textbf{Augmented input:} and the second months rent was paid on eITHer April two or April three.\par
\begin{figure}[!htbp]
\centering
\includegraphics[width=\linewidth]{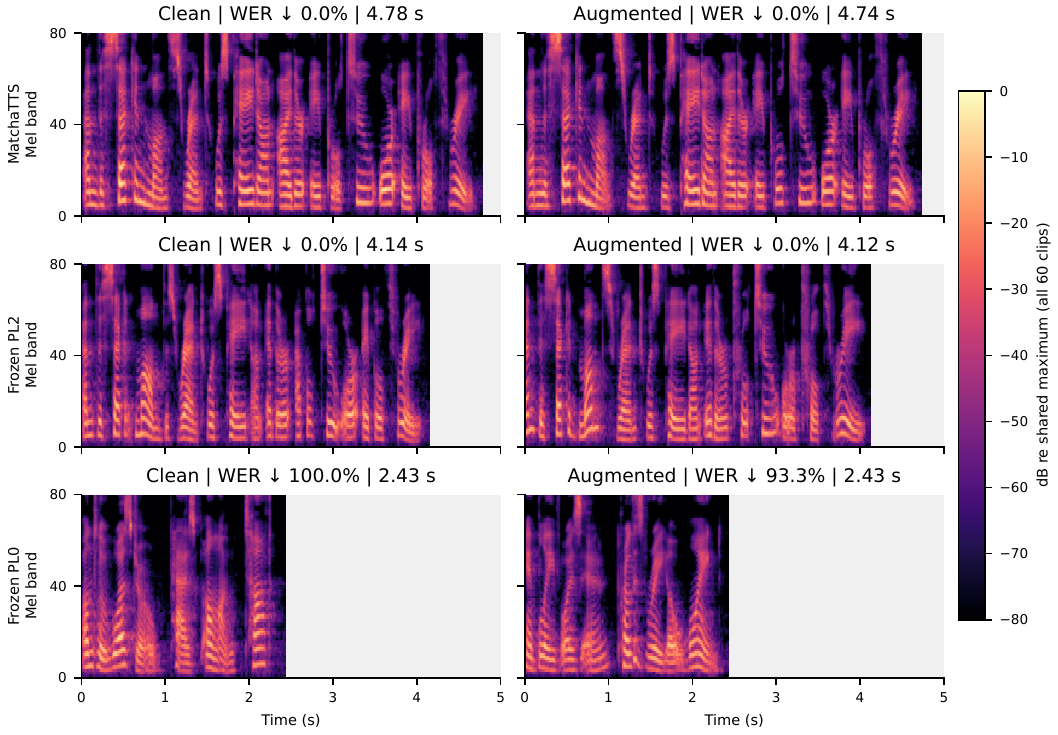}
\caption{LJ038-0221: all three systems under both input views. WER uses the clean reference; gray after a panel ends indicates no generated audio. Axes are shared within this example; the color reference is shared across all ten examples.}
\label{fig:tts-sample-09}
\end{figure}
\noindent\textbf{Whisper hypotheses (normalized, as scored).}
\par\smallskip
\begingroup\small
\noindent\begin{tabular}{@{}p{0.13\linewidth}p{0.405\linewidth}p{0.405\linewidth}@{}}
\toprule
\thead{System} & \thead{Clean audio} & \thead{Augmented audio} \\
\midrule
MatchaTTS & and the 2nd month is rent was paid on either april 2 or april 3 & and the 2nd month is rent was paid on either april 2 or april 3 \\
Frozen PL2 & and the 2nd month is rent was paid on either april 2 or april 3 & and the 2nd month is rent was paid on either april 2 or april 3 \\
Frozen PL0 & uncle wuzdrow has a long tot & can not believe what is in professor regal is mind \\
\bottomrule
\end{tabular}\par\endgroup
\clearpage
\subsection*{Sample 10: LJ006-0020}
\noindent\textbf{Reference/clean input:} William Crawford had been one of the promoters and managers of the Philanthropic Society's farm school.\par
\noindent\textbf{Augmented input:} William crawfrd had been one OF the promoters and managers of the Philanthropic Societys farm sch.\par
\begin{figure}[!htbp]
\centering
\includegraphics[width=\linewidth]{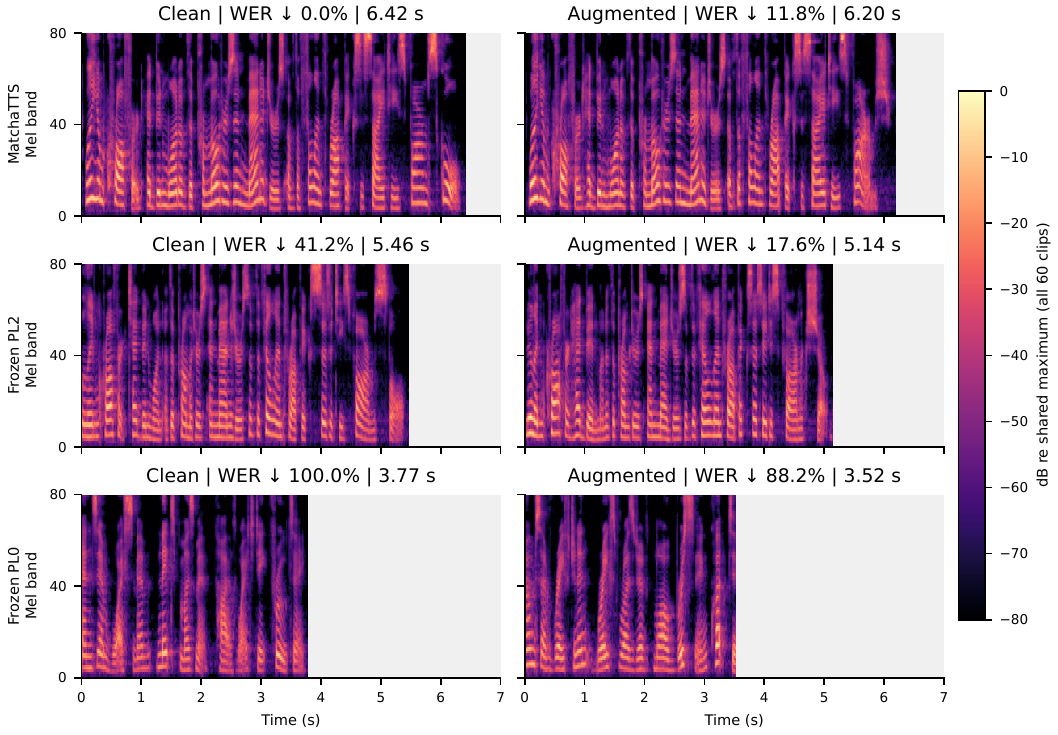}
\caption{LJ006-0020: all three systems under both input views. WER uses the clean reference; gray after a panel ends indicates no generated audio. Axes are shared within this example; the color reference is shared across all ten examples.}
\label{fig:tts-sample-10}
\end{figure}
\noindent\textbf{Whisper hypotheses (normalized, as scored).}
\par\smallskip
\begingroup\small
\noindent\begin{tabular}{@{}p{0.13\linewidth}p{0.405\linewidth}p{0.405\linewidth}@{}}
\toprule
\thead{System} & \thead{Clean audio} & \thead{Augmented audio} \\
\midrule
MatchaTTS & william crawford had been one of the promoters and managers of the philanthropic society is farm school & william crawford had been one of the promoters and managers of the philanthropic society is farms \\
Frozen PL2 & lynn crawford tedbin one of the promoters and managers of the philanthropic societies farms & lillian crawford had been one of the prompters and managers of the philanthropic society is farm show \\
Frozen PL0 & and then we met my husband kyle is wife dave fruits & of gravestone and rice residue while bristling cut to \\
\bottomrule
\end{tabular}\par\endgroup

\section{Computational Properties, Training Recipe, and Responsible Asset Use}
\label{app:reproducibility}
This appendix describes computational cost, the training configuration, and the source-code workflow. It also records the provenance and reuse conditions of the datasets and software used in the experiments.

\subsection{Computational Properties and Sample Budget}
\label{app:complexity}
Let $B$ be the number of sequences actually encoded per view, $V$ the number of views, $N$ the padded input length, $M$ the latent-canvas length, $d$ the width, $h$ the number of attention heads, and $L_e,L_d$ the encoder and decoder depths. Let $K=257$ be the output vocabulary size. The bounds describe dense operations in the supplied implementation. Additional sampled negatives are included in $B$; the sampler's nominal batch size need not equal this effective count.

A forward pass through the encoder, Gaussian resampling, decoder, and token head costs
\begin{align*}
O\!\left(BV\left[ L_e(Nd^2+N^2d)
 +L_d((N+M)d^2+(M^2+MN)d)
 +NMd+MdK\right]\right).
\end{align*}
The terms account for projection/feed-forward layers, encoder self-attention, decoder self/cross-attention, Gaussian resampling, and token logits. Backpropagation has the same asymptotic arithmetic order. The decoder produces all positions in parallel, but attention remains quadratic in sequence length. Gaussian resampling is evaluated even when an equal-length identity branch replaces its result, so this computation is included in the bound.

The parameter count scales as $O((L_e+L_d)d^2+Kd)$, apart from positional buffers; gradients and AdamW states require the same asymptotic parameter storage. A conservative dense training-activation bound is
\begin{align*}
O\!\left(BV\left[(L_eN+L_d(M+N))d
 +h\{L_eN^2+L_d(M^2+MN)\}+MN+MK\right]\right).
\end{align*}
Attention kernels can reduce materialized score storage without changing the dense arithmetic bound. Resampling processes target positions in chunks of at most 256, bounding the temporary weight block by $O(BV\min(M,256)N)$. The training bound also accounts for concatenated outputs and tensors retained for backward.

For $P$ random projections, VISReg projection and per-projection sorting cost $O(V(BdP+PB\log B))$, in addition to centering and scale terms. With projection chunks of size $P_c$, the forward workspace is bounded by $O(VBd+VBP_c+dP_c)$, excluding backward state. The CL defaults use $P=P_c=256$. Token regularization samples one position per sentence, so this branch operates on $B$ sampled states per view rather than all $BM$ positions. Pooling and dense cross-view agreement still require $O(BVMd)$ operations.

For $n_q$ queries and $n_g$ candidates, the exact positional retrieval diagnostic costs $O(n_qn_gRd)$, where $R=16$ sampled relative positions give $Rd=8{,}192$ coordinates. A query block of size $Q$ needs $O(Qn_g)$ score storage, in addition to stored features of order $O((n_q+n_g)Rd)$. The default $Q=64$ limits score workspace but does not remove the quadratic all-pairs cost. Saving full canvases instead requires $O(nMd)$ storage. Forming and decomposing a $d$-dimensional covariance matrix costs $O(nd^2+d^3)$.

Over $S$ optimizer steps, training costs $S$ times the per-step forward/backward cost, plus optimizer updates. PL and CL use a five-epoch budget and nominal training batch size 36; their samplers determine the corresponding number of optimizer steps. The prepared PL and CL training sets contain 8,635,993 and 8,656,553 rows, respectively; their split construction and duplicates are described above. Evaluation uses 40,000 unique targets and 1,024 disjoint calibration targets, while TTS uses 9,170 training and 3,930 validation utterances. These empirical budgets specify the scale of the experiments; the theoretical results concern information retention and population risk. Bootstrap intervals describe the sampled evaluation cohort and do not quantify variation across retraining seeds.

\subsection{Executable Defaults and Training Details}
\label{app:training-recipe}
The effective settings are obtained by merging the shared configuration with each variant's overrides, then applying unoverridden parser and MatchaTTS configuration defaults. All models use a nominal training batch size of 36. The accompanying code package contains the full \path{configs/experiment.json}, \path{provenance/parser_defaults.json}, the MatchaTTS configuration tree, pinned requirements, and a Slurm launcher that prints every command with \texttt{--dry-run}. Together, these files specify the full execution recipe.

\paragraph{PL.}
The launcher uses seed 42, width 512, eight heads, four encoder layers, two decoder layers, a 2,048-position cap, zero dropout, and resampling temperature 0.3. AdamW uses learning rate $5\times10^{-5}$, weight decay 0.01, betas $(0.9,0.98)$, 5\% warmup followed by cosine decay to 0.0005 of the initial rate, gradient clipping at 1, and enabled automatic mixed precision. Training lasts five epochs with no additional positive step cap. The nominal batch is 36, with eight workers, bucket factor 32, and prefetch factor 4. Each canonical item supplies three easy and two hard augmented views in addition to its clean view. The sampler requests six informal and sixteen negated examples, and the default unrelated pool has sixteen entries. The remaining shared objective settings are $n_{\rm pad}=4$, $\tau=0.5$, $\eta=0.1$, $\gamma=1$, $\alpha=0.5$, $\beta=0.3$, and $m_1=m_2=0.2$.

PL0 selects the latent objective with $\lambda=1$ and VISReg weight 0.3. PL1 uses batch 36 and eight unrelated entries. PL2 and PL3 set VISReg weights 0.1 and 0.5. PL4, PL5, PL6, and PL7 respectively set $\alpha$, $\beta$, $\eta$, and $m_2$ to zero. Checkpointing and validation occur every 2,000 steps; logging occurs every 50 steps. The probe interval is 4,000 steps with 400 probe epochs and 160 evaluation texts. The historical PL5 budget exception and the recorded PL0/PL2/PL3 checkpoint steps remain as stated in Appendix~\ref{app:evaluation}; the variant mapping identifies the corresponding experimental comparisons.

\paragraph{CL.}
The default training seed is 1; the data seed is 2027. The nominal batch is 36 with eight workers, and the width, heads, encoder/decoder depths, length cap, dropout, resampling temperature, learning rate, decay, warmup, gradient clip, and five-epoch budget match the settings above. CL uses oracle lengths and recoverable corruption. All four primary arms use VISReg weight 0.01; LP/LT use latent agreement and GP/GT use token grounding, with pooled/token regularization selected by the final letter. VISReg center, scale, and shape weights are one, with shape standard-deviation floor 0.1 and global views 0 and 2. The data preparation fractions are 0.90/0.05/0.05, the streaming shuffle buffer is 10,000, insertion probability is 0.1, and the maximum insertion run is two. The round-trip audit checks 4,096 accepted examples. Nonfinite batches follow the configured skip policy. Logging is every 50 steps, validation/checkpointing every 2,000, and monitoring every 4,000.

\paragraph{Speech and evaluation.}
The MatchaTTS recipe uses seed 42, Adam at $10^{-4}$ with zero weight decay, batch 36, twelve data workers, gradient clipping at 5, mixed bfloat16 precision, and 1,000 epochs. Its architecture and frozen-state interfaces are provided in the code and Appendix~\ref{app:tts}. Stored augmentation is selected with probability 0.7. Evaluation uses ten flow steps, temperature and length scale one, and Whisper large-v3 with float16 inference, English normalization, and beam size five. Text evaluation uses seed 9027, batch four, retrieval blocks of 64, and 500 bootstrap replicates. The executable TTS default is 1,000 bootstrap replicates; the retrospective manuscript WER audit uses 10,000. The executable geometry subset defaults to 2,048 items; the historical PCA diagnostic uses the separately reported 4,096-item subset. The audit and plotting settings identify the corresponding reported analyses.

\paragraph{Environment and selection record.}
The supplied recipe targets Python 3.10--3.12, PyTorch/torchaudio 2.8.0, torchvision 0.23.0, and CUDA 12.8 wheels. All other direct Python requirements are version-pinned, including NumPy 1.26.4, Transformers 4.53.3, Datasets 3.6.0, Lightning 2.5.2, and sentence-transformers 5.0.0. GCC/C++, eSpeak NG, FFmpeg, and a compatible CUDA/cuDNN runtime are system prerequisites. Setup checks dependency consistency and records the installed environment. The evaluation environment is reported separately in Appendix~\ref{app:evaluation}. Training uses the author-confirmed configuration and explicit variant settings described above. The Slurm launcher records resource requests and leaves cluster-specific account and partition fields configurable.

\subsection{Reproduction Workflow}
\label{app:reproduction-scope}
The companion archive supplies data preparation, PL/CL training, feature extraction, text evaluation, speech training, and figure generation, together with pinned dependencies and Slurm instructions. The workflow prepares the corpus, trains the selected models, and passes their outputs to the evaluation and plotting stages. The README gives the setup commands, configuration fields, and output locations; \texttt{--dry-run} prints the complete command and dependency plan before submission.

Corpus input is configured through \texttt{CANOPE\_DATASET\_ID} or the documented local CSV/JSONL interface. The compilation is hosted at \url{https://huggingface.co/datasets/avalonai/english-singlish-g2p}. The schema, filtering rules, split construction, and prepared-data counts are specified in Appendix~\ref{app:evaluation}. Run manifests record configuration, data provenance, and evaluation identities so that comparisons use consistent inputs. The numeric summaries and row-level speech results also support the supplied figure scripts and paired WER audit.

\subsection{Attribution, Licenses, Consent, and Sensitive Content}
\label{app:asset-governance}
\paragraph{Attribution and licenses.}
The component corpora are credited in Appendix~\ref{app:evaluation}. Reused encoders, MatchaTTS, HiFi-GAN, and Whisper are cited in the evaluation and speech sections. Benchmark attribution includes LexNorm~\citep{baldwin2015lexnorm}, SemAntoNeg~\citep{vahtola2022semantoneg}, PAWS~\citep{zhang2019paws}, STS-B~\citep{cer2017sts}, and SICK~\citep{marelli2014sick}. The authors assembled the training compilation from the component corpora identified in Appendix~\ref{app:evaluation}. The code preserves upstream software notices, and its requirements identify external libraries and versions.

The compilation card specifies Apache-2.0, while component data retain their original terms. LJSpeech declares its audio and transcriptions public domain; LibriSpeech specifies CC BY 4.0; and CANNOT specifies CC BY-SA 4.0.\footnote{Creator and dataset pages: \url{https://keithito.com/LJ-Speech-Dataset/}, \url{https://www.openslr.org/12/}, and \url{https://huggingface.co/datasets/tum-nlp/cannot-dataset}.} Wikipedia text is distributed under attribution/share-alike terms, with the applicable version determined by the source revision.\footnote{Wikipedia's licensing transition is documented at \url{https://creativecommons.org/2023/06/29/wikipedia-moves-to-cc-4-0-licenses/}.} The supplied MatchaTTS and HiFi-GAN source includes MIT license notices. The companion asset inventory records asset-specific provenance and licensing information; separately distributed data and model weights remain subject to their respective release terms. The compilation label does not replace these component conditions.

\paragraph{Data-provider and consent context.}
This study reuses publicly released corpora distributed by their original providers and curators. It involved no new participant recruitment, crowdsourcing, or collection of private messages. The authors' contribution to the dataset was aggregation and preparation for the experiments. Collection procedures and participant-consent conditions belong to the original source releases and are described in the cited curator documentation; aggregation does not introduce a new participant-consent process. This statement describes the context of secondary use rather than asserting independently obtained consent from every person mentioned in the source text.

\paragraph{Potentially sensitive content.}
Personal messages, user-authored text, and historical material can contain identifying details, offensive language, or stereotypes. Informal-language and negation transformations can preserve or alter this content. The experiments analyze aggregate representation and speech metrics; data-validity and invertibility checks address the experimental protocol rather than sensitive-content screening. Reuse of individual examples should therefore follow the source terms and take the original content context into account.

\end{document}